\documentclass{article}
\usepackage{microtype}
\usepackage{graphicx}
\usepackage{subfigure}
\usepackage{booktabs} % for professional tables
\usepackage{multirow}
\usepackage{tabto}
\usepackage{hyperref}

\usepackage{xcolor,colortbl, soul, color}
\newcommand{\bb}{\textit{black-box}\xspace}
\newcommand{\wb}{\textit{white-box}\xspace}

\usepackage[accepted]{icml2019}

\icmltitlerunning{There are No Bit Parts for Sign Bits in Black-Box Attacks}

\usepackage{matharticle}
\usepackage[abs]{overpic}
\usepackage{anyfontsize}
\usepackage{paralist}
\usepackage{caption}
\begin{document}

\twocolumn[
\icmltitle{There are No Bit Parts for Sign Bits in Black-Box Attacks}




\begin{icmlauthorlist}
\icmlauthor{Abdullah Al-Dujaili}{mit}
\icmlauthor{Una-May O'Reilly}{mit}
\end{icmlauthorlist}

\icmlaffiliation{mit}{CSAIL, MIT, USA}

\icmlcorrespondingauthor{Abdullah Al-Dujaili}{aldujail@mit.edu}
\icmlcorrespondingauthor{Una-May O'Reilly}{unamay@csail.mit.edu}

\icmlkeywords{Machine Learning, ICML}

\vskip 0.3in
]



\printAffiliationsAndNotice{}  % leave blank if no need to mention equal contribution

\begin{abstract}
We present a black-box adversarial attack algorithm which sets new state-of-the-art model evasion rates for query efficiency in the $\ell_\infty$ and $\ell_2$ metrics, where only loss-oracle access to the model is available. On two public black-box attack challenges, the algorithm achieves the highest evasion rate, surpassing all of the submitted attacks. Similar performance is observed on a model that is secure against substitute-model attacks. For standard models trained on the MNIST, CIFAR10, and IMAGENET datasets, averaged over the datasets and metrics, the algorithm is $3.8\times$ less failure-prone, and spends in total $2.5\times$  fewer queries than the current state-of-the-art attacks combined given a budget of $10,000$ queries per attack attempt. Notably, it requires no hyperparameter tuning or any data/time-dependent prior. The algorithm exploits a new approach, namely sign-based rather than magnitude-based gradient estimation.  This shifts the estimation from continuous to binary black-box optimization. With three properties of the directional derivative, we examine three approaches to adversarial attacks. This yields  a superior algorithm breaking a standard MNIST model using just 12 queries on average!
\end{abstract}

\section{Introduction}
\label{sec:introduction}

\paragraph{Problem.}
Deep Neural Networks (DNNs) are vulnerable to adversarial examples, which are malicious inputs designed to fool the network's prediction---see~\citep{biggio2018wild} for a comprehensive, recent overview of adversarial examples. Research on generating these malicious inputs started in the \wb setting, where access to the gradients of the models was assumed. Since the gradient points to the direction of steepest ascent, a malicious input can be perturbed along that gradient to maximize the network's loss, thereby fooling its prediction.
The assumption of access to the underlying gradient does not however reflect real world scenarios.
Attacks accepting a more realistic,  restrictive \bb threat model, which do not assume access to gradients, have since been studied as will be summarized shortly.

Central to the approach of generating adversarial examples in a \bb threat model is estimating the gradients of the model being attacked.
In estimating these gradients (their magnitudes and signs), the community at large has focused on formulating it as a problem in continuous optimization.
Their works seek to reduce the query complexity from the standard $O(n)$, where $n$ is the number of input features/covariates.
In this paper, we  take a different view and focus on estimating just the sign of the gradient by reformulating the problem as minimizing the Hamming distance to the gradient sign. Given access to a Hamming distance (to the gradient sign) oracle, this view guarantees a query complexity of $\Omega(n/\log_2(n+1))$: an order of magnitude lesser than the full gradient estimation's query complexity for most practically-occurring input
dimensions $n$. Our key objective is to answer the following: \textit{
Is it possible to recover the sign of the gradient with high query efficiency and generate adversarial examples as effective as those generated by full gradient estimation approaches?
}

To this end, we propose a novel formulation capitalizing on some properties of the directional derivative which, approximated by finite difference of loss queries, has been the powerhouse of black-box attacks. Particularly, this  leads to the following contributions at the intersection of adversarial machine learning and black-box (zeroth-order) optimization:
\begin{inparaenum}[1)]
\item We present three properties of the directional derivative of the loss function of the model under attack in the direction of $\{\pm1\}^n$ vectors, and propose methods to estimate the gradient sign bits exploiting these properties. Based on one of the properties, namely separability, we devise a divide-and-conquer algorithm, which we refer to as \signhunter, that reduces the search complexity from $2^n$ sign vectors to $O(n)$. When given a budget of $O(n)$ queries, \signhunter is guaranteed to perform at least as well as \fgsm~\citep{aes2015goodfellow}, which has access to the model's gradient. Through rigorous experiments on both standard and adversarially trained models, we find that \signhunter, in its search for the gradient sign, crafts adversarial examples within a fraction of this number of queries outperforming \fgsm and other state-of-the-art black-box attacks.\footnote{The code for reproducing our work will be made available at https://github.com/ALFA-group}
	\item We release a software framework %\footnote{This builds on other open-source frameworks such as the MNIST and CIFAR challenges~\citep{madry2017towards}.} 
	 to systematically benchmark adversarial black-box attacks on DNNs for \mnist, \cifar, and \imgnt datasets in terms of their success rate, query count, and other related metrics. %This was motivated by the problem we faced in comparing approaches from the literature, where different researchers evaluated their approaches on different datasets, metrics, and setups---e.g., some compared only on \mnist while others considered \cifar and \imgnt.
	\item We identify several key areas of research which we believe will help the community of adversarial learning and gradient-free optimization.
\end{inparaenum}

%We also identify several key areas of research which we believe will help the community towards query-efficient adversarial attacks and gradient-free optimization.

\begin{figure*}[t]
	\centering
	\resizebox{0.88\textwidth}{!}{
		\begin{tabular}{ccc}
			\includegraphics[width=0.3\textwidth, clip, trim=0.95cm 0.9cm 0.8cm 0.1cm]{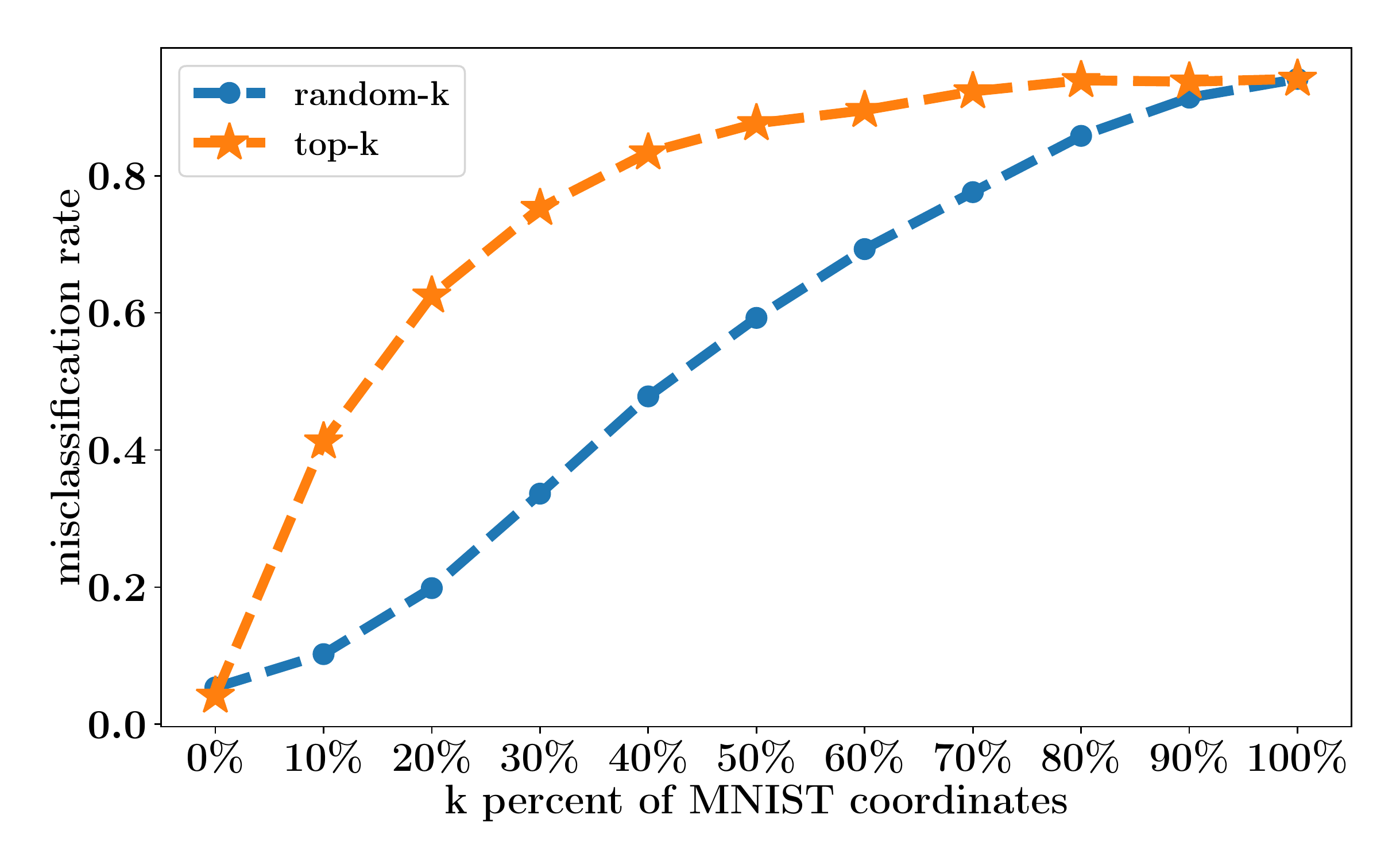} &
			\includegraphics[width=0.3\textwidth,clip, trim=0.95cm 0.9cm 0.8cm 0.1cm]{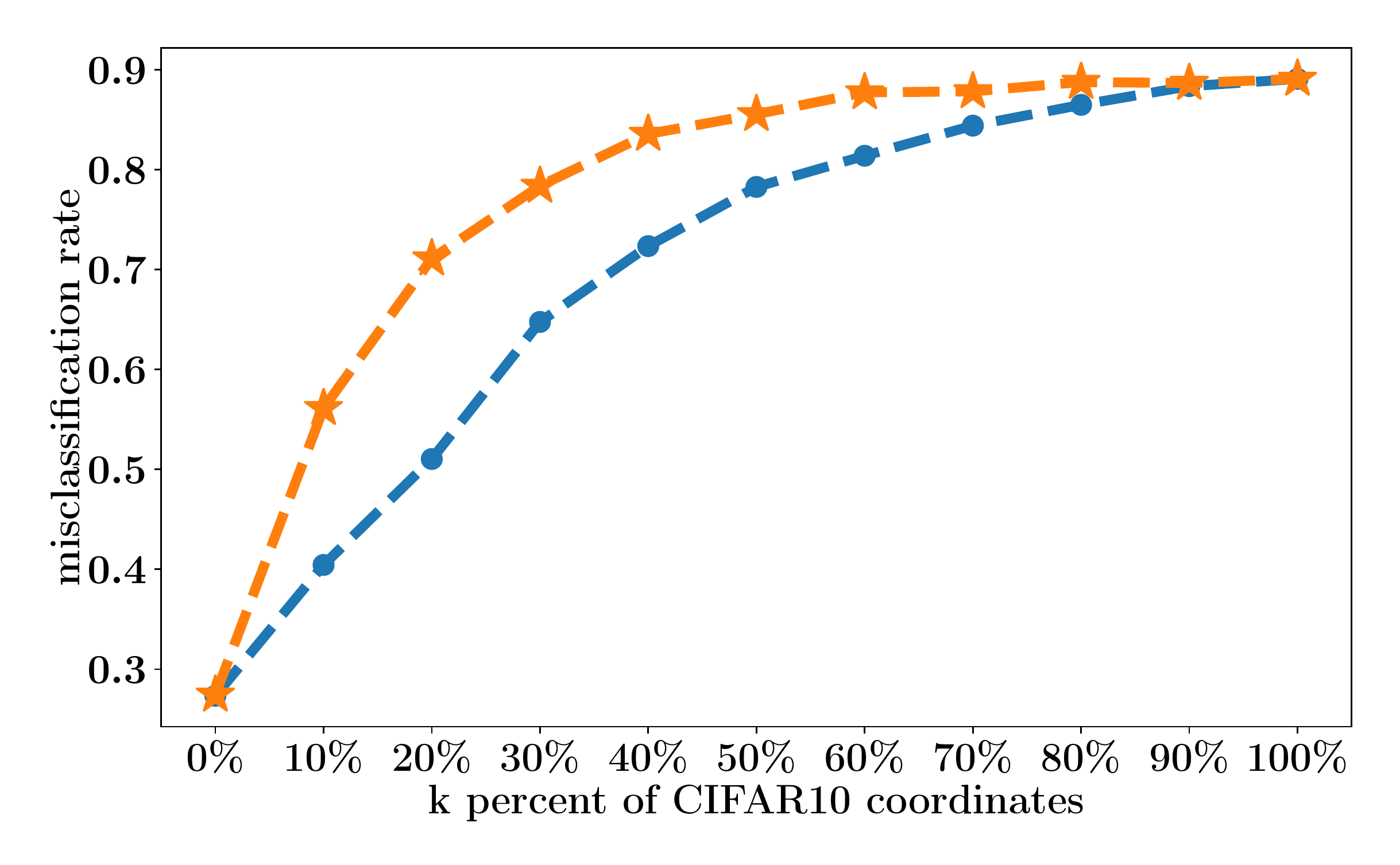} &
			\includegraphics[width=0.3\textwidth,clip, trim=0.95cm 0.9cm 0.8cm 0.1cm]{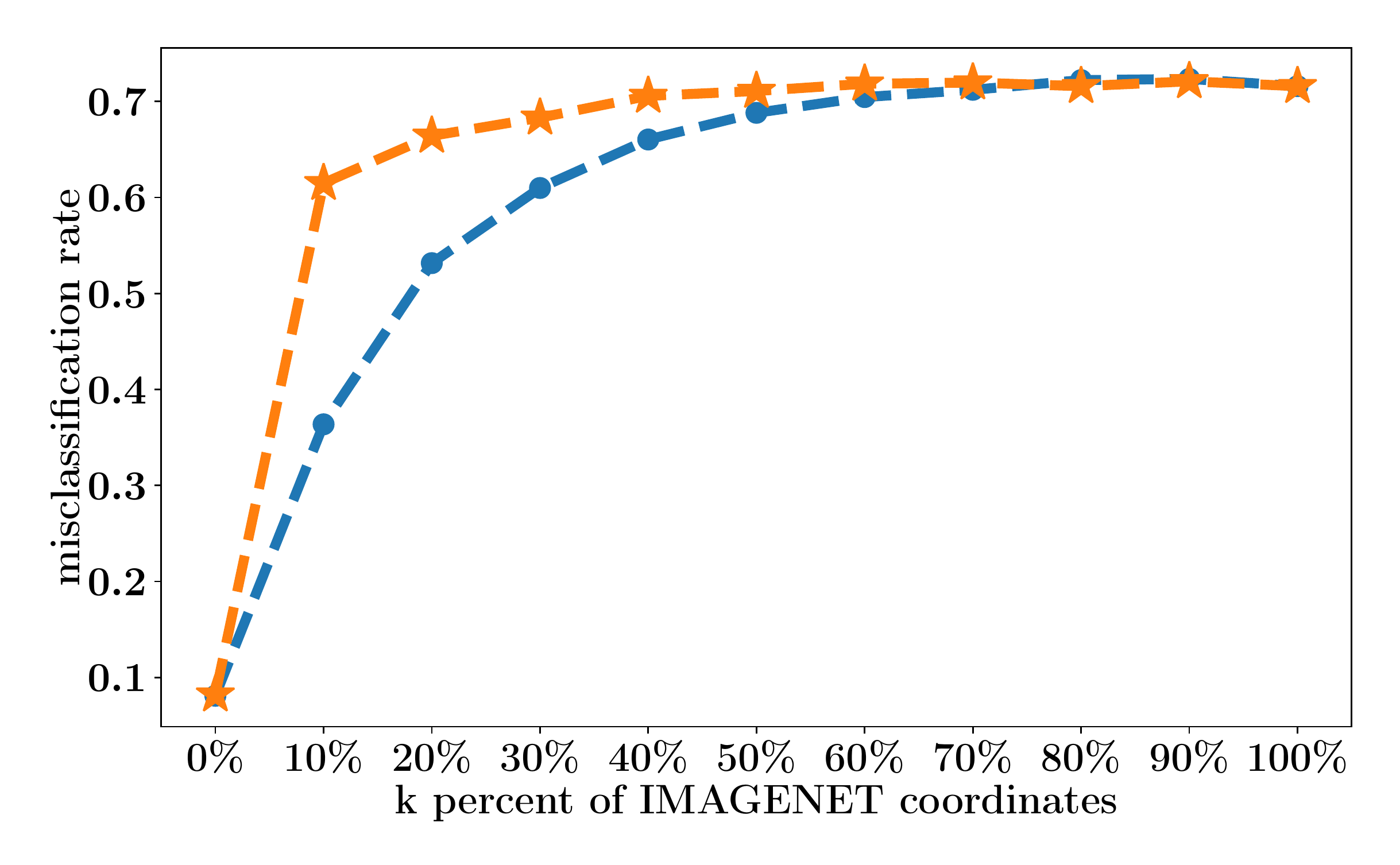} \\
			{ \textbf{(a)}} & { \textbf{(b)}} & { \textbf{(c)}}
			\vspace*{-2mm}
		\end{tabular}
	}
	\caption{Misclassification rate of three neural nets (for (a) \mnist, (b) \cifar, and (c) \imgnt) on the \emph{noisy} \fgsm's adversarial examples as a function of correctly estimated coordinates of $\sgn(\nabla_\vx f(\vx, y))$ on random $1000$ images from the corresponding datasets. Across all the models, estimating the sign of the top $30\%$ gradient coordinates (in terms of their magnitudes)  is enough to achieve a misclassification rate of $\sim70\%$. More details can be found in Appendix~A.
		%, but it is produced with \textsf{TensorFlow} rather than \textsf{PyTorch}.
	}
	\label{fig:keep_k_signs}
\end{figure*}

%\label{sec:related-work}
\paragraph{Related Work.} 
We organize the related work in two themes, namely \emph{Adversarial Example Generation} and \emph{Sign-Based Optimization}. 

\emph{Adversarial Example Generation.}
This literature can be organized as generating examples in either a \wb or a \bb setting.
\citet{nelson2012query} provide a theoretical framework to analyze adversarial querying in a \wb setting. Following the works of \citet{biggio2013evasion} and
\citet{aes2015goodfellow} who introduced the fast gradient sign method (\texttt{FGSM}), several methods to produce adversarial examples have been proposed for various learning tasks and threat perturbation constraints~\citep{carlini2017towards, moosavi2016deepfool, Hayes2017MachineLA, al2018adversarial, huang2018visual, kurk17advsatscale,shamir2019simple}. These methods assume a white-box setup and are not the focus of this work. An approach, which has received the community's attention, involves learning adversarial examples for one model (with access to its gradient information) to transfer them against another~\citep{liu2016delving, papernot2017practical}. As an alternative to the transferability phenomenon, \citet{xiao2018generating} use a Generative Adversarial Network (GAN) to generate adversarial examples which are based on small norm-bounded perturbations.
Both approaches involve learning on a different model, which is expensive, and does not lend itself to comparison in our setup, where we directly query the model of interest.  Among works which generate examples in a \bb setting through iterative optimization schemes, 
\citet{narodytska2017simple} showed how a na\"{i}ve policy of perturbing random segments of an image achieved adversarial example generation. They do not use any gradient information.
\citet{bhagoji2017exploring} reduce the dimensions of the feature space using Principal Component Analysis (PCA) and random feature grouping, before estimating gradients. This enables them to bound the number of queries made.  \citet{chen2017zoo} introduced a principled approach to solving this problem using gradient based optimization. They employ finite differences, a zeroth-order optimization tool, to estimate the gradient and then use it to design a gradient-based attack on models. 
While this approach successfully generates adversarial examples, it is expensive in the number of queries made to the model.
\citet{ilyas18nes} substitute traditional finite differences methods with \texttt{N}atural \texttt{E}volutionary \texttt{S}trategies (\nes) to obtain an estimate of the gradient. \citet{tu2018autozoom} provide an adaptive random gradient estimation algorithm that balances query counts and distortion, and introduces a trained auto-encoder to achieve attack acceleration.
\cite{ilyas2018prior} extend this line of work by proposing the idea of gradient priors and bandits: \bandit. Our work contrasts the general approach used by these works.
We investigate whether just estimating the \emph{sign} of the gradient suffices to efficiently generate examples.

\emph{Sign-Based Optimization.} In the context of general-purpose continuous optimization methods, sign-based stochastic gradient descent was studied in both zeroth- and first-order setups.  In the latter, \citet{bernstein2018signsgd} analyzed \texttt{signSGD}, a \texttt{sign}-based \texttt{S}tochastic \texttt{G}radient \texttt{D}escent, and showed that it enjoys a faster empirical convergence than \texttt{SGD} in addition to the cost reduction of communicating gradients across multiple workers. \citet{liu2018signsgd} extended \texttt{signSGD} to zeroth-order setup with the \zo algorithm.  \zo requires $\sqrt{n}$ times more iterations than \texttt{signSGD}, leading to a convergence rate of $O(\sqrt{n/T})$, where $n$ is the number of optimization variables, and $T$ is the number of iterations.

\emph{Adversarial Examples Meet Sign-based Optimization.} In the context of adversarial examples generation, the effectiveness of  sign of the gradient coordinates was noted in both white- and black-box settings. In the former, the \texttt{F}ast \texttt{G}radient \texttt{S}ign \texttt{M}ethod (\fgsm)---which is algorithmically similar to \texttt{signSGD}---was proposed to generate white-box adversarial examples~\citep{aes2015goodfellow}. \citet{ilyas2018prior} examined a noisy version of \fgsm to address the question of \emph{How accurate of a gradient estimate is necessary to execute a successful attack on a neural net}. In Figure~\ref{fig:keep_k_signs}, we reproduce their experiment on an \imgnt-based model---Plot (c)---and extended it to the \mnist and \cifar datasets---Plots (a) and (b). Observe that estimating the sign of the top $30\%$ gradient coordinates (in terms of their magnitudes)  is enough to achieve a misclassification rate of $\sim70\%$. Furthermore, \zo~\citep{liu2018signsgd} was shown to perform better than \nes at generating adversarial examples against a black-box neural network on the \mnist dataset.

%The rest of the paper is structured as follows. First, a formal background is presented in~Section~\ref{sec:background}. Section~\ref{sec:methods} describes our approach for black-box adversarial attacks by examining three properties of the loss's directional derivative of the model under attack. Experiments are discussed in Section~\ref{sec:experiments}. Using two public black-box attack challenges, we evaluate the approach against one of the defenses developed to mitigate adversarial examples in Section~\ref{sec:challenge}. Finally, open questions and conclusions are outlined in Sections~\ref{sec:open} and~\ref{sec:conclusion}.
\section{Formal Background}
\label{sec:background}

\paragraph{Notation.} 
Let $n$ denote the dimension of a neural network's input. Denote a hidden $n$-dimensional binary code by $\vq^*$. That is, $\vq^*\in \calH \equiv \{-1, +1\}^n$. The response of the Hamming (distance) oracle~$\calO$ to the $i$th query $\vq^{(i)}\in \calH$ is denoted by~$r^{(i)} \in \{0,\ldots, n\}$ and equals the Hamming distance~$r^{(i)}=\ham{\vq^{(i)} - \vq^*}$,
%\begin{equation}
%r^{(i)}=\ham{\vq^{(i)} - \vq^*}\;,
%\label{eq:hamming-dist}
%\end{equation}
where the Hamming norm $\ham{\vv}$ is defined as the number of non-zero entries of vector $\vv$. We also refer to $\calO$ as the \emph{noiseless} Hamming oracle, in contrast to the \emph{noisy} Hamming oracle $\hat{\calO}$, which returns noisy versions of $\calO $'s responses. $\mathbf{1}_n$ is the $n$-dimensional vector of ones.  The query ratio~$\rho \in (0,1]$ is defined as $m/n$ where $m$ is the number of queries to $\calO$ required to retrieve $\vq^*$. Furthermore, denote the directional derivative of some function $f$ at a point $\vx$ in the direction of a vector $\vv$ by $D_\vv f(\vx)\equiv \vv^T \nabla_\vx f(\vx)$ which often can be approximated by the \emph{finite difference} method. That is, for $\delta >0$, we have
\begin{equation}
D_\vv f(\vx) = \vv^T \nabla_\vx f(\vx) \approx 
\frac{ f(\vx + \delta \vv) - f(\vx) }{\delta}\;.
\label{eq:dir-deriv}
\end{equation}
Let $\Pi_{S}(\cdot)$ be the projection operator onto the set $S$, $B_p(\vx, \epsilon)$ be the $\ell_p$ ball of radius $\epsilon$ around $\vx$.
Next, we provide lower and upper bounds on the query ratio~$\rho$.

\paragraph{Bounds on the Query Ratio $\rho$.} Using a packing argument,
~\citet{vaishampayan2012query} proved the following lower bound on query ratio $\rho$.
\begin{theorem}\citep[Theorem 1]{vaishampayan2012query}
	\label{thm:lb}
	For the noiseless Hamming oracle $\calO$, the query ratio must satisfy $\rho=m/n \geq \frac{1}{\log_2(n+1)}$ for any sequence of $m$ queries that determine every $n$-dimensional binary code $\vq^*$ uniquely.
\end{theorem}
\vspace{-1.2em}
\begin{proof}
	See \citep[Page 4]{vaishampayan2012query}.
\end{proof}
\vspace{-1em}
In the following theorem, we show that no more than $n$ queries are required to retrieve the hidden $n$-dimensional binary code $\vq^*$.
\begin{theorem}
	\label{thm:ub}
	A hidden $n$-dimensional binary code $\vq^* \in \calH$ can be retrieved exactly with no more than $n$ queries
	to the noiseless Hamming oracle $\calO$. %\footnote{The result also applies to $n$-dimensional binary codes represented in $\{0,1\}^n$.}
\end{theorem}
\vspace{-1.2em}
\begin{proof}
	See Appendix C.
\end{proof}

\section{Gradient Estimation Problem} 

At the heart of black-box adversarial attacks is generating a \emph{perturbation vector} to slightly modify the original input $\vx$ so as to fool the network prediction of its true label $y$. Put differently, an adversarial example $\vx^\prime$ maximizes the network's loss $L(\vx^\prime, y)$ but still remains $\epsilon_p$-close to the original input $\vx$. Although the loss function $L$ can be non-concave, gradient-based techniques are often very successful in crafting an adversarial example~\citep{madry2017towards}. That is, setting the perturbation vector as a step in the direction of $\nabla_\vx L(\vx, y)$. Consequently, the bulk of black-box attack methods try to \emph{estimate the gradient} by querying an oracle that returns, for a given input/label pair $(\vx, y)$, the value of the network's loss $L(\vx, y)$. Using only such value queries, the basic approach relies on the \emph{finite difference method} to approximate the directional derivative~(\eqref{eq:dir-deriv}) of the function $L$ at the input/label pair $(\vx, y)$ in the direction of a vector $\vv$, which corresponds to $\vv^T\nabla_\vx L(\vx,y)$. With $n$ linearly independent vectors $\{{\vv^{(i)}}^T\nabla_\vx L(\vx,y) = d^{(i)}\}_{1\leq i \leq n}$, one can construct a linear system of equations to recover the full gradient. Clearly, this approach's query complexity is $O(n)$, which can be prohibitively expensive for large $n$ (e.g., $n=268,203$ for the \imgnt dataset). Moreover, the queries are not adaptive, whereas one could make use of the past queries' responses to construct the new query and recover the full gradient with less queries. Recent works tried to mitigate this issue by exploiting data- and/or time-dependent priors~\citep{tu2018autozoom,ilyas18nes,ilyas2018prior}.

The lower bound of Theorem~\ref{thm:lb} on the query complexity of a Hamming oracle $\calO$ to find a hidden vector $\vq^*$ suggests the following: \emph{instead of estimating the full gradient (sign and magnitude) and apart from exploiting any data- or time-dependent priors; focus on estimating its sign} After all, simply leveraging (noisy) sign information of the gradient yields successful attacks; see Figure~\ref{fig:keep_k_signs}. Therefore, our interest in this paper is the \emph{gradient sign estimation problem}, which we formally define next, breaking away from the \emph{full} gradient estimation to construct black-box adversarial attacks.

% \texttt{P}rojected \texttt{G}radient \texttt{D}escent  (\texttt{PGD}) is one such (iterative) technique~\cite{madry2017towards}. Mathematically, we have~\cite{ilyas2018prior}
%$$\vx^\prime = \Pi_{B_p(\vx, \epsilon)}(\vx + \eta \vs)\;,$$
%with 
%$$\vs = \Pi_{\partial B_p(0,1) \de}

\begin{defn}(Gradient Sign Estimation Problem) For an input/label pair $(\vx,y)$ and a loss function $L$, let $\vg^*=\nabla_\vx L(\vx,y)$ be the gradient of $L$ at $(\vx,y)$ and $\vq^*=\sgn(\vg^*) \in \calH$ be the sign bit vector of $\vg^*$.\footnote{Without loss of generality, we encode the sign bit vector in $\calH\equiv \{-1, +1\}^n$ rather than $\{0, 1\}^n$. This is a common representation in sign-related literature. 
Note that the standard $\sgn$ function has the range of $\{-1, 0, +1\}$. Here, we use the non-standard definition~\citep{zhao2018sparse} whose range is $\{-1, +1\}$. This follows from the observation that DNNs' gradients are not sparse~\citep[Appendix B.1]{ilyas2018prior}.} Then the goal of the gradient sign estimation problem is to find a binary %\footnote{Throughout the paper, we use the terms binary vectors and sign vectors/bits interchangeably.} 
vector $\vq \in \calH$ minimizing the Hamming norm
\begin{equation}
\min_{\vq \in \calH}\ \ham{\vq -\vq^*}\;,
\label{eq:grad-est-obj}
\end{equation}
or equivalently maximizing the directional derivative\footnote{The equivalence follows from  $D_{\vq}L(\vx, y)=\vq^{T}\vg^*$, which is maximized when $\vq=\vq^*=\sgn(\vg^*)$, which in turn is a minimizer of \eqref{eq:grad-est-obj}.}
\begin{equation}
\max_{\vq \in \calH} D_{\vq}L(\vx, y)\;,
\label{eq:grad-est-obj-df}
\end{equation}
from a limited number of (possibly adaptive) function value queries $L(\vx^\prime, y)$.
\label{def:grad-sign-est}
\end{defn}
Next, we tackle the gradient sign estimation problem  leveraging three properties of the loss directional derivative $D_\vq L(\vx,y)$ which, in the black-box setup, is approximated by finite difference of loss value queries $L(\vx^\prime, y)$.

\begin{comment}
\begin{rem}
Recall that our definition of the Hamming distance here is over the binary vectors (\eqref{eq:hamming-dist}), a formal statement of the gradient sign estimation problem. In contrast, \citet{shamir2019simple} consider the Hamming distance in defining the $\ell_0$ threat perturbation constraint: if the threat perturbation constraint is $k$, only $k$ data features (pixels) are allowed to be changed, and each one of them can change a lot.
\end{rem}
\end{comment}

\section{A Framework for Estimating Sign of the Gradient from Loss Oracles}
\label{sec:methods}
Our goal is to estimate the gradient sign bits of the loss function $L$ of the model under attack at an input/label pair ($\vx, y$) from a limited number of loss value queries $L(\vx^\prime, y)$. To this end, we examine the basic concept of directional derivatives that has been employed in recent black-box adversarial attacks. Particularly, we present three approaches to estimate the gradient sign bits based on three properties of the directional derivative  $D_\vq L(\vx, y)$ of the loss in the direction of a sign vector $\vq \in \calH$. For the rest of the paper, we only discuss the most successful one: Divide \& Conquer. Others are described in Appendix B.

\paragraph{Approach 1: Divide \& Conquer.}
\label{sec:approach-3}
Based on the definition of the directional derivative~(\eqref{eq:dir-deriv}), we state the following property.
\begin{property} [Separability of $D_{\vq} L(\vx, y)$]
	\label{prop:separable} The directional derivative $D_\vq L(\vx, y)$ of the loss function $L$ at an input/label pair $(\vx,y)$ in the direction of a binary code $\vq$ is separable. That is,
	\begin{equation}
	\label{eq:separable}
	\max_{\vq \in \calH} D_\vq L(\vx, y) = \max_{\vq \in \calH} \vq^T \vg^* = \sum_{i=1}^{n} \max_{q_i \in \{-1,+1\}} q_i g^*_i
	\end{equation}
\end{property}
%Instead of considering the $2^n$ search space (Section~\ref{sec:approach-2}), 
We employ the above property in a divide-and-conquer search which we refer to as \signhunter. As outlined in Algorithm~\ref{alg:sign-hunter}, the technique starts with a random guess of the sign vector $\vq_1$. It then proceeds to flip the sign of all the coordinates to get a new sign vector $\vq_2$, and revert the flips if the loss oracle returned a value $L(\vx + \delta \vq_2, y)$ (or equivalently the directional derivative ) less than the best obtained so far $L(\vx + \delta \vq_1, y)$. \signhunter applies the same rule to the first half of the coordinates, the second half, the first quadrant, the second quadrant, and so on. For a search space of dimension $n$, \signhunter needs $2^{\lceil \log(n) +1\rceil}-1$ sign flips to complete its search. If the query budget is not exhausted by then, one can update $\vx$ with the recovered signs and restart the procedure at the updated point with a new starting code $\vq_1$ ($\vs$ in Algorithm~\ref{alg:sign-hunter}). In the next theorem, we show that \signhunter is guaranteed to perform at least as well as the Fast Gradient Sign Method \fgsm with $O(n)$ oracle queries.
\begin{theorem}(Optimality of \signhunter)
	\label{thm:signhunter} Given $2^{\lceil \log(n) +1\rceil}$ queries and that the directional derivative is well approximated by the finite-difference~(\eqref{eq:dir-deriv}), \signhunter is at least as effective as \fgsm~\citep{aes2015goodfellow} in crafting adversarial examples.
\end{theorem}
\vspace{-1.2em}
\begin{proof} See Appendix C.
\end{proof}
\vspace{-0.5em}
Theorem~\ref{thm:signhunter} provides an upper bound on the number of queries required for \signhunter to recover the gradient sign bits, and perform as well as \fgsm. In practice (as will be shown in our experiments), \signhunter crafts adversarial examples with a fraction of this upper bound. Note that one could recover the gradient sign vector with $n+1 < 2^{\lceil \log(n) +1\rceil}$ queries by starting with an arbitrary sign vector and flipping its bits sequentially. Nevertheless, \signhunter incorporates its queries in a framework of majority voting (weighted by the magnitude of the gradient coordinates) to recover as many sign bits as possible with as few queries as possible. Consider the case where all the gradient coordinates have the same magnitude. If we start with a random sign vector whose Hamming distance to the optimal sign vector $\vq^*$ is $n/2$: agreeing with $\vq^*$ in the first half of coordinates. In this case, \signhunter needs just \emph{three} queries to recover the entire sign vector, whereas the sequential bit flipping would require $n+1$ queries. While the magnitudes of gradient coordinates may not have the same value as considered in the previous example; through empirical evaluation (see Appendix G), we found them to be concentrated. Consequently and with high probability, their votes on retaining or reverting sign flips are weighted similarly.
\begin{algorithm}[t]
	\caption{ \signhunter\\
		\textbf{Input:}\\
		\begin{tabular}{p{1.45cm}p{5.75cm}}
			\footnotesize $g:\calH \to \R$ & \footnotesize : the black-box function to be maximized over the binary  hypercube~$\calH\equiv\{-1, +1\}^n$	
		\end{tabular}
	}    
	\label{alg:sign-hunter}
	\begin{algorithmic}
		\footnotesize
		\STATE \textbf{def} init$(g):$\\
		\hspace{2em}$i\gets 0$\\
		\hspace{2em}$h\gets 0$\\
		\hspace{2em}$g\gets g$\\
		\hspace{2em}$\vs \sim \calU(\calH)$ \hfill// e.g., all ones $[+1,+1,\ldots, +1]$\\
		\hspace{2em}done $\gets false$\\
		\hspace{2em}$g_{best} \gets -\infty$
		\STATE ~\\ \textbf{def} is\_done$():$\\
		\hspace{2em} return done\\
		\STATE ~\\ \textbf{def} step$():$\\
		\hspace{2em} chunk\_len $\gets \lceil n / 2^h\rceil$ \\
		\hspace{2em} flip the bits of $\vs$ indexed from $i *$chunk\_len till \\
		\hspace{3em}$(i+1)*$chunk\_len\\
		\hspace{2em} if $g(\vs) \geq g_{best}$:\\
		\hspace{4em} $g_{best} \gets g(\vs)$\\
		\hspace{2em} else:\\
		\hspace{4em} flip back the bits of $\vs$ indexed from \\
		\hspace{5em} $i *$chunk\_len till  $(i+1)*$chunk\_len\\
		\hspace{2em} increment $i$\\ 
		\hspace{2em} if $i == 2^h$:\\
		\hspace{4em} $i \gets 0$ \\
		\hspace{4em} increment $h$\\
		\hspace{4em} if $h == \lceil \log_2(n) \rceil + 1$: done $\gets true$\\
		\STATE~\\ \textbf{def} get\_current\_sign\_estimate$():$\\
		\hspace{2em} return $\vs$
	\end{algorithmic}
\end{algorithm}
\vspace{-1pt}

\begin{algorithm}[t]
	\caption{Black-Box Adversarial Example Generation with \signhunter \\
		{\bfseries Input:} \\
		\begin{tabular}{p{0.8cm}p{6cm}}
			\footnotesize
			$\vx_{init}$& \footnotesize: input to be perturbed, \\
			\footnotesize $y_{init}$&
			\footnotesize : $\vx_{init}$'s true label,  \\
			\footnotesize$B_{p}(., \epsilon)$&
			\footnotesize : $\ell_p$ perturbation ball of radius~$\epsilon$\\
			\footnotesize$L$ &
			\footnotesize : loss function of the neural net under attack
		\end{tabular}
	}  
	\label{alg:craft-adv}
	\begin{algorithmic}[1]
		\footnotesize
		\label{line:signhunter-ae}
		\STATE  $\delta \gets \epsilon$ \hfill // set finite-difference probe to perturbation bound \\
		\STATE $\vx_o \gets \vx_{init}\ $ \\
		\STATE Define the function $g$ as  
		$$g(\vq) = \frac{L(\Pi_{B_p(\vx_{init}, \epsilon)}(\vx_o + \delta \vq), y_{init}) - L(\vx_o, y_{init})}{\delta}$$\\
		\STATE  \signhunter.init$(g)$
		\STATE $// C(\cdot)$ returns top class
		\WHILE	{$C(\vx) = y_{init}$}  
		\STATE \signhunter.step()
		\STATE $\vs \gets$ \signhunter.get\_current\_sign\_estimate()
		\STATE  $\vx \gets \Pi_{B_p(\vx_{init}, \epsilon)}(\vx_o + \delta \vs) $
		\IF {\signhunter.is\_done()} 
		
		\STATE $\vx_o \gets \vx$ \label{line:start-if}\\
		\STATE Define the function $g$ as  
		$$g(\vq) = \frac{L(\Pi_{B_p(\vx_{init}, \epsilon)}(\vx_o + \delta \vq), y_{init}) - L(\vx_o, y_{init})}{\delta}$$\\ 
		
		\STATE \signhunter.init$(g)$ \label{line:end-if}\\
		\ENDIF
		\ENDWHILE
		\STATE \textbf{return} $\vx$
	\end{algorithmic}
\end{algorithm}
\vspace{-1pt}

\begin{comment}
\begin{algorithm}[t]
	\SetKwInOut{Input}{Input}
	\SetKwInOut{Variables}{Variables}
	\SetKwInOut{Output}{Output}
	\SetKwInOut{Initialization}{Initialization}
	
	\Input{
		$g:\calH \to \R$ : the black-box linear function to be maximized over the binary  hypercube~$\calH\equiv\{-1, +1\}^n$	
	}
	%\Initialization{$\mathcal{T}_1\gets\{(0,0)\}$\\ $t\gets 1$}
	%\Output{}
	\BlankLine
	\caption{\signhunter}
	\label{alg:sign-hunter}
	\textbf{def} init$(g):$\\
	\hspace{2em}$i\gets 0$\\
	\hspace{2em}$h\gets 0$\\
	\hspace{2em}$g\gets g$\\
	\hspace{2em}$\vs \sim \calU(\calH)$\\
	\hspace{2em}done $\gets false$\\
	\hspace{2em}$g_{best} \gets g(\vs)$
	~\newline \\
	\textbf{def} is\_done$():$\\
	\hspace{2em} return done
	~\newline \\
	\textbf{def} step$():$\\
	\hspace{2em} chunk\_len $\gets \lceil n / 2^h\rceil$ \\
	\hspace{2em} flip the bits of $\vs$ indexed from $i *$chunk\_len till  $(i+1)*$chunk\_len\\
	\hspace{2em} if $g(\vs) \geq g_{best}$:\\
	\hspace{4em} $g_{best} \gets g(\vs)$\\
	\hspace{2em} else:\\
	\hspace{4em} flip back the bits of $\vs$ indexed from $i *$chunk\_len till  $(i+1)*$chunk\_len\\
	\hspace{2em} increment $i$\\ 
	\hspace{2em} if $i == 2^h$:\\
	\hspace{4em} $i \gets 0$ \\
	\hspace{4em} increment $h$\\
	\hspace{4em} if $h == \lceil \log_2(n) \rceil + 1$:\\
	\hspace{6em} done $\gets true$
	~\newline \\
	\textbf{def} get\_current\_sign\_estimate$():$\\
	\hspace{2em} return $\vs$
\end{algorithm}
\end{comment}

Further, \signhunter is amenable to parallel hardware architecture and thus can carry out attacks in batches more efficiently, compared to the other two approaches we considered. We tested all the proposed approaches (\signhunter and those in Appendix B) on a set of toy problems and found that \signhunter perform significantly better. For these reasons, in our experiments on the real datasets \mnist, \cifar, \imgnt; we opted for \signhunter as our algorithm of choice to estimate the gradient sign in crafting black-box adversarial attacks as outlined in Algorithm~\ref{alg:craft-adv}.

\begin{comment}

\begin{algorithm}[t]
	\SetKwInOut{Input}{Input}
	\SetKwInOut{Variables}{Variables}
	\SetKwInOut{Output}{Output}
	\SetKwInOut{Initialization}{Initialization}
	
	\Input{\\
		$\vx_{init}$: input to be perturbed, \\
		$y_{init}$: $\vx_{init}$'s true label,  \\
		$B_{p}(., \epsilon)$: $\ell_p$ perturbation ball of radius~$\epsilon$,\\
		$L$: loss function of the neural net under attack
	}
	%\Initialization{$\mathcal{T}_1\gets\{(0,0)\}$\\ $t\gets 1$}
	%\Output{}
	\BlankLine
	$\delta \gets \epsilon$\\
	$\vx_o \gets \vx_{init}\ //$ Adversarial input to be constructed\\
	Define the function $g$ as  
	$$g(\vq) = \frac{L(\Pi_{B_p(\vx_{init}, \epsilon)}(\vx_o + \delta \vq), y_{init}) - L(\vx_o, y_{init})}{\delta}$$\\
	\textit{\signhunter.init$(g)$}\\
	$// C(\cdot)$ returns top class\\
	\While{$C(\vx) = y_{init}$}  {
		\textit{\signhunter.step()}\\
		$\vs \gets$ \textit{\signhunter.get\_current\_sign\_estimate()}\\
		$\vx \gets \Pi_{B_p(\vx_{init}, \epsilon)}(\vx_o + \delta \vs) $ \label{line:signhunter-ae}\\
		\If{\signhunter.is\_done()}{ \label{line:start-if}
			$\vx_o \gets \vx$\\
			Define the function $g$ as  
			$$g(\vq) = \frac{L(\Pi_{B_p(\vx_{init}, \epsilon)}(\vx_o + \delta \vq), y_{init}) - L(\vx_o, y_{init})}{\delta}$$\\
			\textit{\signhunter.init$(g)$}}} \label{line:end-if}
	\textbf{return} $\vx$\\
	\caption{Black-Box Adversarial Example Generation with \signhunter}
	\label{alg:craft-adv}
\end{algorithm}
\end{comment}

%\input{methodotherproperties}

\section{Experiments}
\label{sec:experiments}

We evaluate \signhunter and compare it with established algorithms from the literature: \zo~\citep{liu2018signsgd}, \nes~\citep{ilyas18nes}, and \bandit~\citep{ilyas2018prior} in terms of their effectiveness in crafting (without loss of generality) untargeted black-box adversarial examples. Both $\linf$ and $\ltwo$ threat models are considered on the \mnist, \cifar, and \imgnt datasets.

\paragraph{Experiments Setup.} Our experiment setup is similar to \citep{ilyas2018prior}. Each attacker is given a budget of $10,000$ oracle queries per attack attempt and is evaluated on $1000$ images from the test sets of \mnist, \cifar, and the validation set of \imgnt. We did not find a standard practice of setting the perturbation bound $\epsilon$, arbitrary bounds were used in several papers.

We set the perturbation bounds based on the following. For the $\ell_\infty$ threat model, we use \citep{madry2017towards}'s bound for \mnist and \citep{ilyas2018prior}'s bounds for both~\cifar and \imgnt. 

For the $\ell_2$ threat model, \citep{ilyas2018prior}'s bound is used for \imgnt. \mnist's bound is set based on the sufficient distortions observed in \citep{liu2018signsgd}, which are smaller than the one used in \citep{madry2017towards}. We use the observed bound in \citep{cohen2019smoothing} for \cifar.

We show results based on standard models--i.e., models that are not adversarially hardened. For \mnist and \cifar, the naturally trained models from~\citep{madry2017towards}'s \mnist\footnote{\tiny \url{https://github.com/MadryLab/mnist_challenge}} and \cifar\footnote{\tiny \url{https://github.com/MadryLab/cifar10_challenge}} challenges are used.
For \imgnt, the  \emph{Inception-V3} model from \textsf{TensorFlow} is used.\footnote{\tiny \url{
		https://bit.ly/2VYDc4X}
		%https://github.com/tensorflow/tensorflow/blob/master/tensorflow/contrib/slim/python/slim/nets/inception_v3_test.py}
} The loss oracle represents the cross-entropy loss of the respective model. General setup of the experiments is summarized in Appendix D.

\paragraph{Hyperparameters Setup.} To ensure a fair comparison among the considered algorithms, we did our best in tuning their hyperparameters. Initially, the hyperparameters were set to the values reported by the corresponding authors, for which we observed suboptimal performance. This  can be attributed to either using a different software framework,  %(e.g., \textsf{TensorFlow} vs. \textsf{PyTorch})
 models, or the way the model's inputs are transformed. %(e.g., some models take pixel values to be in the range $[0,1]$ while others are built for $[0, 255]$)
 We made use of a synthetic concave loss function to tune the algorithms for each dataset $\times$ perturbation constraint combination. The performance curves on the synthetic loss function using the tuned values of the hyperparameters did show consistency with the reported results from the literature. For instance, we noted that \zo converges faster than \nes. Further, \bandit outperformed the rest of the algorithms towards the end of query budget. That said, we invite the community to provide their best tuned attacks. Note that \signhunter does not have any hyperparameters to tune. The finite difference probe $\delta$ for \signhunter is set to the perturbation bound $\epsilon$ because this perturbation is used for for both computing the finite difference and crafting the adversarial examples---see Line~1 in Algorithm~\ref{alg:craft-adv}. This parameter-free setup of \signhunter offers a robust edge over the state-of-the-art black-box attacks, which often require expert knowledge to carefully tune their parameters as discussed above. More details on the  hyperparameters setup can be found in Appendix D.

\begin{figure*}[t]
	\centering\resizebox{0.9\textwidth}{!}{
		\begin{tabular}{ccc}
			\includegraphics[width=0.3\textwidth ]{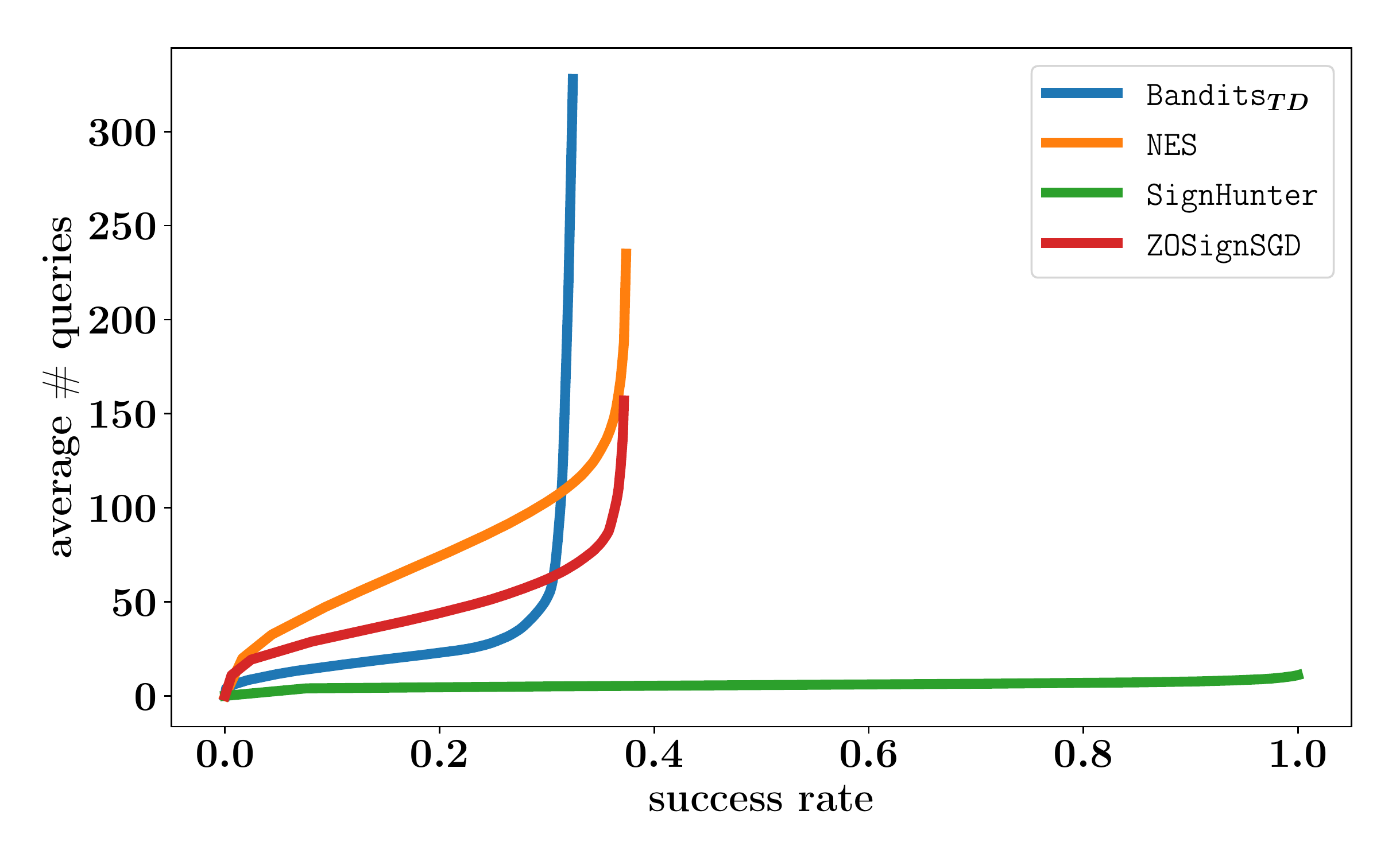} &
			\includegraphics[width=0.3\textwidth ]{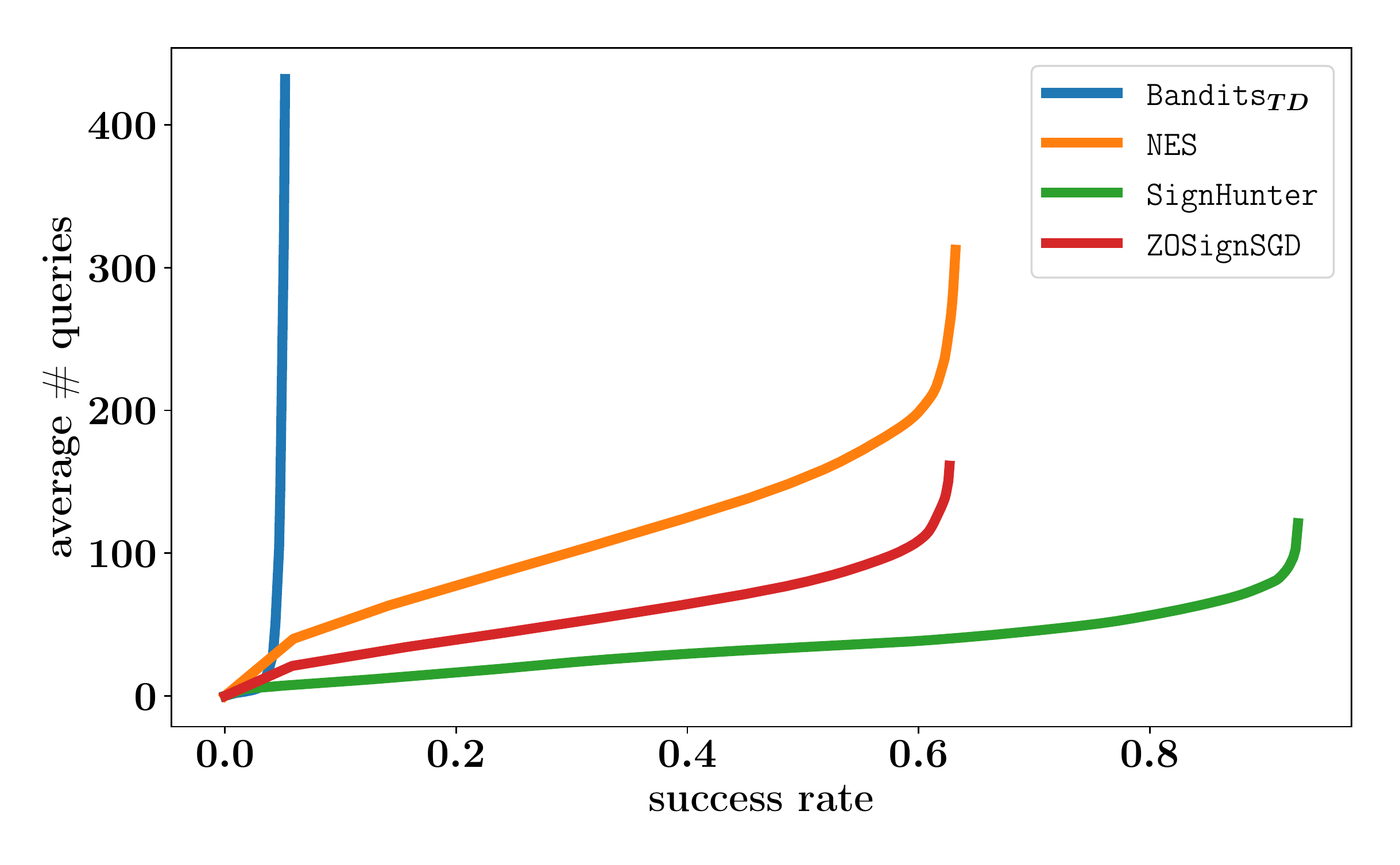} &
			\includegraphics[width=0.3\textwidth ]{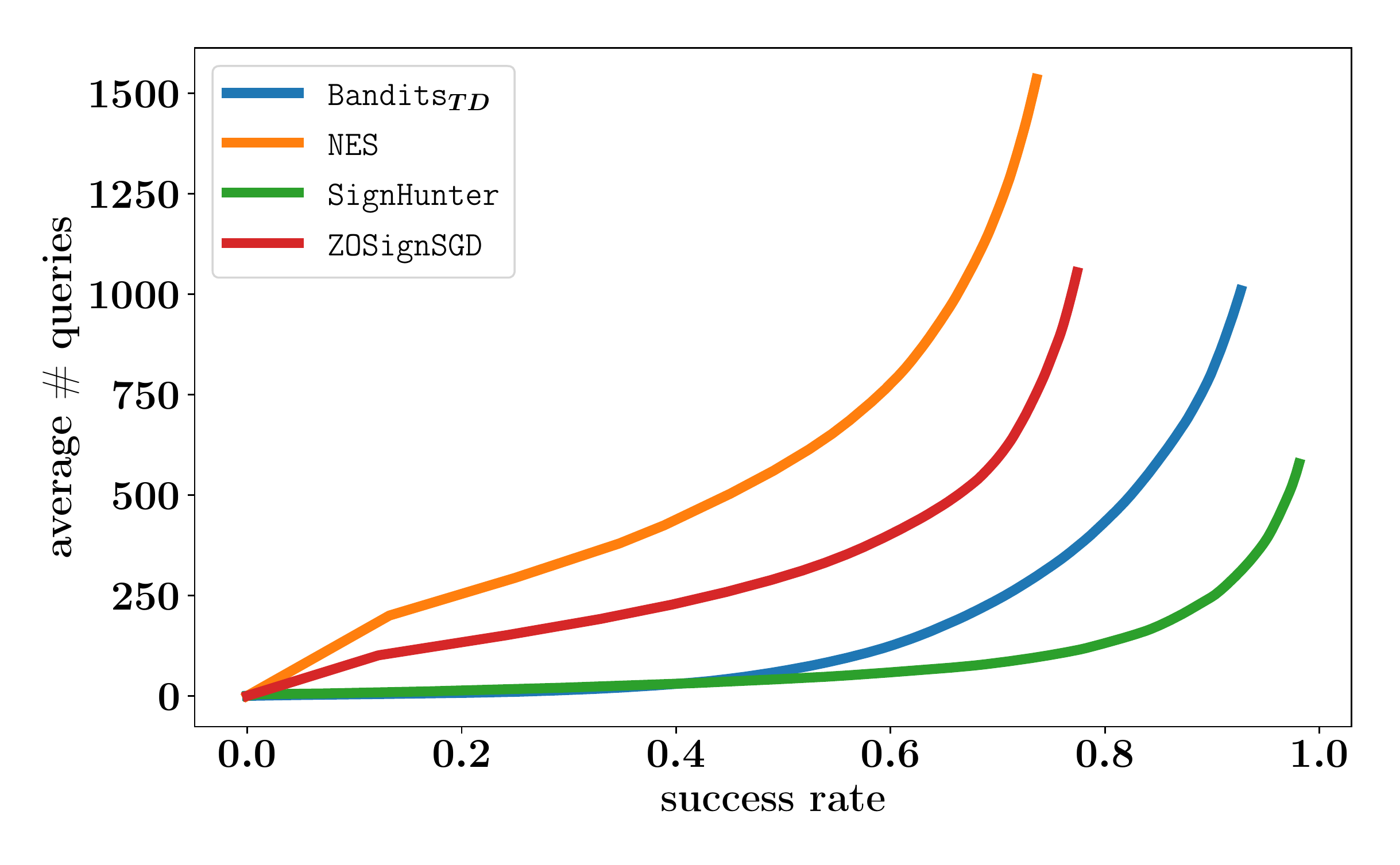}  \\
			(a) \mnist $\linf$ & (b)  \cifar $\linf$  & (c) \imgnt $\linf$ \\ 
			\includegraphics[width=0.3\textwidth ]{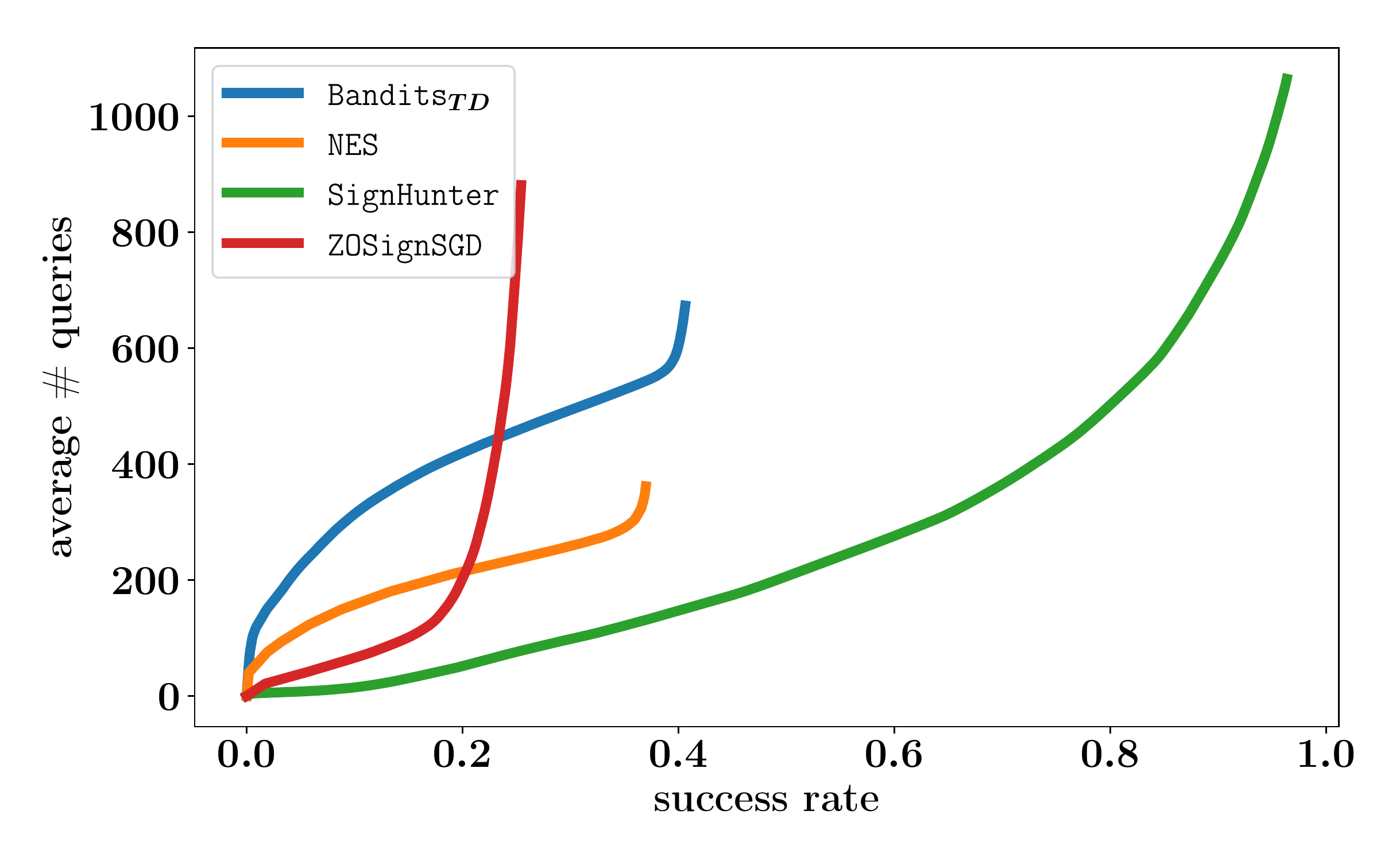} &
			\includegraphics[width=0.3\textwidth ]{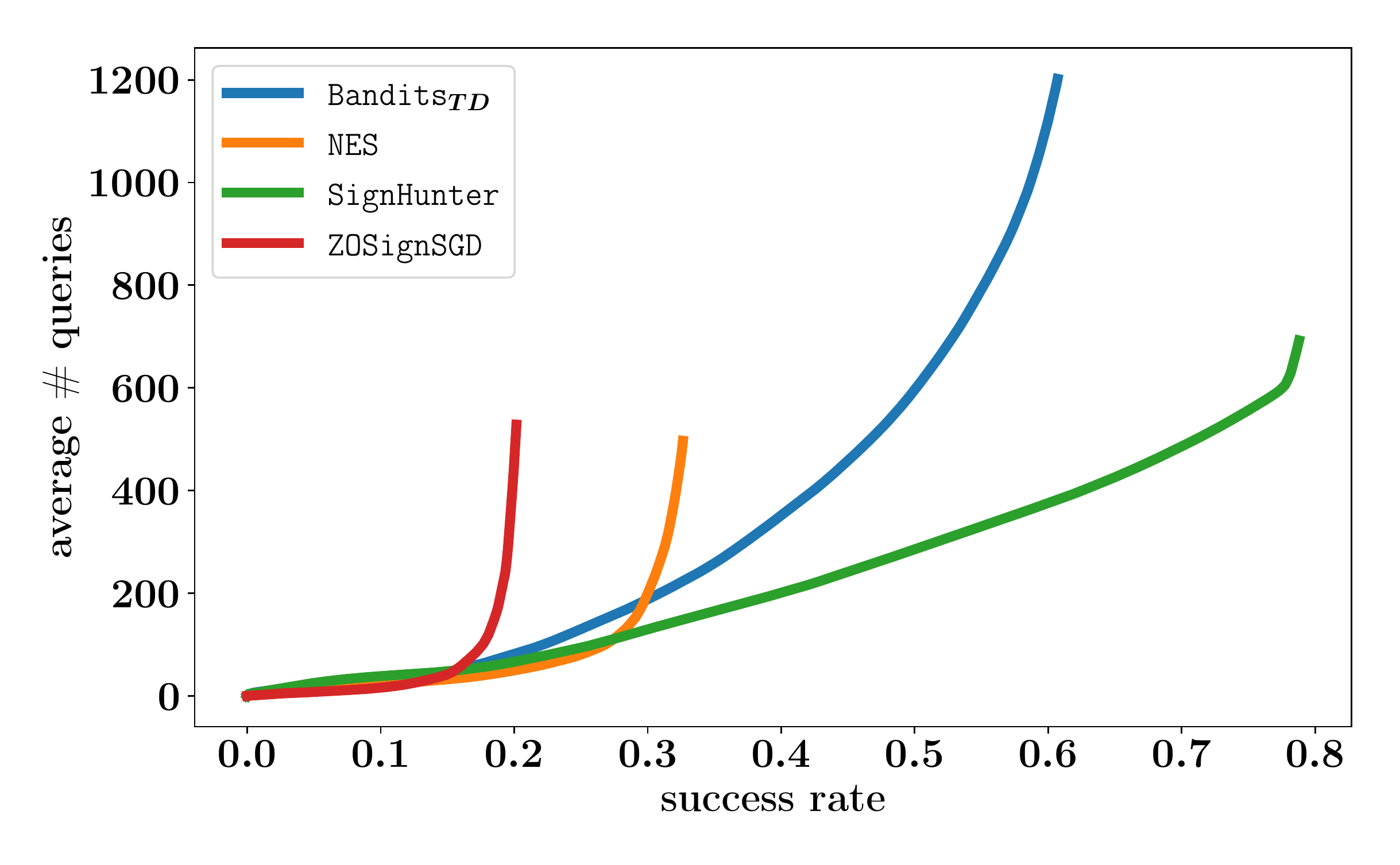} &
			\includegraphics[width=0.3\textwidth ]{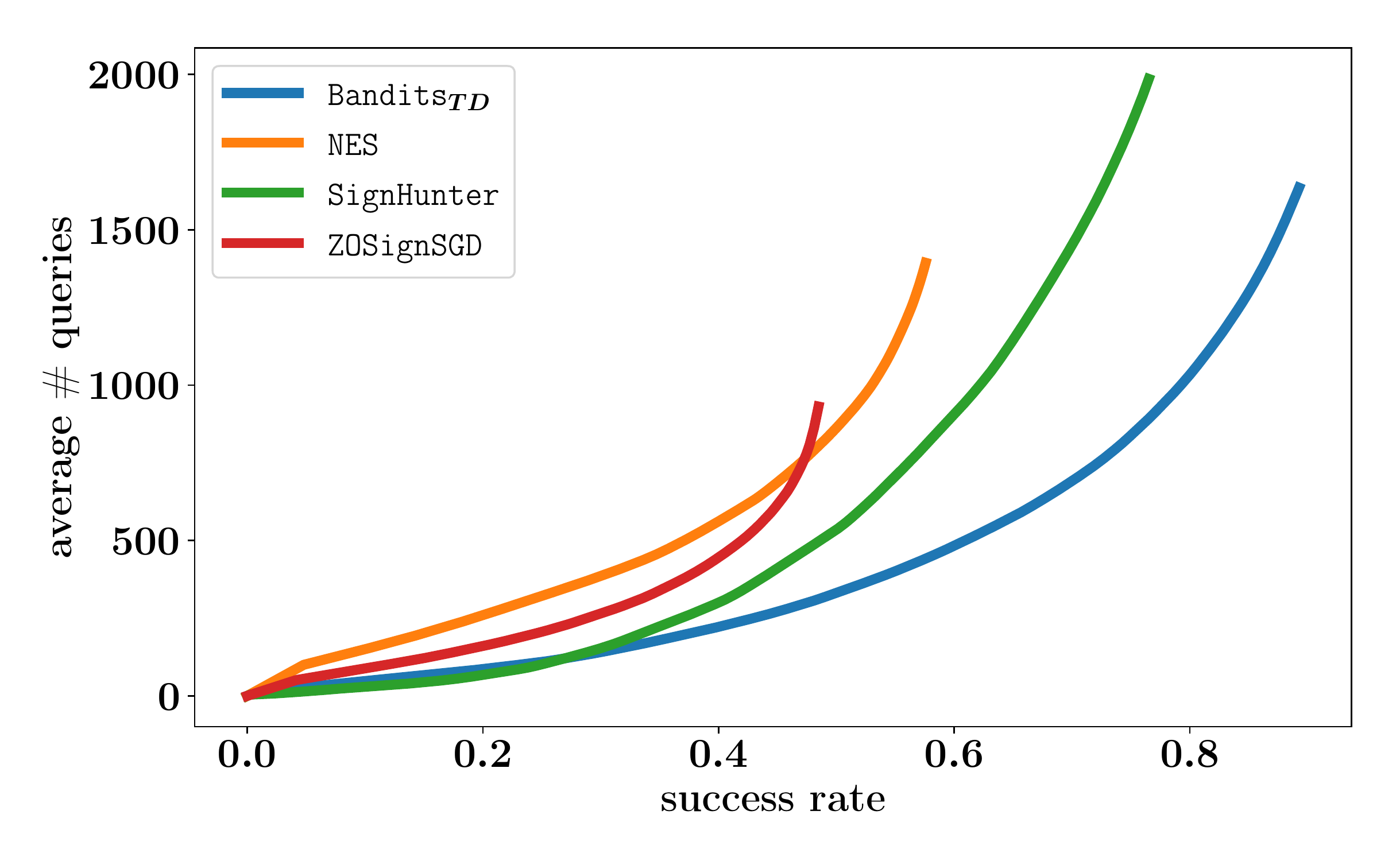}  \\
			(d) \mnist $\ltwo$ & (e)  \cifar $\ltwo$ & (f) \imgnt $\ltwo$ \\
			\\
	\end{tabular}}
	\caption{Performance of black-box attacks in the $\linf$ and $\ltwo$ perturbation constraint. The plots show the average number of queries used per successful image for each attack when reaching a specified success rate.}
	\label{fig:res-linf}
\end{figure*}

%We record the effectiveness of the compared algorithms in terms of the number of queries used and the success (model evasion) rate. 

\paragraph{Results.} Figure~\ref{fig:res-linf} shows the trade-off between the success (evasion) rate and the mean number of queries (of the successful attacks) needed to generate an adversarial example for the \mnist, \cifar, and \imgnt classifiers in the $\linf$ and $\ltwo$ perturbation constraints. In other words, these figures indicate the average number of queries required for a desired success rate. Tabulated summary of these plots can be found in Appendix E, namely Tables~6, 7, and 8. Furthermore, we plot the classifier loss and the gradient estimation quality (in terms of Hamming distance and Cosine similarity) averaged over all the images as a function of the number of queries used in Figures~10, 11, and 12 of Appendix E. Based on the results, we observe the following:

\begin{comment}
\begin{figure*}[h]
	\centering\resizebox{0.9\textwidth}{!}{
	\begin{tabular}{ccc}
		\includegraphics[width=0.3\textwidth ]{figs/mnist_sota_tbl_plots/mnist_2_qrt_plt.pdf} &
		\includegraphics[width=0.3\textwidth ]{figs/cifar10_sota_tbl_plots/cifar10_2_qrt_plt.pdf} &
		\includegraphics[width=0.3\textwidth ]{figs/imagenet_sota_tbl_plots/imagenet_2_qrt_plt.pdf}  \\
		(a) \mnist & (b)  \cifar & (c) \imgnt\\
	\end{tabular}}
	\caption{Performance of black-box attacks in the $\ltwo$ perturbation constraint. The plots show the average number of queries used per successful image for each attack when reaching a specified success rate.}
	\label{fig:res-l2}
\end{figure*}
\end{comment}

For any given success rate, \signhunter dominates the previous state of the art approaches in all settings except the \imgnt $\ltwo$ setup,\footnote{Strictly speaking, all the algorithms are comparable in the \cifar $\ltwo$ setup for success rate $\leq 0.3$.} where \bandit shows a better query efficiency when the desired success rate is greater than or equal $\sim0.35$.  Our approach is remarkably efficient in the $\linf$ setup (e.g., achieving a $\bf 100\%$ evasion using---on average---just $\bf 12$ queries per image against the \mnist classifier!). Its performance degrades---yet, still outperforms the rest, most of the time---in the $\ltwo$ setup. This is expected, since \signhunter perturbs all the coordinates with the same magnitude and the $\ltwo$ perturbation bound $\epsilon_2$ for all the datasets in our experiments is set such that $\epsilon_{2} / \sqrt{n} \ll \epsilon_{\infty} $ as shown in Table~1 of Appendix D. Take the case of \mnist ($n=28\times 28$), where $\epsilon_\infty = 0.3$ and $\epsilon_2=3$. For \signhunter, the $\ltwo$ setup is equivalent to an $\linf$ perturbation bound of $3/28\approx 0.1$. The employed $\ell_2$ perturbation bounds give the state of the art---continuous optimization  based---approaches more perturbation options. For instance, it is possible for \nes to perturb just one pixel in an MNIST image by a magnitude of $3$; two pixels by a magnitude of $2.1$ each; ten pixels by a magnitude of $0.9$ each, etc. On the other hand, the binary optimization view of \signhunter limits it to always perturb all $28\times 28$ pixels by a magnitude of $0.1$. Despite its less degrees of freedom, \signhunter maintains its effectiveness in the $\ltwo$ setup. The plots can be viewed as a sensitivity assessment of \signhunter as $\epsilon$ gets smaller for each dataset. Moreover, the performance of \signhunter is in line with Theorem~\ref{thm:signhunter} when compared with the performance of \fgsm (the noisy \fgsm at $k=100\%$ in Figures~1 and 2 of Appendix A) in both $\linf$ and $\ltwo$ setups for \mnist and \cifar---for \imgnt, $2n=536,406$ is beyond our query budget of $10,000$ queries. For example, \fgsm has a failure rate of $0.32$ for \cifar $\ltwo$ (Appendix A, Figure~2 (b)), while \signhunter achieves a failure rate of $0.21$ with $692.39 < 2n = 2 \times 3 \times 32 \times 32= 6144$ queries (Appendix~E, Table~7). 

Incorporating \signhunter in an iterative framework of perturbing the data point $\vx$ till the query budget is exhausted (Lines~10 to~14 in Algorithm~\ref{alg:craft-adv}) supports the observation in white-box settings that iterative \fgsm---or \texttt{P}rojected \texttt{G}radient \texttt{D}escent (\texttt{PGD})---is stronger than \fgsm~\citep{madry2017towards,al2018adversarial}. This is evident by the upticks in \signhunter's performance on the \mnist $\ltwo$ case (Figure~10 of Appendix E: classifier's loss, Cosine and Hamming similarity plots), which happens after every iteration (after every other $2\times 28 \times 28$ queries). Plots of the Hamming similarity capture the quality of the gradient sign estimation in terms of~\eqref{eq:grad-est-obj}, while plots of the average Cosine similarity capture it in terms of~\eqref{eq:grad-est-obj-df}. Both \signhunter and \bandit consistently optimize both objectives. In general, \signhunter enjoys a faster convergence especially on the Hamming metric as it is estimating the signs compared to \bandit's full gradient estimation. This is highlighted in the \imgnt $\ltwo$ setup. Note that once an attack is successful, the gradient sign estimation at that point is used for the rest of the plot. This explains why, in the $\linf$ settings, \signhunter's plot does not improve compared to its $\ltwo$ counterpart, as most of the attacks are successful in the very first few queries made to the loss oracle.

Overall, \signhunter is $3.8\times$ less failure-prone than the state-of-the-art approaches combined, and spends over all the images (successful and unsuccessful attacks) $2.5\times$ less queries. The number of queries spent is computed based on Tables~6,~7,~and~8 of Appendix E as \newline~\tabto{0.75cm}~\textit{ \small (1 - fail\_rate) * avg_\#_queries +  fail\_rate * 10,000}.

\section{Attack Effectiveness Under Defenses}
\label{sec:challenge}

\begin{table}[t]
	\caption{Top-3 attacks on the \mnist black-box challenge. Adapted from the challenge's website---as of Feb 22, 2019.}
	\label{tbl:mnist-challenge}
	\centering
	\resizebox{0.49\textwidth}{!}{
		\begin{tabular}{p{11cm}|c}
			\toprule
			\textbf{Black-Box Attack} & \textbf{Model Accuracy}\\ 
			\toprule
			\signhunter (Algorithm~\ref{alg:craft-adv}) & $\bf 91.47\%$\\
			\cite{xiao2018generating}	& $92.76\%$	\\
			\texttt{PGD} against three independently and
			adversarially trained copies of the network	&	$93.54\%$	\\ 
			%\fgsm on the CW loss for model B from 
			%\citep{tramer2017ensemble}	&	$94.36\%$ \\
			%\fgsm on the CW loss for the 
			%naturally trained public network&$96.08\%$ \\
			%\texttt{PGD} on the cross-entropy loss for the
			%naturally trained public network	&	$96.81\%$ \\
			%Attack using Gaussian Filter for selected pixels
			%on the adversarially trained public network	&	$97.33\%$	\\
			%\fgsm on the cross-entropy loss for the
			%adversarially trained public network	&$97.66\%$	\\
			%\texttt{PGD} on the cross-entropy loss for the
			%adversarially trained public network	&	$97.79\%$	\\
			\bottomrule
		\end{tabular}
	}
\end{table}

\begin{table}[h!]
	\caption{Top-3 attacks for the \cifar black-box challenge. Adapted from the challenge's website---as of Feb 22, 2019.}
	\label{tbl:cifar-challenge}
	\centering
	\resizebox{0.49\textwidth}{!}{
		\begin{tabular}{p{11cm}|c}
			\toprule
			\textbf{Black-Box Attack} & \textbf{Model Accuracy}\\ 
			\toprule
			\signhunter (Algorithm~\ref{alg:craft-adv}) & $\bf 47.16\%$ \\
			\texttt{PGD} on the cross-entropy loss for the
			adversarially trained public network &	$63.39\%$\\
			\texttt{PGD} on the CW loss for the
			adversarially trained public network&	$64.38\%$\\
			%	\fgsm on the CW loss for the
			%	adversarially trained public network&	$67.25\%$\\
			%	\fgsm on the CW loss for the
			%	naturally trained public network&	$85.23\%$ \\
			\bottomrule
		\end{tabular}
	}
\end{table}

\begin{table}[h!]
	\caption{Top 1 Error percentage. The numbers between brackets are computed on 10,000 images from the validation set. The rest are from~\citep[Table 4]{tramer2017ensemble}.}
	\label{tbl:img-challenge}
	\centering
	\resizebox{0.49\textwidth}{!}{
		\begin{tabular}{lcccc}
			\toprule
			\multirow{2}{1cm}{\textbf{Model}} & \multirow{2}{1cm}{\textbf{clean}} & \multirow{2}{3cm}{\textbf{Max. Black-box}} & \multicolumn{2}{c}{\textbf{\signhunter}} \\ 
			& & & after 20 queries & after 1000 queries \\
			\toprule
			v3$_{\text{adv-ens4}}$ & 24.2 (26.73) & 33.4 & (40.61)& \textbf{(90.75)}\\
			\bottomrule
		\end{tabular}
	}
\end{table}

To complement our results in Section~\ref{sec:experiments}, we evaluated \signhunter against \emph{adversarial training},  an effective way to improve
the robustness of DNNs~\citep{madry2017towards}. In particular, we attacked the \emph{secret} models used in public challenges for \mnist and \cifar. There was no corresponding challenge for \imgnt. Instead, we used \emph{ensemble adversarial training}, a method that argues security against black-box attacks based on transferability/substitute models~\cite{tramer2017ensemble}. The same metrics used in Section~\ref{sec:experiments} are recorded for the experiments here  in Appendix~F.

\paragraph{Public MNIST Black-Box Attack Challenge.}
\label{sec:mnist-challenge}
In line with the challenge setup, $10,000$ test images were used with an $\linf$ perturbation bound of $\epsilon=0.3$. Although the secret model is released, we treated it as a black box similar to our experiments in Section~\ref{sec:experiments}. No maximum query budget was specified, so we set it to $5,000$ queries. This is similar to the number of iterations given to  a \texttt{PGD} attack in the white-box setup of the challenge: 100-steps
with 50 random restarts. As shown in Table~\ref{tbl:mnist-challenge}, \signhunter's attacks resulted in the lowest model
accuracy of $\bf 91.47\%$, outperforming all other state-of-the-art
black-box attack strategies submitted to the challenge with an average number of queries of $\bf 233$ per successful attack. We would like to note that the attacks submitted to the challenge are based on transferability and do not query the model of interest. On the other hand, the most powerful \emph{white-box} attack  by~\citet{zheng2018distributionally}---as of Feb 22, 2019---resulted in a model accuracy of  $88.56\%$--not shown in the table. Further, a \texttt{PGD} attack with $5000$ iterations/back-propagations ($100$ steps and $50$ random restarts) achieves $89.71\%$ in contrast to \signhunter's $91.47\%$ with just $5000$ forward-propagations.

\paragraph{Public CIFAR10 Black-Box Attack Challenge.} In line with the challenge setup, $10,000$ test images were used with an $\linf$ perturbation bound of $\epsilon=8$. Although the secret model is released, we treated it as a black box similar to our experiments in Section~\ref{sec:experiments}. Similar to the MNIST challenge, the query budget is $5,000$ queries. From Table~\ref{tbl:cifar-challenge}, \signhunter's attacks resulted in the lowest model
accuracy of $\bf 47.16\%$, outperforming all other state-of-the-art
black-box attack strategies submitted to the challenge with an average number of queries of $\bf 569$ per successful attack. We would like to note that the attacks submitted to the challenge are based on transferability and do not query the model of interest. On the other hand, the most powerful \emph{white-box} attack by \citet{zheng2018distributionally}, ---as of Feb 22, 2019---resulted in a model accuracy of  $44.71\%$--not shown in the table. Further, a \texttt{PGD} attack with $200$ iterations/back-propagations ($20$ steps and $10$ restarts) achieves $45.71\%$ in contrast to \signhunter's $47.16\%$ with $5000$ forward-propagations.

\paragraph{Ensemble Adversarial Training on IMAGENET.} In line with~\cite{tramer2017ensemble}, we set~$\epsilon=0.0625$ and report the model's misclassification over 10,000 random images from \imgnt's validation set. We attack the v3$_{\text{adv-ens4}}$ model.\footnote{\tiny\url{https://bit.ly/2XWTdKx} %\url{http://download.tensorflow.org/models/ens4_adv_inception_v3_2017_08_18.tar.gz}
	} As shown in Table~\ref{tbl:img-challenge}, after $20$ queries, \signhunter achieves a top-1 error of $40.61\%$ greater than the $33.4\%$ rate of a series of black-box attacks
(including \texttt{PGD} with $20$ iterations) transferred from a substitute model. With $1000$ queries, \signhunter breaks the model's robustness with a top-1 error of $90.75\%$!

	\begin{figure*}[h!]
		\centering
		\resizebox{\textwidth}{!}{
			\begin{tabular}{ccc}
				\includegraphics[width=0.3\textwidth]{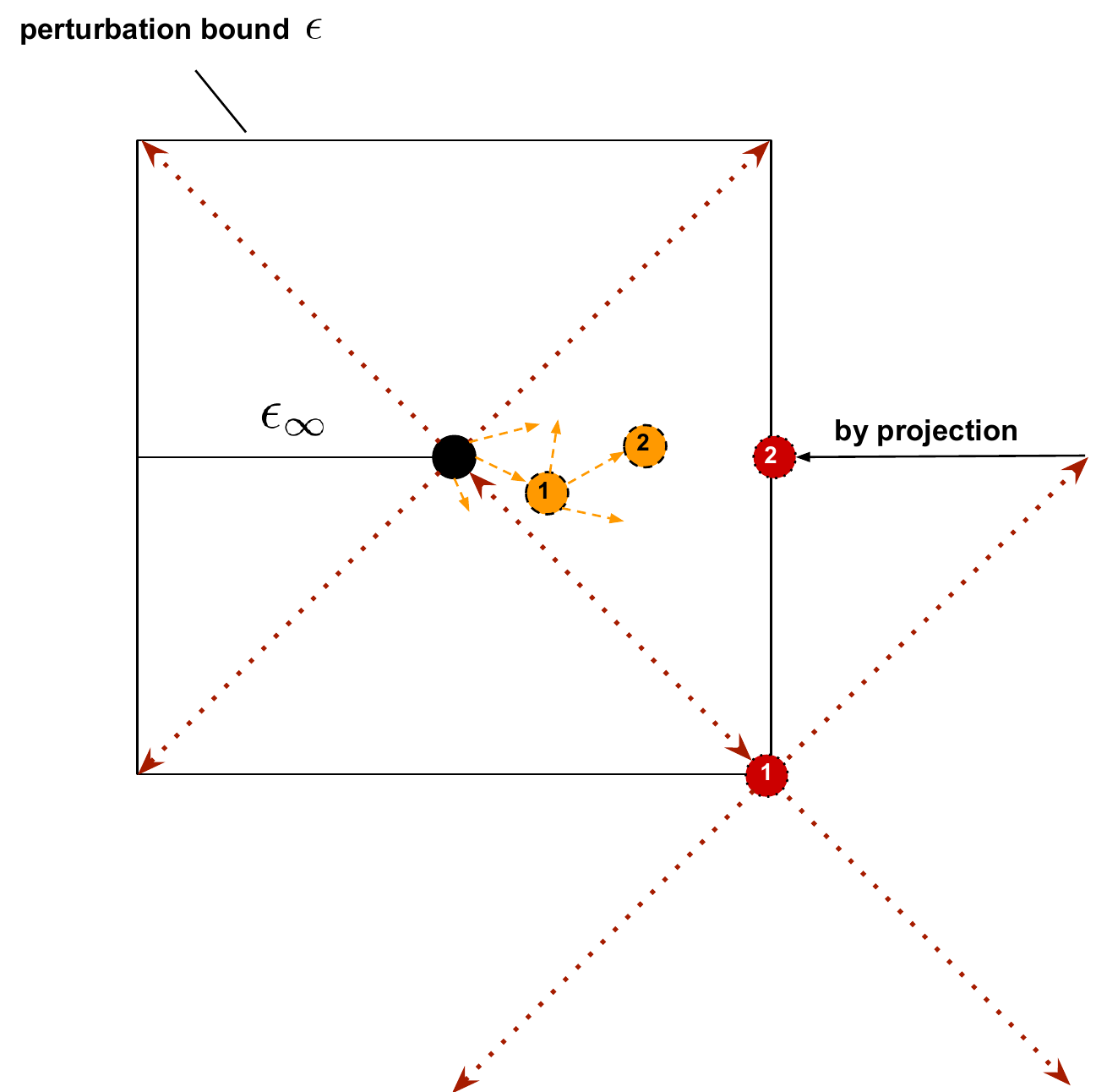} &
				\includegraphics[width=0.33\textwidth]{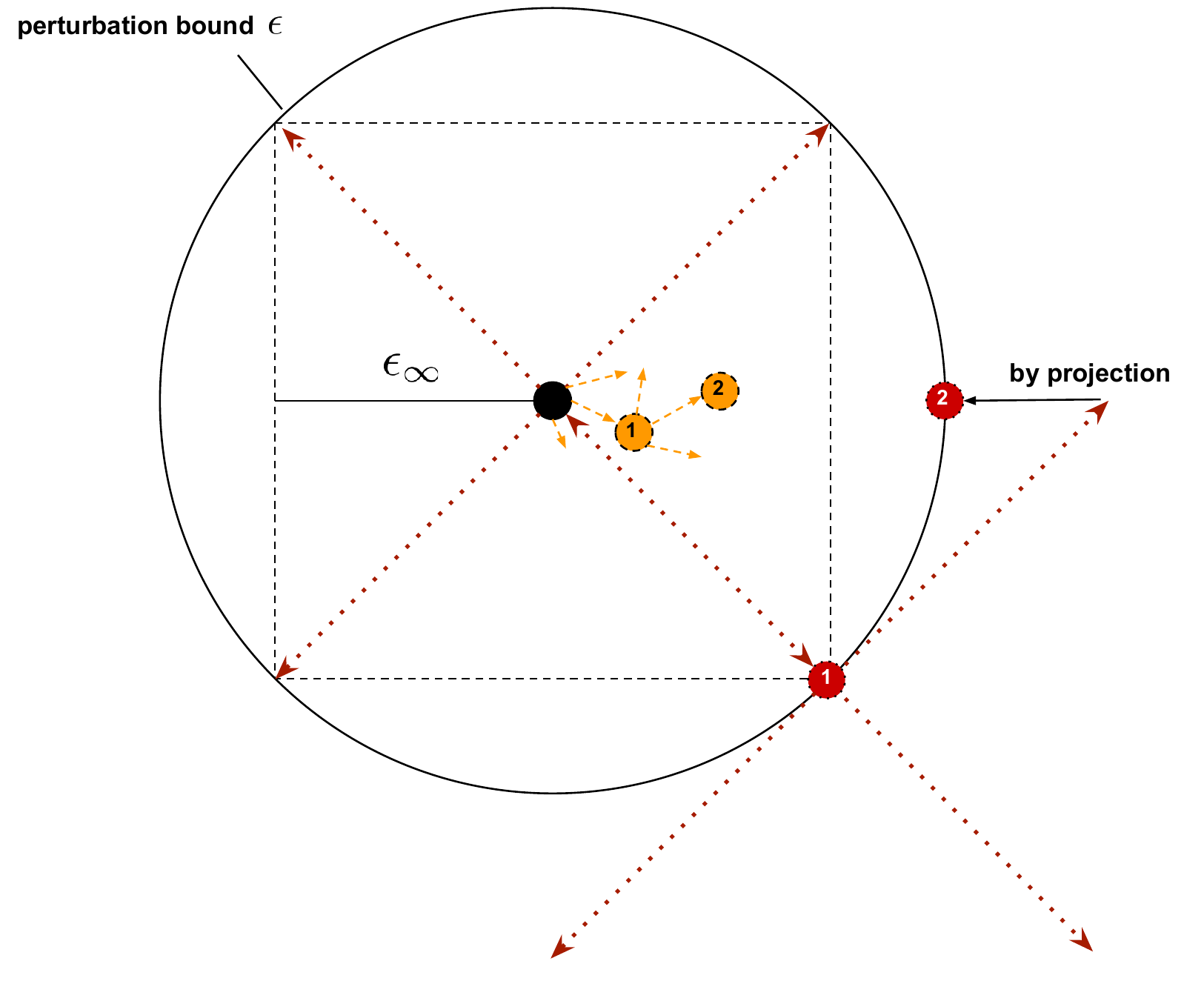} 
				&
				\includegraphics[width=0.21\textwidth]{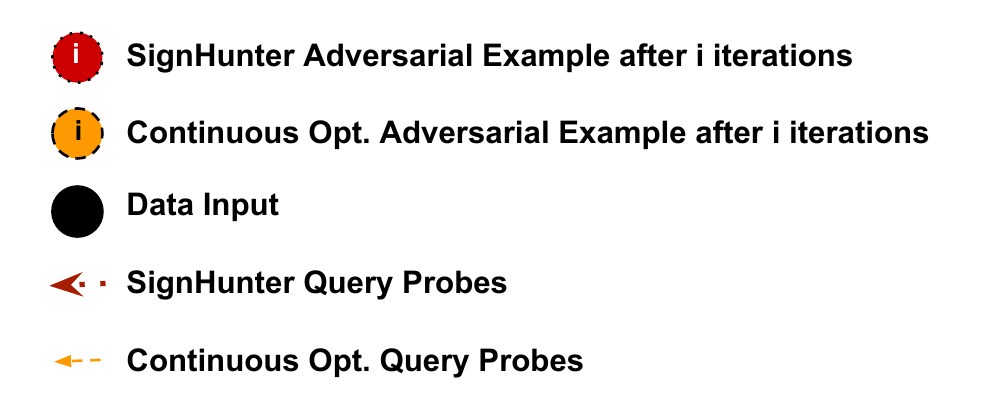} \\
				(a) $\linf$ perturbation &  (b) $\ltwo$ perturbation &\\
			\end{tabular}
		}
		\caption{Illustration of adversarial examples crafted by \signhunter in comparison to attacks that are based on the continuous optimization in both (a) $\linf$ and (b) $\ltwo$ settings. If \signhunter is given a query budget $> 2n$, which is the case here, the crafted adversarial examples are not necessary at the perturbation vertices, e.g., the red ball 2. We can modify \signhunter to strictly look up perturbation vertices. This could be done by doubling the step size from $\epsilon_\infty$
			to $2\epsilon_\infty$ and we leave this for future work as outlined in Section~\ref{sec:open}.
		}
		\label{fig:perturb-illust}
	\end{figure*}

\section{Open Questions}
\label{sec:open}

\begin{figure*}[t!]
	\centering
	\resizebox{\textwidth}{!}{
		\begin{tabular}{cccc}
			\includegraphics{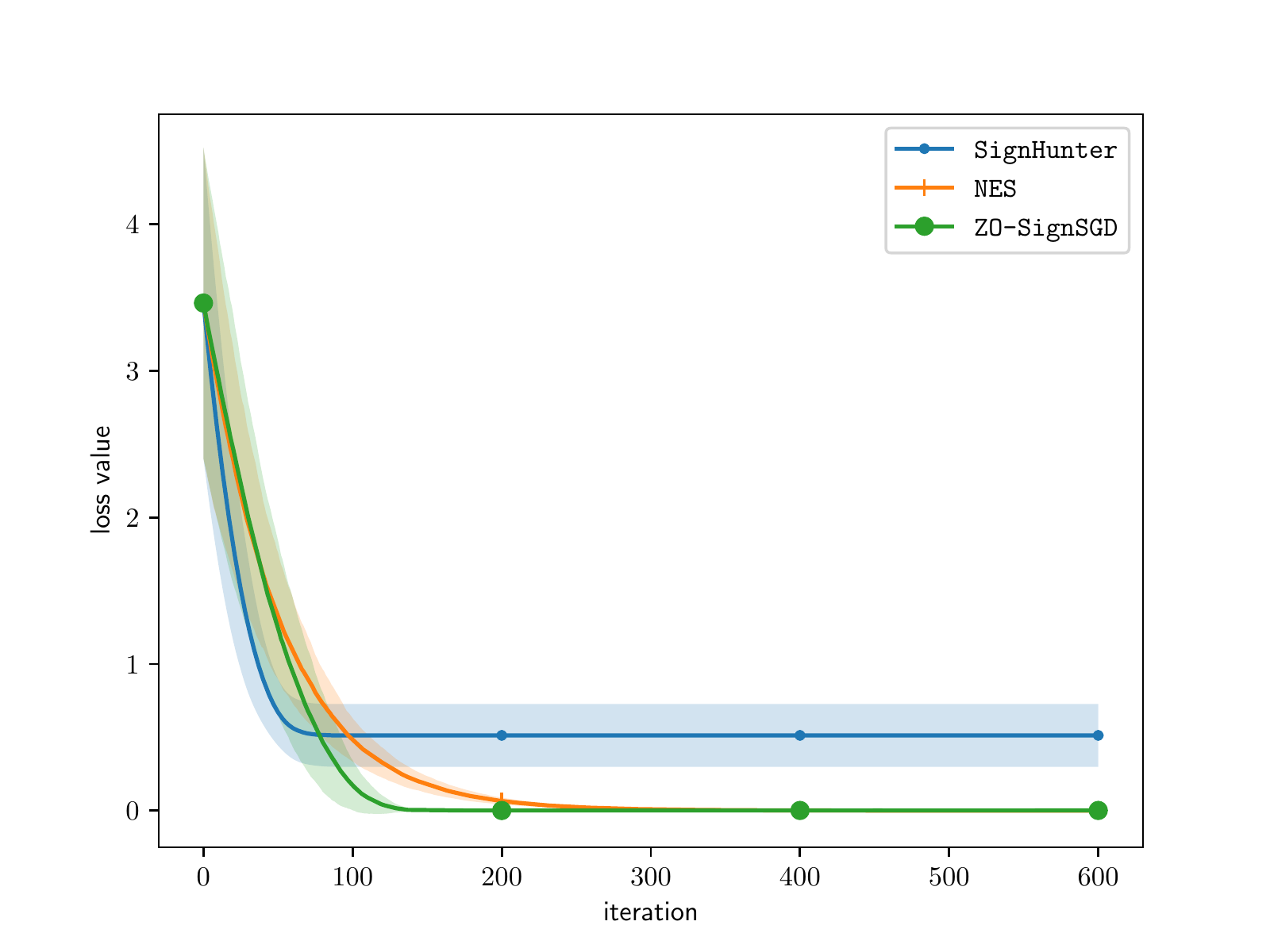} &
			\includegraphics{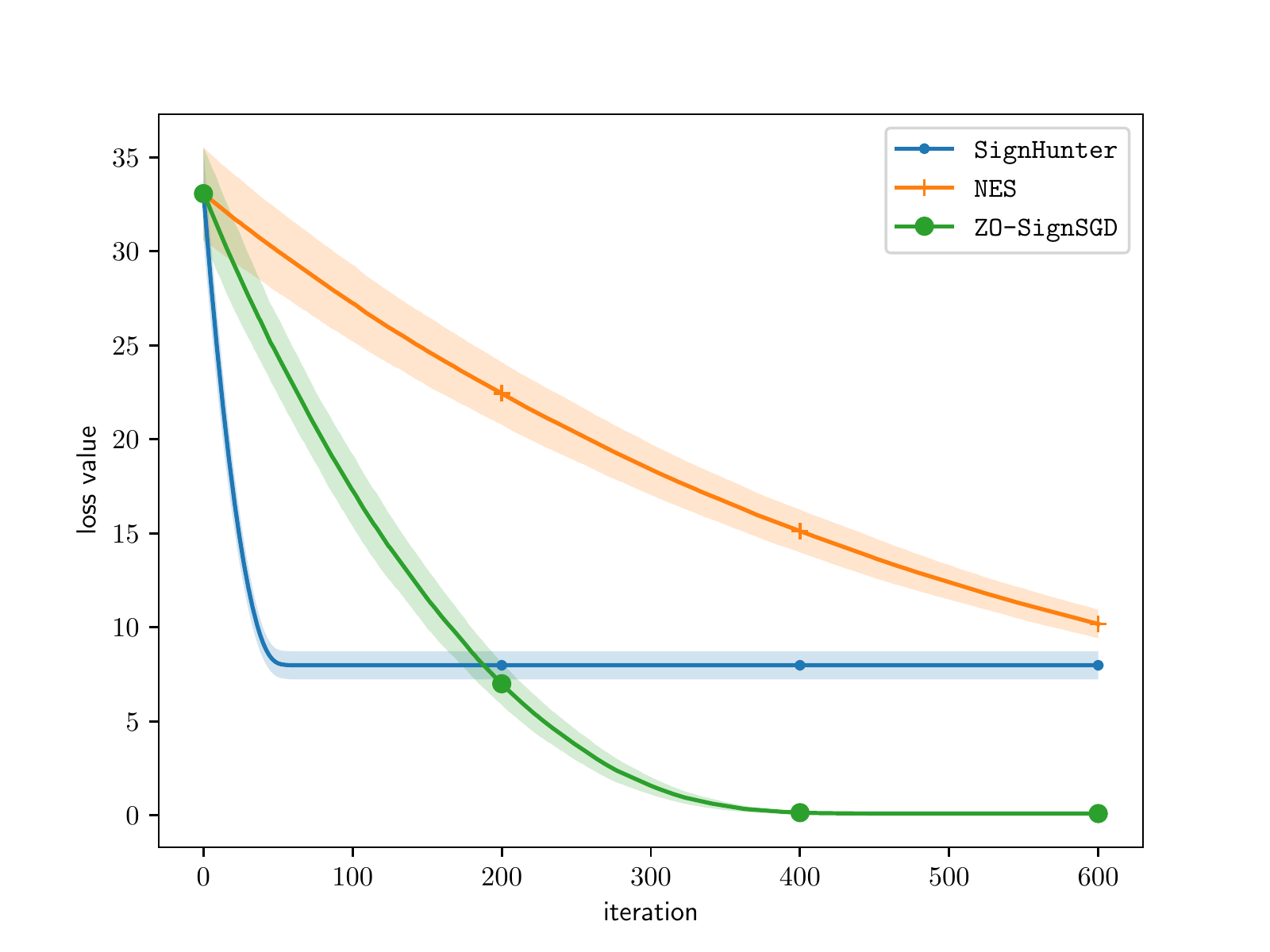} &
			\includegraphics{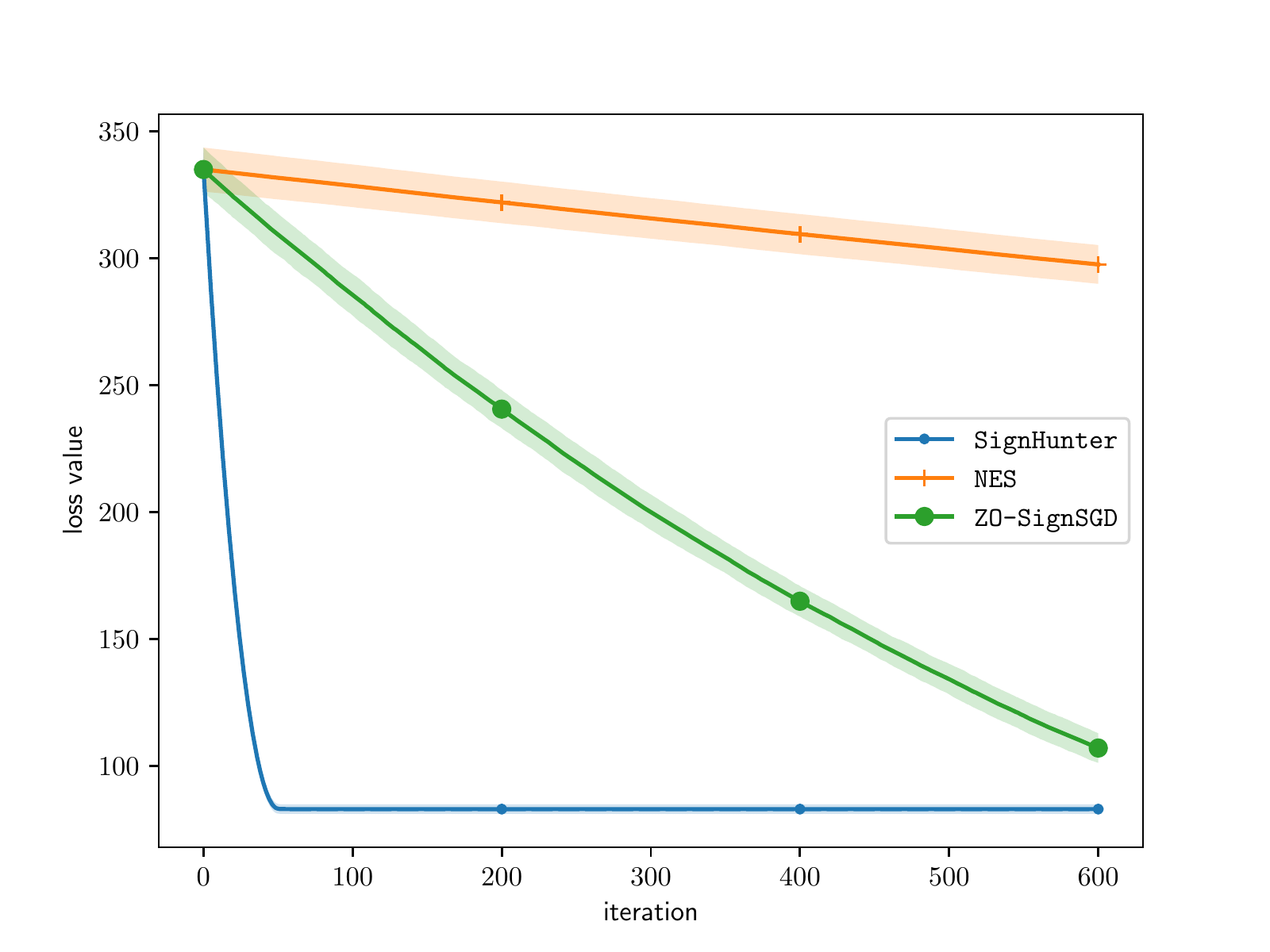}
			& 
			\includegraphics{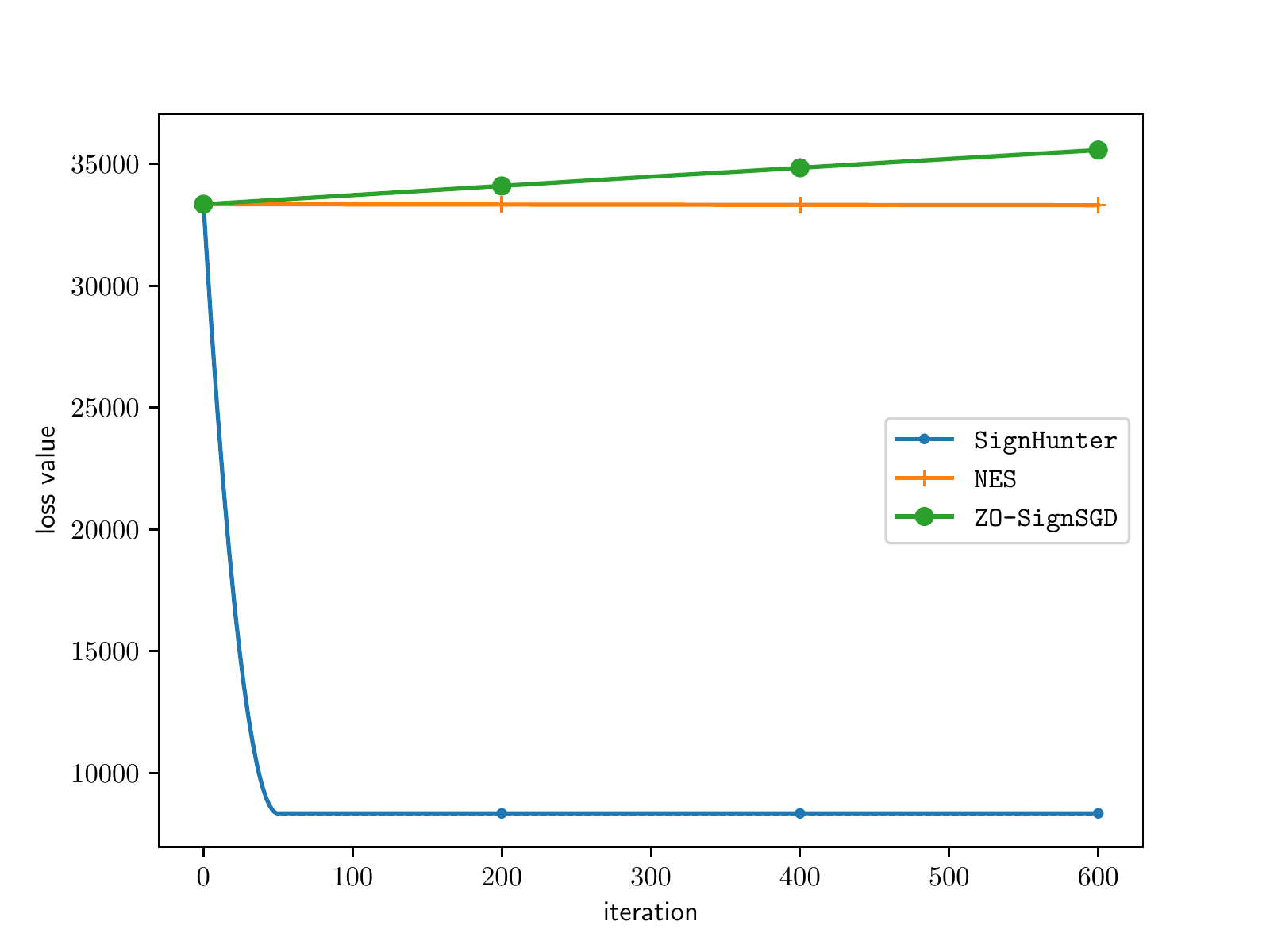}\\
			\Huge (a) $n=10$ & \Huge (b) $n=100$ & \Huge (c) $n=1,000$ &  \Huge (d) $n=100,000$\\
		\end{tabular}}
		\caption{ \emph{\signhunter for continuous optimization.} In this basic experiment, we run \nes, \zo and \signhunter to minimize a function $f:\mathbb{R}^n \to \mathbb{R}$ of the form $||\vx - \vx^*||^2_2$ for $n\in \{10, 100, 1000, 100000\}$. The solid line represents the loss averaged over 30 independent
			trials with random $\vx^* \sim \calU([0,1]^n)$ and the shaded region indicates the standard deviation of results over random
			trials.  We used a fixed step size of $0.01$ in line with \citep{liu2018signsgd} and a finite difference perturbation of $0.001$.  The starting point $\vx^{(0)}$ for all the algorithms was set to be the all-one vector~$\mathbf{1}_n$. 
		} 
		\label{fig:cont-opt}
	\end{figure*}

There are many interesting questions left open by our research:

\paragraph{Priors.} Current version of \signhunter does not exploit any data- or time-dependent priors. With these priors, algorithms such as \bandit operate on a search space of dimensionality~$\sim36\times$ less than that of \signhunter for \imgnt. In domain-specific examples such as images, \emph{can Binary Partition Trees (BPT)~\citep{al2015graphbpt} be incorporated  in \signhunter to have a data-dependent  grouping of gradient coordinates instead of the current equal-size grouping?}

\paragraph{Adversarial Training.} Compared to other attacks that are based on transferability and generative adversarial networks, our approach showed more effectiveness towards (ensemble) adversarial training. Standard adversarial training relies on attacks that employ iterative continuous optimization methods such as \texttt{PGD} in contrast to our attack which stems from a binary optimization view. \emph{What are the implications?}

\paragraph{Other Domains.} Much of the work done to understand and counter adversarail examples has occurred in the image classification domain. The binary view of our approach lends itself naturally to other domains where binary features are used (e.g., malware detection~\citep{al2018adversarial,luca2019explaining}). \emph{How effective our approach is on these domains?}

\paragraph{Perturbation Vertices.}\footnote{We define perturbation vertices as  extreme points of the  region $B_p(\vx, \epsilon)$. That is,  $\vx\pm \epsilon_\infty$, where $\epsilon_\infty = \epsilon$ when $p=\infty$ and $\epsilon_\infty = \epsilon / \sqrt{n}$ when $p=2$. See Figure~\ref{fig:perturb-illust}.} Using its first $O(n)$ queries, \signhunter probes $O(n)$ extreme points of the perturbation region as potential adversarial examples, while iterative continuous optimization such as \nes probes points in the Gaussian sphere around  the current point as shown in Figure~\ref{fig:perturb-illust}. \emph{Does looking up extreme points (vertices) of the perturbation region suffice to craft adversarial examples? If that is the case, how to efficiently search through them?} \signhunter searches through $2n$ vertices out of $2^n$ and it could find adversarial examples among a tiny fraction of these vertices. Recall, in the \mnist $\linf$ setup in Section~\ref{sec:experiments}, it was enough to look up just $\sim 12$ out of $2^{784}$ vertices for each image achieving a $100\%$ evasion over $1,000$ images. Note that after $2n$ queries, \signhunter may not visit other vertices as they will be $2\epsilon_\infty$ away as shown in Figure~\ref{fig:perturb-illust}. We ignored this effect in our experiments.\footnote{This effect is negligible for \imgnt as $2n < 10,000$.} \emph{Will \signhunter be more effective if the probes are made strictly at the perturbation vertices?} This question shows up clearly in the public MNIST challenge where the loss value at the potential adversarial examples dips after every $\sim2n$ queries (see top left plot of Figure~13 in Appendix!F). We conjecture the reason is that these potential adversarial examples are not extreme points as illustrated in Figure~\ref{fig:perturb-illust}: they are like the red ball~2 rather than the red ball~1.%Does there exist an analogy here to the \emph{simplex} and \emph{interior point} methods of linear programming?

\paragraph{\signhunter for Black-Box Continuous Optimization.} In~\citep{salimans2017evolution, chrabaszcz2018back}, it was shown that a class of black-box continuous optimization algorithms (\nes as well as a very basic canonical \texttt{ES} algorithm) rival the performance of standard reinforcement learning techniques. On the other hand, \signhunter is tailored towards recovering the gradient sign bits and creating adversarial examples similar to \fgsm using the best gradient sign estimation obtained so far. Can we incorporate \signhunter in an iterative framework for continuous optimization? Figure~\ref{fig:cont-opt} shows a small, preliminary experiment comparing \nes and \zo to a simple iterative framework employing \signhunter. In the regime of high dimension/few iterations, \signhunter can be remarkably faster. However, with more iterations, the algorithm fails to improve further and starts to oscillate. The reason is that \signhunter always provides $\pm 1$ updates (non-standard sign convention) compared to the other algorithms whose updates can be zero. \emph{Can we get the best of both worlds?}

\section{Conclusion}
\label{sec:conclusion}

Assuming a \emph{black-box} threat model, we studied the problem of generating adversarial examples for neural nets % Motivated by i) the significant empirical effectiveness of gradient sign information; and ii) the low query complexity of recovering a sign vector  using a noiseless Hamming distance oracle, 
and proposed the gradient \emph{sign} estimation problem as the core challenge in crafting these examples. We formulate the problem as a \emph{binary black-box optimization} one: minimizing the Hamming distance to the gradient sign or, equivalently, maximizing the directional derivative. Approximated by the finite difference of the loss value queries, we examine three properties of the directional derivative of the model's loss in the direction of $\{\pm1\}^n$ vectors. 
%Based on the first property, the loss oracle can be used as a noisy Hamming distance oracle. We found that current search Hamming search strategies (e.g. \citet{maurer2009search}) are not suitable for such oracles. The second property lets us employ the \emph{optimism in the face of uncertainty principle} in the form of hierarchical bandits. This resulted in \goo, an optimistic optimization algorithm for binary black-box optimization problems with a finite-time analysis on its regret. However, its query complexity is worse than the continuous optimization setup.  
The separability property helped us devise \signhunter, a divide-and-conquer algorithm that is guaranteed to perform \emph{at least} as well as \fgsm after $O(n)$ queries. In practice, \signhunter needs a fraction of this number of queries to craft adversarial examples. To verify its effectiveness on real-world datasets, \signhunter was evaluated on neural network models for the \mnist, \cifar, and \imgnt datasets. \signhunter yields black-box attacks that
are $2.5\times$ more query efficient and $3.8\times$ less failure-prone than the state of the art attacks combined. Moreover, \signhunter achieves the highest evasion rate on
two public black-box attack challenges. We also show that models that are robust against substitute-model attacks are vulnerable to our attack.

\section*{Acknowledgements}
This work was supported by the MIT-IBM Watson AI Lab. We would like to thank
Shashank Srikant for his timely help. We are grateful for feedback from Nicholas Carlini. 

%\textbf{Do not} include acknowledgements in the initial version of
%the paper submitted for blind review.

%If a paper is accepted, the final camera-ready version can (and
%probably should) include acknowledgements. In this case, please
%place such acknowledgements in an unnumbered section at the
%end of the paper. Typically, this will include thanks to reviewers
%who gave useful comments, to colleagues who contributed to the ideas,
%and to funding agencies and corporate sponsors that provided financial
%support.


\onecolumn
\section*{Appendix A. Noisy \fgsm}

This section shows the performance of the noisy \fgsm on standard models (described in Section 5 of the main paper) on the \mnist, \cifar and \imgnt datasets. In Figure~\ref{fig:keep_k_signs_linf}, we consider the $\linf$ threat perturbation constraint. Figure~\ref{fig:keep_k_signs_l2} reports the performance for the $\l2$ setup. Similar to~\cite{ilyas2018prior}, for each $k$ in the experiment, the top $k$ percent of the signs of the coordinates---chosen either randomly (\texttt{random-k}) or by the corresponding magnitude $|\partial L(\vx, y)/\partial x_i|$~(\texttt{top-k})---are set correctly, and the rest are set to $-1$ or $+1$ at random. The misclassification rate shown considers only images that were correctly classified (with no adversarial perturbation). In accordance with the models' accuracy, there were $987$, $962$, and $792$ such images for \mnist, \cifar, and \imgnt out of the sampled $1000$ images, respectively. These figures also serve as a validation for Theorem~\ref{thm:signhunter} when compared to \signhunter's performance shown in Appendix C.

\begin{figure*}[h!]
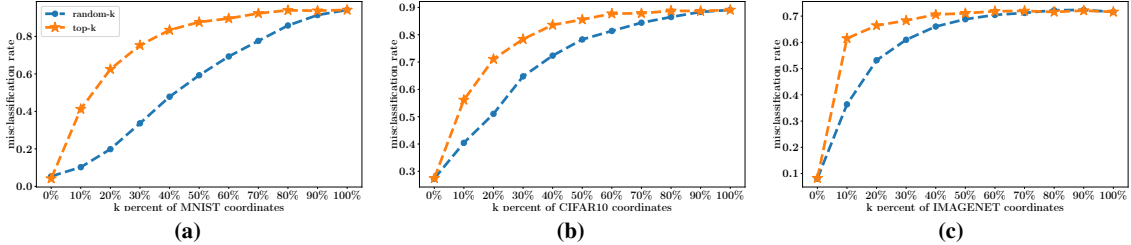

	\centering
	\resizebox{0.9\textwidth}{!}{
		\begin{tabular}{ccc}
			\includegraphics[width=0.32\textwidth, clip, trim=0.95cm 0.9cm 0.8cm 0.1cm]{figs/keep_k/keep_k_sign_mnist_linf.pdf} &
			\includegraphics[width=0.32\textwidth,clip, trim=0.95cm 0.9cm 0.8cm 0.1cm]{figs/keep_k/keep_k_sign_cifar10_linf.pdf} &
			\includegraphics[width=0.32\textwidth,clip, trim=0.95cm 0.9cm 0.8cm 0.1cm]{figs/keep_k/keep_k_sign_imagenet_linf.pdf} \\
			{ \textbf{(a)}} & { \textbf{(b)}} & { \textbf{(c)}}
			\vspace*{-2mm}
		\end{tabular}
	}
	\caption{ \small Misclassification rate of three neural nets (for (a) \mnist, (b) \cifar, and (c) \imgnt, respectively) on the \emph{noisy} \fgsm's adversarial examples as a function of correctly estimated coordinates of $\sgn(\nabla_\vx f(\vx, y))$ on random $1000$ images from the corresponding evaluation dataset, with the maximum allowed $\linf$ perturbation $\epsilon$ being set to $0.3$, $12$, and $0.05$, respectively.  Across all the models, estimating the sign of the top $30\%$ gradient coordinates (in terms of their magnitudes)  is enough to achieve a misclassification rate of $\sim70\%$. Note that Plot (c) is similar to \cite{ilyas2018prior}'s Figure 1, but it is produced with \textsf{TensorFlow} rather than \textsf{PyTorch}.
	}
	\label{fig:keep_k_signs_linf}
\end{figure*}

\begin{figure*}[h!]
	\centering
	\resizebox{0.9\textwidth}{!}{
		\begin{tabular}{ccc}
			\includegraphics[width=0.32\textwidth, clip, trim=0.95cm 0.9cm 0.8cm 0.1cm]{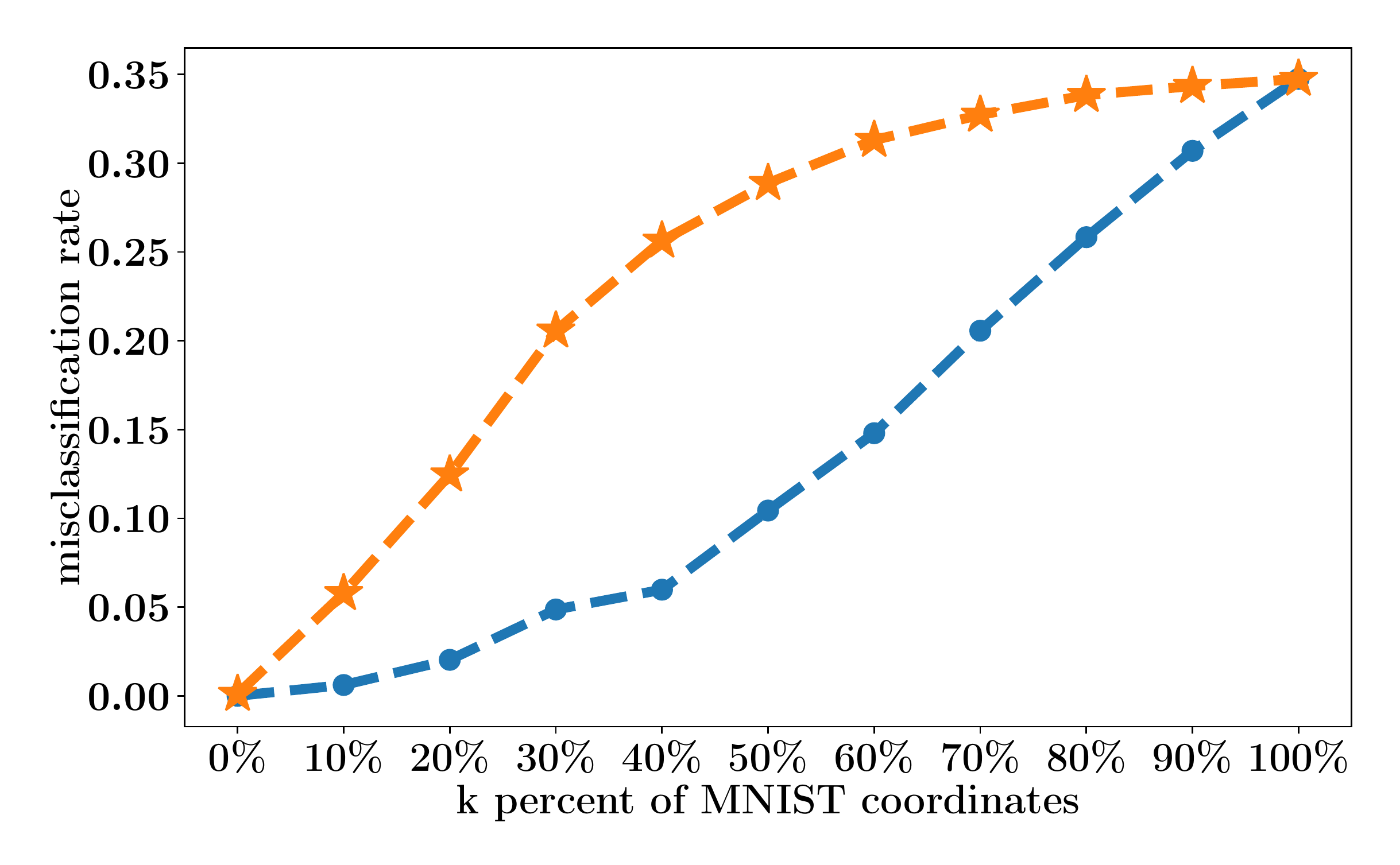} &
			\includegraphics[width=0.32\textwidth,clip, trim=0.95cm 0.9cm 0.8cm 0.1cm]{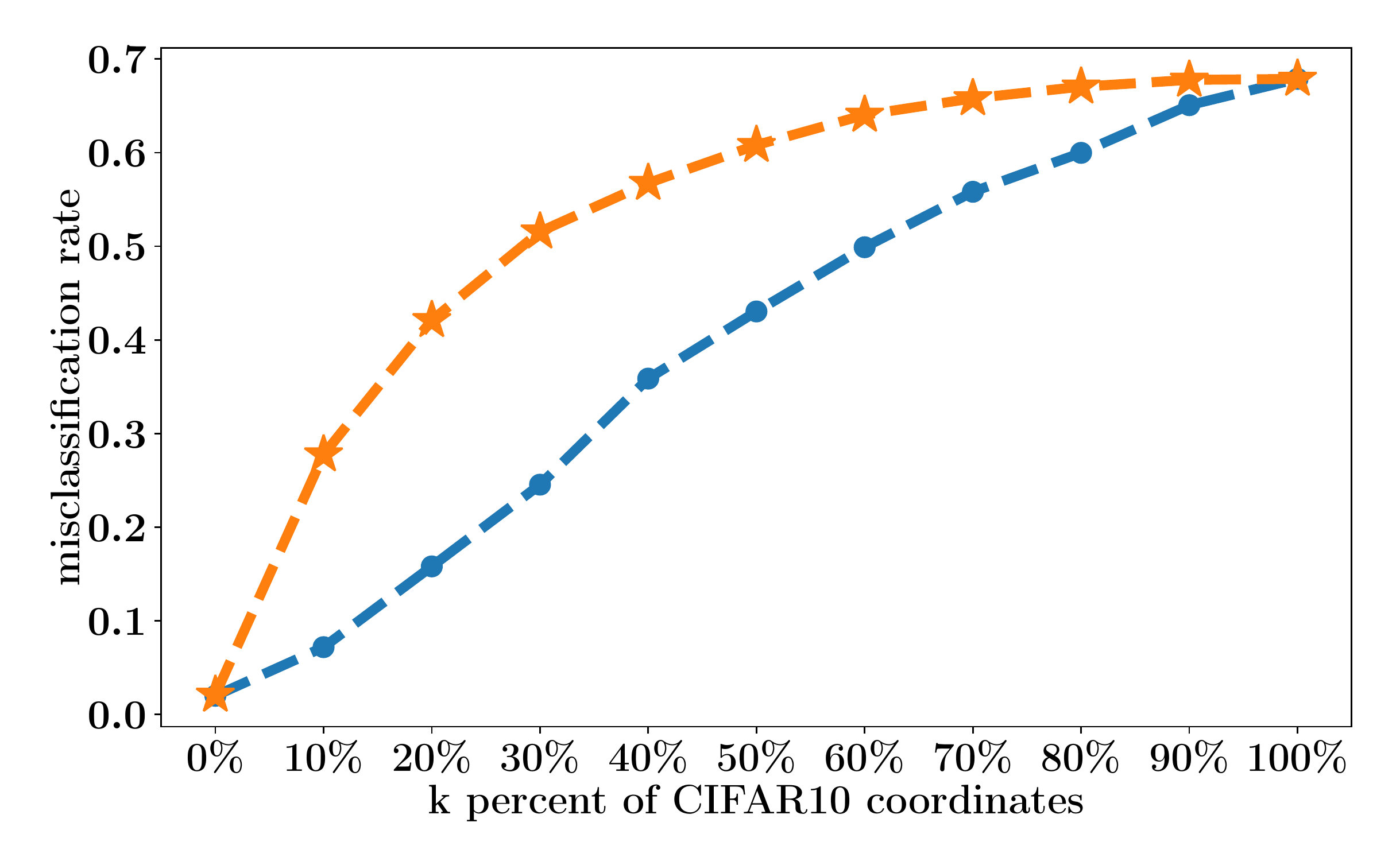} &
			\includegraphics[width=0.32\textwidth,clip, trim=0.95cm 0.9cm 0.8cm 0.1cm]{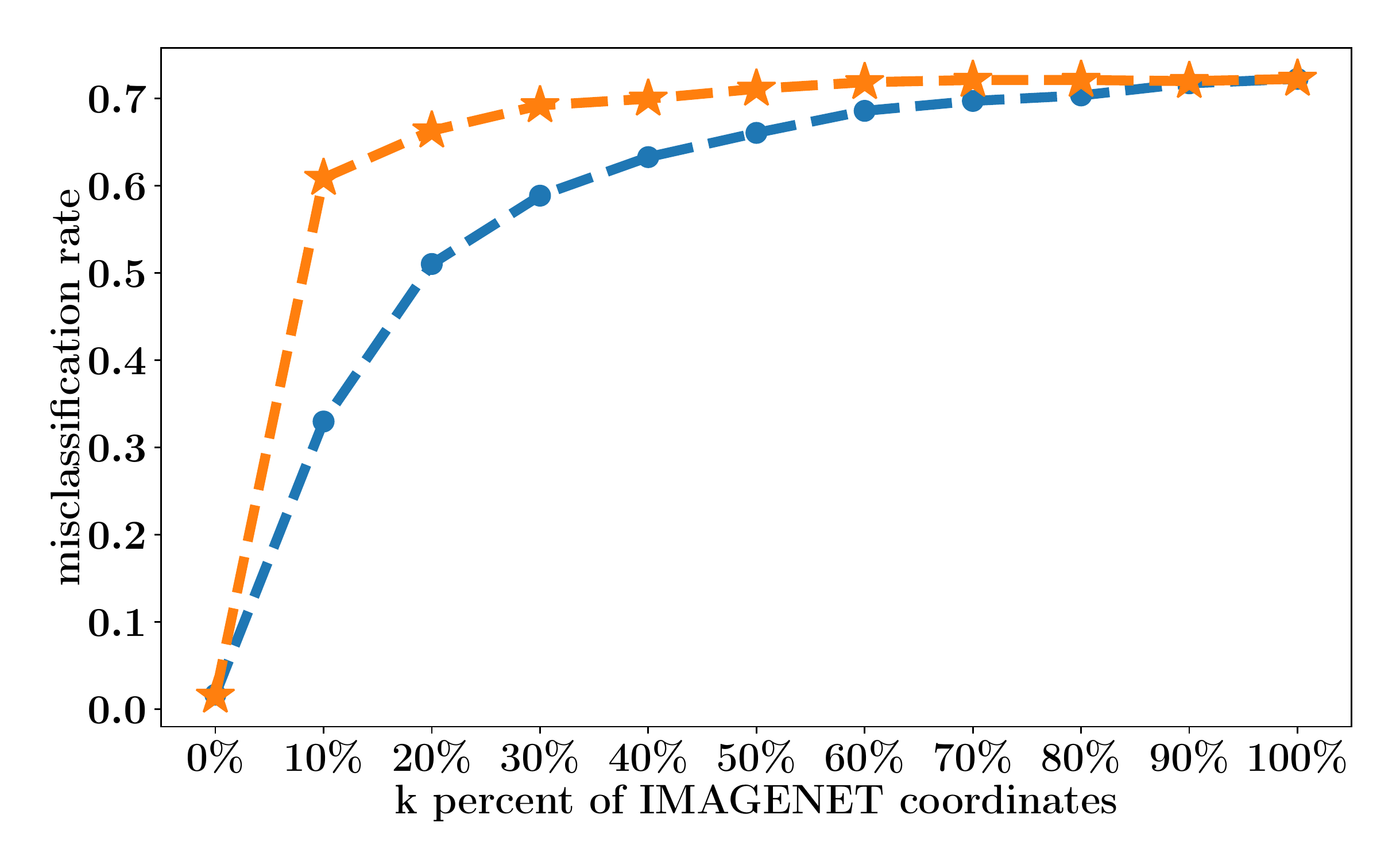} \\
			{ \textbf{(a)}} & { \textbf{(b)}} & { \textbf{(c)}}
			\vspace*{-2mm}
		\end{tabular}
	}
	\caption{ \small Misclassification rate of three neural nets (for (a) \mnist, (b) \cifar, and (c) \imgnt, respectively) on the \emph{noisy} \fgsm's adversarial examples as a function of correctly estimated coordinates of $\sgn(\nabla_\vx f(\vx, y))$ on random $1000$ images from the corresponding evaluation dataset, with the maximum allowed $\ell_2$ perturbation $\epsilon$ being set to $3$, $127$, and $5$, respectively. Compared to~Figure~\ref{fig:keep_k_signs_linf}, the performance on \mnist and \cifar drops significantly.
	}
	\label{fig:keep_k_signs_l2}
\end{figure*}
\newtheorem{innercustomprop}{Property}
\newenvironment{customprop}[1]
{\renewcommand\theinnercustomprop{#1}\innercustomprop}
{\endinnercustomprop}

\newtheorem{innercustomdefn}{Definition}
\newenvironment{customdefn}[1]
{\renewcommand\theinnercustomdefn{#1}\innercustomdefn}
{\endinnercustomdefn} 
\twocolumn
\section*{Appendix B. A Framework for Estimating Sign of the Gradient from Loss Oracles}
\label{sec:methods}

Our interest in this section is to estimate the gradient sign bits of the loss function $L$ of the model under attack at an input/label pair ($\vx, y$) from a limited number of loss value queries $L(\vx^\prime, y)$. To this end, we examine the basic concept of directional derivatives that has been employed in recent black-box adversarial attacks. Particularly, we present three approaches to estimate the gradient sign bits based on three properties of the directional derivative  $D_\vq L(\vx, y)$ of the loss in the direction of a sign vector $\vq \in \calH$.

\subsection*{Approach 1: Divide \& Conquer}

Covered in the main article.

\subsection*{Approach 2: Loss Oracle as a Noisy Hamming Oracle}

The directional derivative of the loss function $L$ at $(\vx,y)$ in the direction of a binary code $\vq$ can be written as 
\begin{eqnarray}
D_\vq L(\vx, y) &=& \vq^T \vg^* \nonumber \\
&=& \sum_{i \in \calI^+_\vq} |g^*_i| -  \sum_{i \in \calI^-_\vq} |g^*_i| \nonumber \\
&=& |\calI^+_q| \bar{g}_{\calI^+_\vq} -  |\calI^-_q| \bar{g}_{\calI^-_\vq}\;, \label{eq:dir-deriv-ham-dist}
\end{eqnarray}
where $\calI^+_\vq \equiv \{i \mid i \in [n]\;,\;q^*_i= q_i \}$, $\calI^-_\vq\equiv [n]\setminus \calI^+_\vq$. Note that $|\calI^+_q| + |\calI^-_q| =n$. The quantities $\bar{g}_{\calI^+_\vq}$ and $\bar{g}_{\calI^-_\vq}$ are the means of $\{|g_i|\}_{i\in \calI^+_\vq}$ and $\{|g_i|\}_{i\in \calI^-_\vq}$, respectively. Observe that $|\calI^-_\vq|=\ham{\vq - \vq^*}:$ the Hamming distance between $\vq$ and the gradient sign $\vq^*$. In other words, the directional derivative $D_\vq L(\vx, y) $ has the following property.
\begin{customprop}{2}	\label{prop:hamming} The directional derivative $D_\vq L(\vx, y)$ of the loss function $L$ at an input/label pair $(\vx,y)$ in the direction of a binary code $\vq$ can be written as an affine transformation of the Hamming distance between $\vq$ and $\vq^*$. Formally, we have
	\begin{equation}
	D_\vq L(\vx, y) = n \bar{g}_{\calI^+_\vq} -  (\bar{g}_{\calI^-_\vq} +  \bar{g}_{\calI^+_\vq}) \ham{\vq - \vq^*} 
	\label{eq:prop1}
	\end{equation}
\end{customprop}

If we can recover the Hamming distance from the directional derivative based on~\eqref{eq:prop1}, efficient Hamming search strategies---e.g., \citep{maurer2009search}---can then be used to recover the gradient sign bits $\vq^*$ with a query complexity $\Omega(n/ \log_2(n+1))$ as stated in Theorem~1. However, not all terms of~\eqref{eq:prop1} is known to us. While $n$ is the number of data features (known a priori) and $D_\vq L(\vx, y)$ is available through a finite difference oracle, $\bar{g}_{\calI^+_\vq}$ and $\bar{g}_{\calI^-_\vq}$ are not known. Here, we propose to approximate these values by their Monte Carlo estimates: averages of the magnitude of sampled gradient components. Our assumption is that the magnitudes of gradient coordinates are not very different from each other, and hence a Monte Carlo estimate is good enough (with small variance). Our experiments on \mnist, \cifar, and \imgnt confirm the same---see Figure~18 in Appendix G.

To use the $i$th gradient component $g^*_i$ as a sample for our estimation, one can construct two binary codes $\vu$ and $\vv$ such that \emph{only} their $i$th bit is different, i.e., $\ham{\vu - \vv} = 1$. Thus, we have
\begin{eqnarray}
|g^*_i|&=& \frac{|D_\vu L(\vx, y) - D_\vv L(\vx, y)|}{2} \label{eq:mc-abs-val}
\\
q^*_i = \sgn(g^*_i)&=&  
\begin{cases}
u_i& \text{if } D_\vu L(\vx, y) > D_\vv L(\vx, y)\;,\\
v_i              & \text{otherwise}\;.
\end{cases}
\label{eq:mc-sign}
\end{eqnarray}
Let $\calD$ be the set of indices of gradient components we have recovered---magnitude and sign---so far through~\eqref{eq:mc-abs-val} and~\eqref{eq:mc-sign}. Then, 
\begin{eqnarray}
\bar{g}_{\calI^+_\vq}\approx \frac{1}{|\calD^+_\vq| } \sum_{d \in \calD^+_\vq} |g^*_d| \;,\\
\bar{g}_{\calI^-_\vq}\approx  \frac{1}{|\calD^-_\vq| } \sum_{d \in \calD^-_\vq} |g^*_d| \;,
\end{eqnarray}
where $\calD^+_\vq \equiv \{d \mid d \in \calD\;,\;  q^*_d = q_i\}$ and $\calD^-_\vq \equiv \calD \setminus \calD^+_\vq$.\footnote{It is possible that one of $\calD^+_\vq$ and $\calD^-_\vq$ will $\emptyset$ (e.g., when we only have one sample). In this case, we make the approximation as $\bar{g}_{\calI^+_\vq} = \bar{g}_{\calI^-_\vq} \approx \frac{1}{|\calD| } \sum_{d \in \calD} |g^*_d|\;$.} As a result, the Hamming distance between $\vq$ and the gradient sign $\vq^*$ can be approximated with the following quantity, which we refer to as the \emph{noisy} Hamming oracle $\hat{\calO}$.
\begin{equation}
\ham{\vq - \vq^*} \;\approx \;  \frac{ \frac{n}{|\calD^+_\vq| } \sum_{d \in \calD^+_\vq} |g^*_d| - D_\vq L(\vx, y)}{ 
	\frac{1}{|\calD^+_\vq| } \sum_{d \in \calD^+_\vq} |g^*_d| + \frac{1}{|\calD^-_\vq| } \sum_{d \in \calD^-_\vq} |g^*_d|
}
\label{eq:ham-dist-approx}
\end{equation}

\begin{figure*}[t!]
	\centering
	\resizebox{0.9\textwidth}{!}{
		\begin{tabular}{cc}
			\includegraphics[width=0.25\textwidth,clip,]{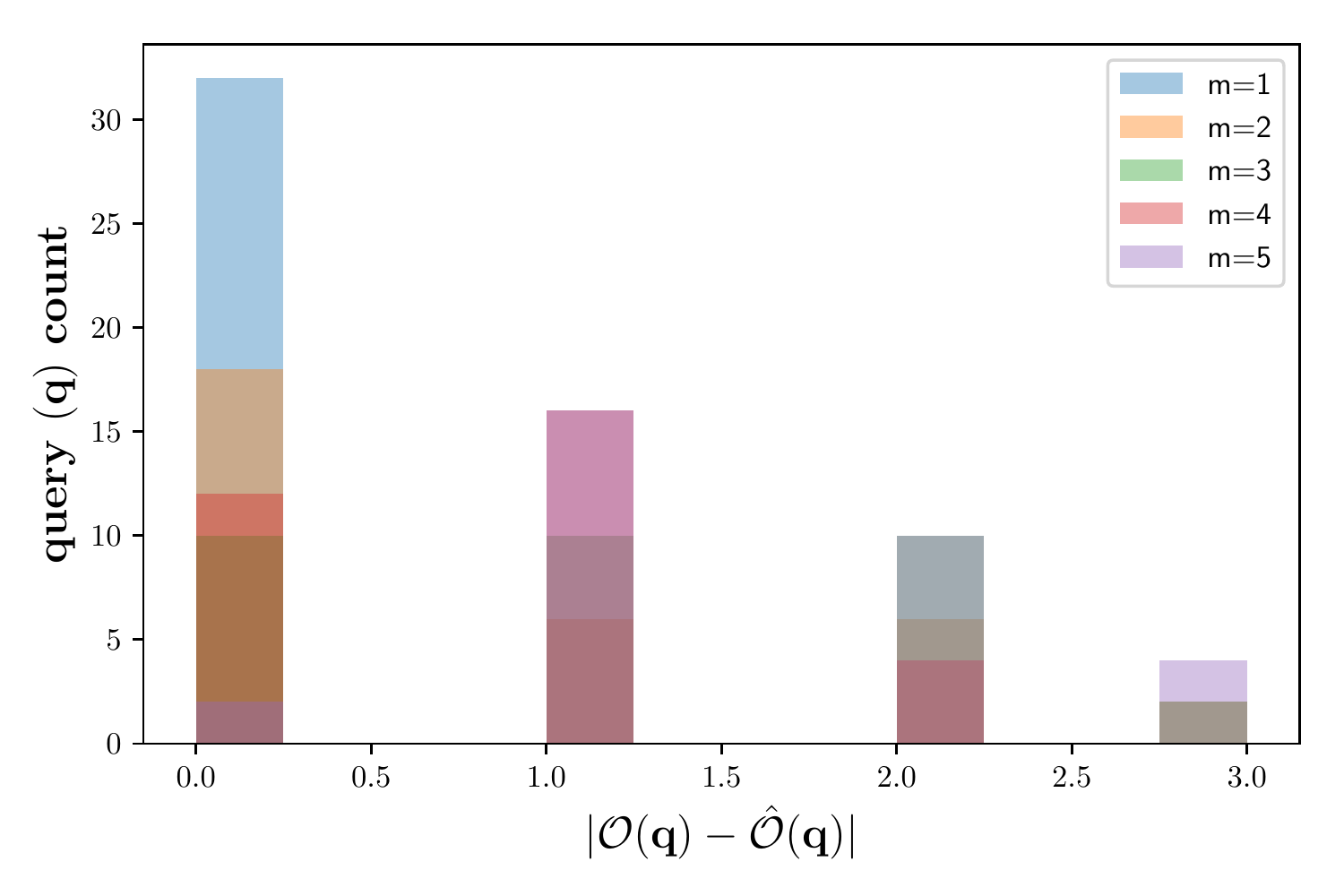} &
			\includegraphics[width=0.25\textwidth]{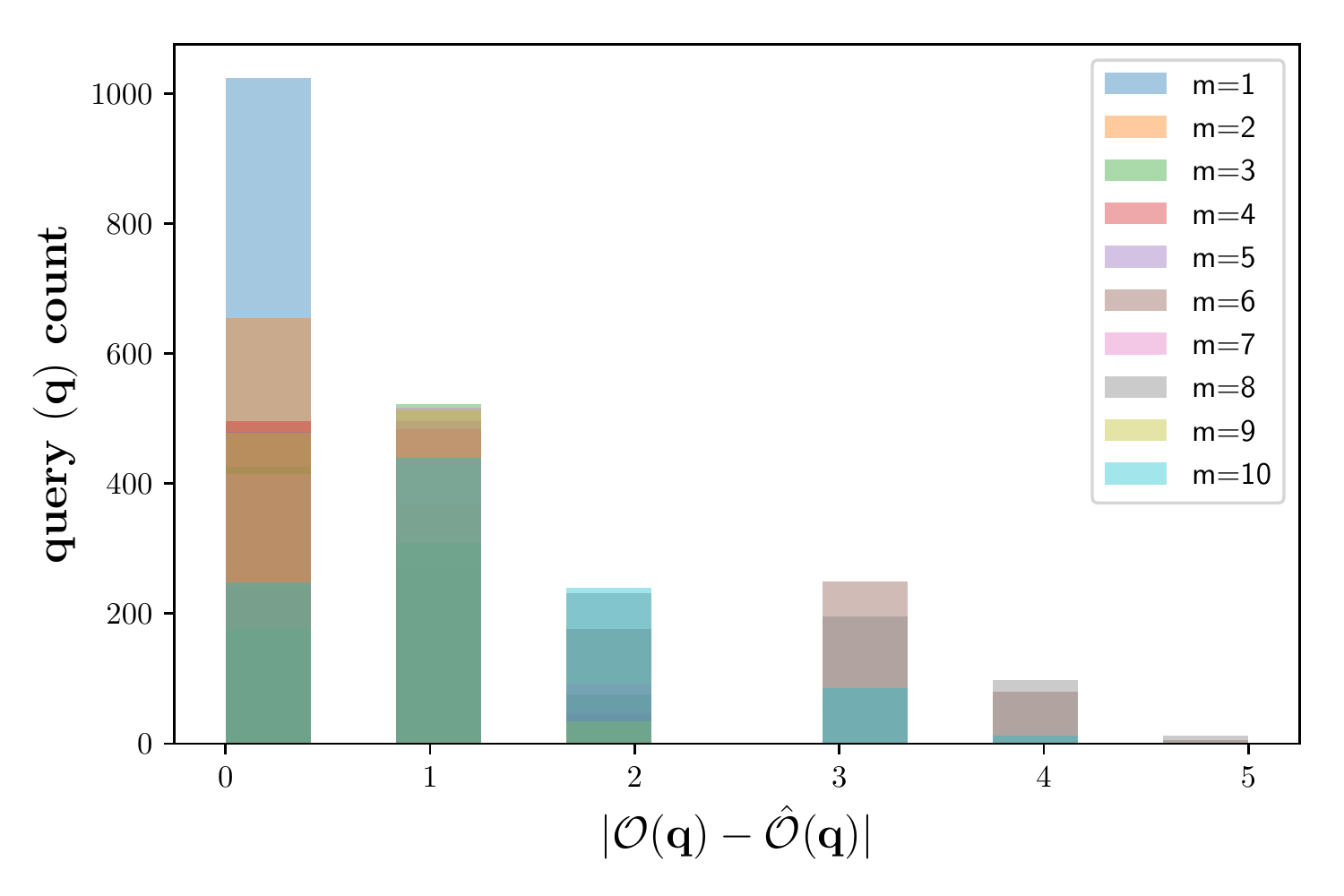}\\
			\tiny (a) $n=5$ & \tiny (b) $n=10$
		\end{tabular}
	}
	\caption{The error distribution of the \emph{noisy} Hamming oracle $\hat{\calO}$ (right side of~\eqref{eq:ham-dist-approx}) compared to the 
		\emph{noiseless} counterpart $\calO$ (left side of~\eqref{eq:ham-dist-approx}) as a function of the number of unique values (magnitudes) of the gradient coordinates $m$. Here, $L(\vx, y)$ has the form $\mathbf{c}^T\vx$. That is, $m=|\texttt{uniq}(|\mathbf{c}|)| \leq n$ with $n\in \{5, 10\}$ being the input length. With $m=1$, the estimation is exact ($\hat{\calO} = \calO$) for all the binary code queries $\vq$---32 codes for $n=5$, 1028 codes for $n=10$. The error seems to be bounded by $\lceil \frac{n}{2}\rceil$. For a given $m$, $c_i$---the $i^{th}$ coordinate of $\mathbf{c}$---is randomly assigned a value from the $m$ evenly spaced numbers in the range $[0.1, m/n ]$. We set the size of the sampled gradient coordinates set~$|\calD|$ to $\lfloor n / 4 \rfloor$.}
	\label{fig:hd-est-quality}
\end{figure*}

We empirically evaluated the quality of $\hat{\calO}$'s responses on a toy problem where we controlled the magnitude spread/concentration of the gradient coordinates with $m$ being the number of unique values (magnitudes) of the gradient coordinates. As detailed in Figure~\ref{fig:hd-est-quality}, the error can reach $\lceil n / 2\rceil$. This a big mismatch, especially if we recall the Hamming distance's range is $[0,n]$. The negative impact of this on the Hamming search strategy by~\citet{maurer2009search} was verified empirically in Figure~\ref{fig:maur-perf}. We considered the simplest case where \maur was given access to the noisy Hamming oracle $\hat{\calO}$ in a setup similar to the one outlined in Figure~\ref{fig:hd-est-quality}, with $n=80$, $|\calD|=n/4= 20$, $m\in \{1,2\}$, and the hidden code $\vq^*=[+1,\ldots, +1]$. To account for the randomness in constructing $\calD$, we ran $30$ independent runs and plot the average Hamming distance (with confidence bounds) over \maur queries. In Figure~\ref{fig:maur-perf} (a), $m=1$ which corresponds to exact estimation $\hat{\calO} = \calO$, \maur spends $21$ queries to construct $\calD$ and terminates one query afterwards with the true binary code $\vq^*$, achieving a query ratio of 21/80. On the other hand, when we set $m=2$ in Figure~\ref{fig:maur-perf} (b); \maur returns a 4-Hamming-distance away solution from the true binary code $\vq^*$ after $51$ queries. This is not bad for an $80$-bit long code. However, this is in a tightly controlled setup where the gradient magnitudes are  just one of two values. To be studied further is the bias/variance decomposition of the returned solution and the corresponding query ratio. We leave this investigation for future work.

\begin{figure}[h!]
	\centering
	\resizebox{0.45\textwidth}{!}{
		\begin{tabular}{cc}
			\includegraphics[width=0.25\textwidth]{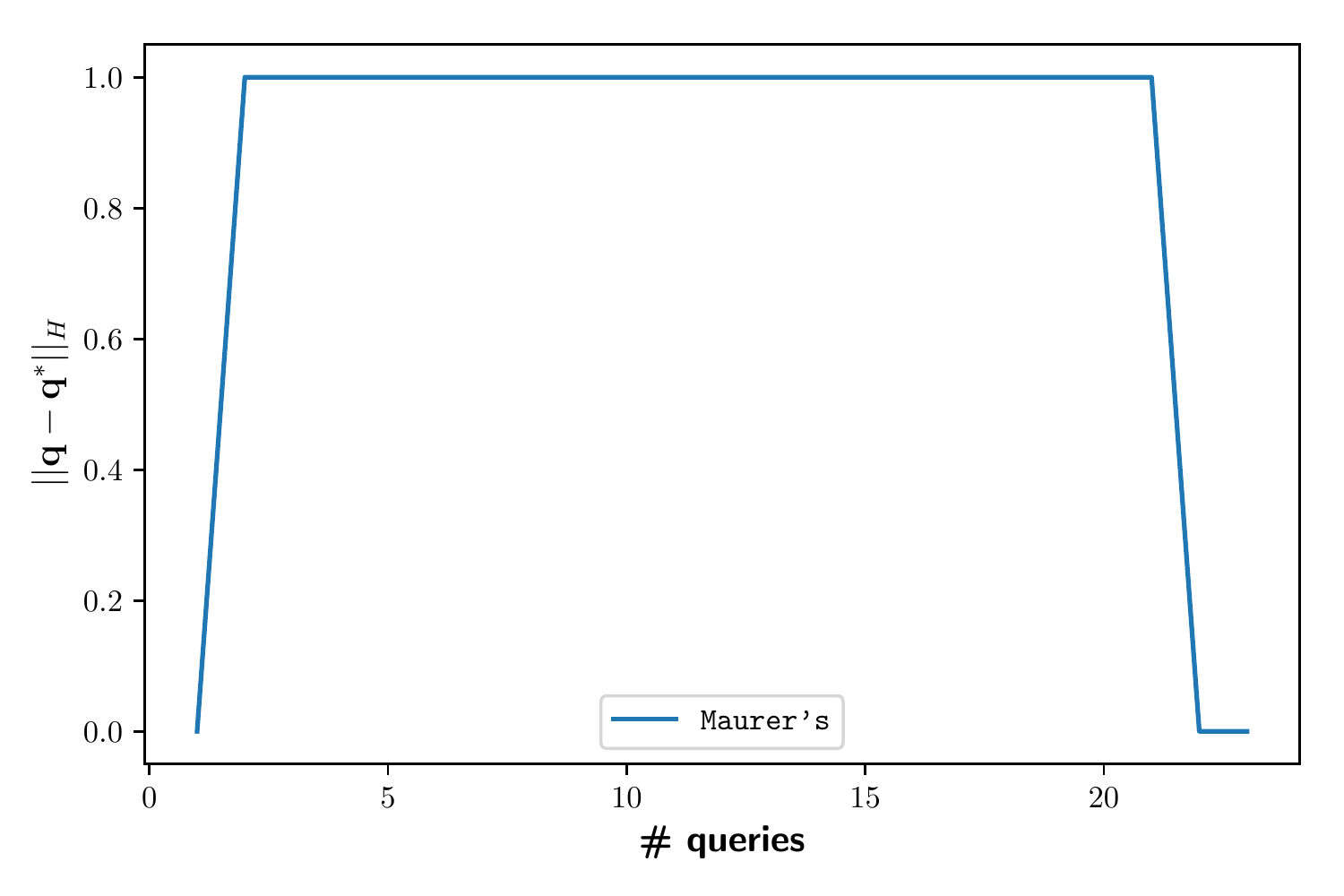} &
			\includegraphics[width=0.25\textwidth]{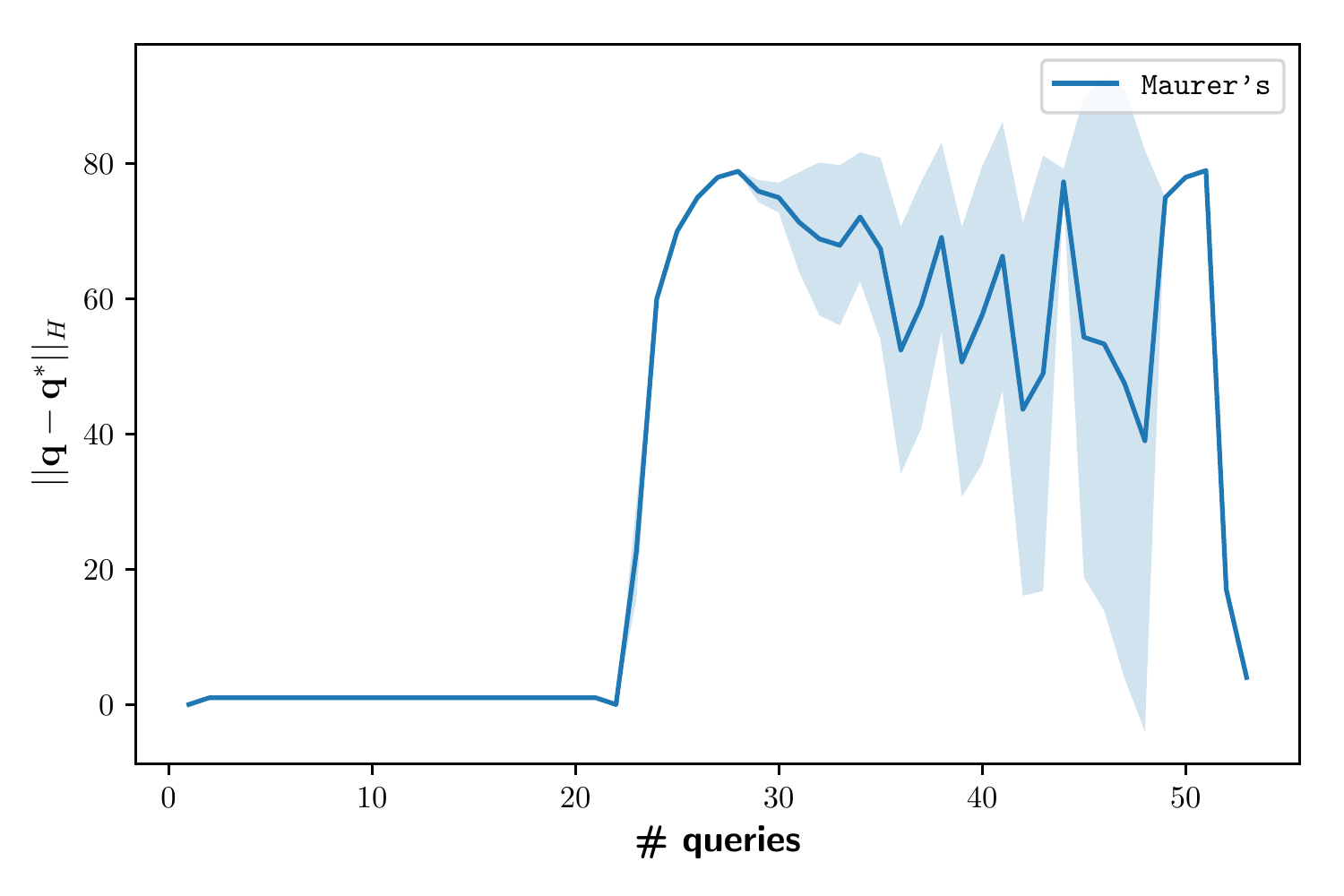}\\
			\tiny (a) $m=1$ & \tiny (b) $m=2$
		\end{tabular}
	}
	\caption{Performance of \maur on the noisy Hamming oracle $\hat{\calO}$. The setup is similar to that of Figure~\ref{fig:hd-est-quality}, with $n=80$ and $m\in \{1,2\}$.
	} 
	\label{fig:maur-perf}
\end{figure}

\subsection*{Approach 3: Optimism in the Face of Uncertainty} 
\label{sec:approach-2}
In the previous approach, we considered the approximated Hamming distance~(\eqref{eq:ham-dist-approx}) as a surrogate for the formal optimization  objective~(Eq. 3 in the main paper) of the gradient sign estimation problem. We found that current Hamming search strategies are not robust to approximation error. In this approach, we consider maximizing the directional derivative (Eq. 4 in the main paper)
\begin{equation}
\max_{\vq \in \calH} D_{\vq}L(\vx, y)\;,
\label{eq:grad-est-obj-df}
\end{equation}
as our formal objective of the gradient sign estimation problem. Formally, we treat the problem as a binary black-box optimization over the $2^n$ hypercube vertices, which correspond to all possible sign vectors. This is significantly worse than $O(n)$ of the continuous optimization view. Nevertheless, the rationale here is that we do not need to solve \eqref{eq:grad-est-obj-df} to optimality (recall the effectiveness of noisy \fgsm in Appendix A); we rather need a fast convergence to a suboptimal but \emph{adversarially helpful} sign vector $\vq$. In addition, the continuous optimization view often employs an iterative scheme of $T$ steps within the perturbation ball $B_p(\vx, \epsilon)$, calling the \emph{gradient estimation routine} in every step leading to a search complexity of $nT$. In our setup, we use the best obtained solution for~\eqref{eq:grad-est-obj-df} so far in a similar fashion to the noisy \fgsm.  In other words, our \emph{gradient sign estimation routine} runs at the top level of our adversarial example generation procedure instead of calling it as a subroutine. In this and the next approach, we address the following question: \emph{how do we solve \eqref{eq:grad-est-obj-df}?}

Optimistic methods, i.e., methods that implement the optimism in the face of uncertainty principle have demonstrated a theoretical as well as empirical success when applied to black-box optimization problems~\citep{munos2011optimistic,al2017embedded, al2018multi}. Such a principle finds its foundations in the machine learning field addressing the exploration vs. exploitation dilemma, known as the multi-armed bandit problem. Within the context of function optimization, optimistic approaches formulate the complex
problem of optimizing an arbitrary black-box function $g$ (e.g., \eqref{eq:grad-est-obj-df})  over the search space ($\calH$ in this paper) as a hierarchy of simple bandit problems~\citep{kocsis2006bandit} in the form of space-partitioning tree search $\calT$. At step $t$, the algorithm optimistically expands a leaf node (partitions the
corresponding subspace) from the set of leaf nodes $\calL_t$ that may contain the global optimum. The $i^{th}$ node at depth $h$, denoted by $(h,i)$, corresponds to the subspace/cell $\calH_{h,i}$ such that $\calH = \cup_{0\leq i < K^h}\calH_{h,i}$. To each node $(h,i)$, a representative point $\vq_{h,i}\in \calH_{h,i}$ is assigned, and the value of the node $(h,i)$ is set to $g(\vq_{h,i})$. See Figure~\ref{fig:goo-ilust} for an example of a space-partitioning tree $\calT$ of $\calH$, which will be used in our second approach to estimate the gradient sign vector.

Under some assumptions on the optimization objective and the hierarchical partitioning~$\calT$ of the search space, optimistic methods enjoy a finite-time bound on their \emph{regret} $R_t$ defined as
\begin{equation}
R_t =  g(\vq^*) - g(\vq(t))\;,
\end{equation}
where $\vq(t)$ is the best found solution by the optimistic method after $t$ steps. The challenge is how to align the search space such that these assumptions hold. In the following, we show that these assumptions can be satisfied for our optimization objective~(\eqref{eq:grad-est-obj-df}). In particular, when $g(\vq)$ is the directional derivative function $D_{\vq}L(\vx, y)$, and $\calH$'s vertices are aligned on a 1-dimensional line according to the Gray code ordering, then we can construct an optimistic algorithm with a finite-time bound on its regret. To demonstrate this, we adopt the Simultaneous Optimistic Optimization framework by \citet{munos2011optimistic} and the assumptions therein.

\begin{figure}[t]
	\centering
	\resizebox{0.49\textwidth}{!}{
		\begin{tabular}{c||c|c|}
			&\bf \Huge $n=3$ & \bf \Huge $n=4$ \\
			\toprule
			\toprule
			\put(-15,80){
				\rotatebox{90}{  \Huge Linear $f$}} &\includegraphics[]{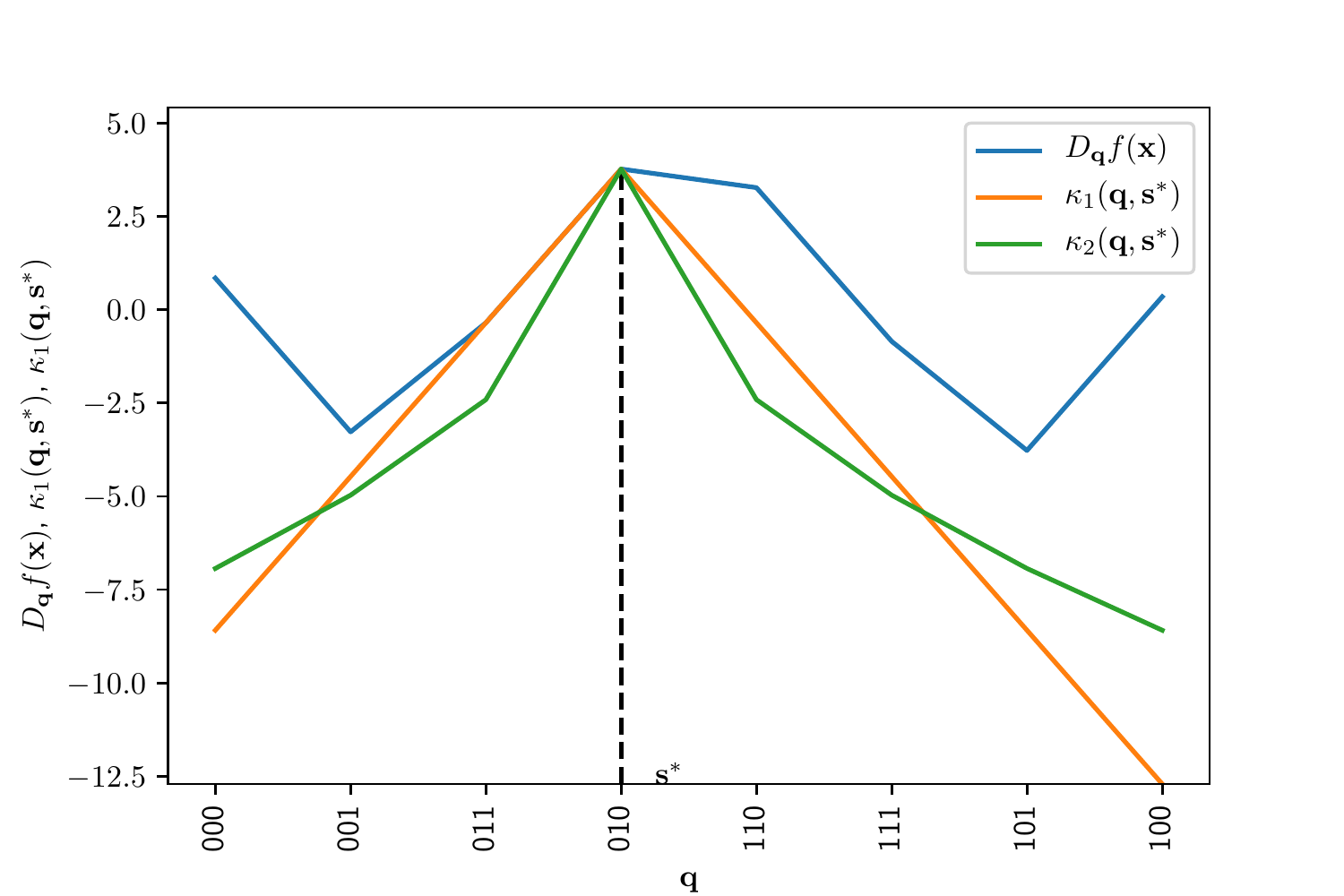} &
			\includegraphics[]{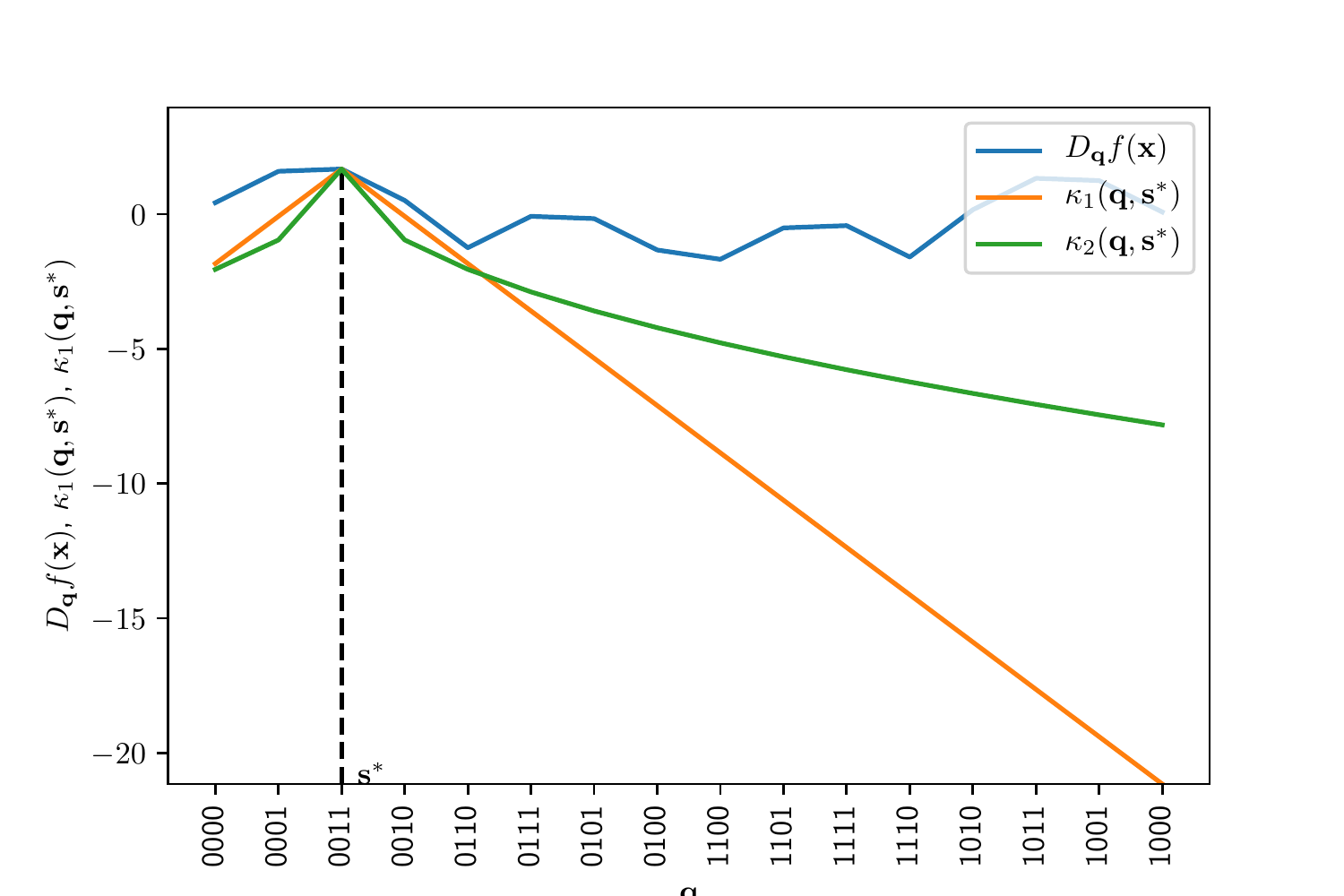}\\
			&\bf \Huge (a) & \bf \Huge (b) \\ \midrule
			\put(-15,80){
				\rotatebox{90}{  \Huge Quadratic $f$}}
			&\includegraphics[]{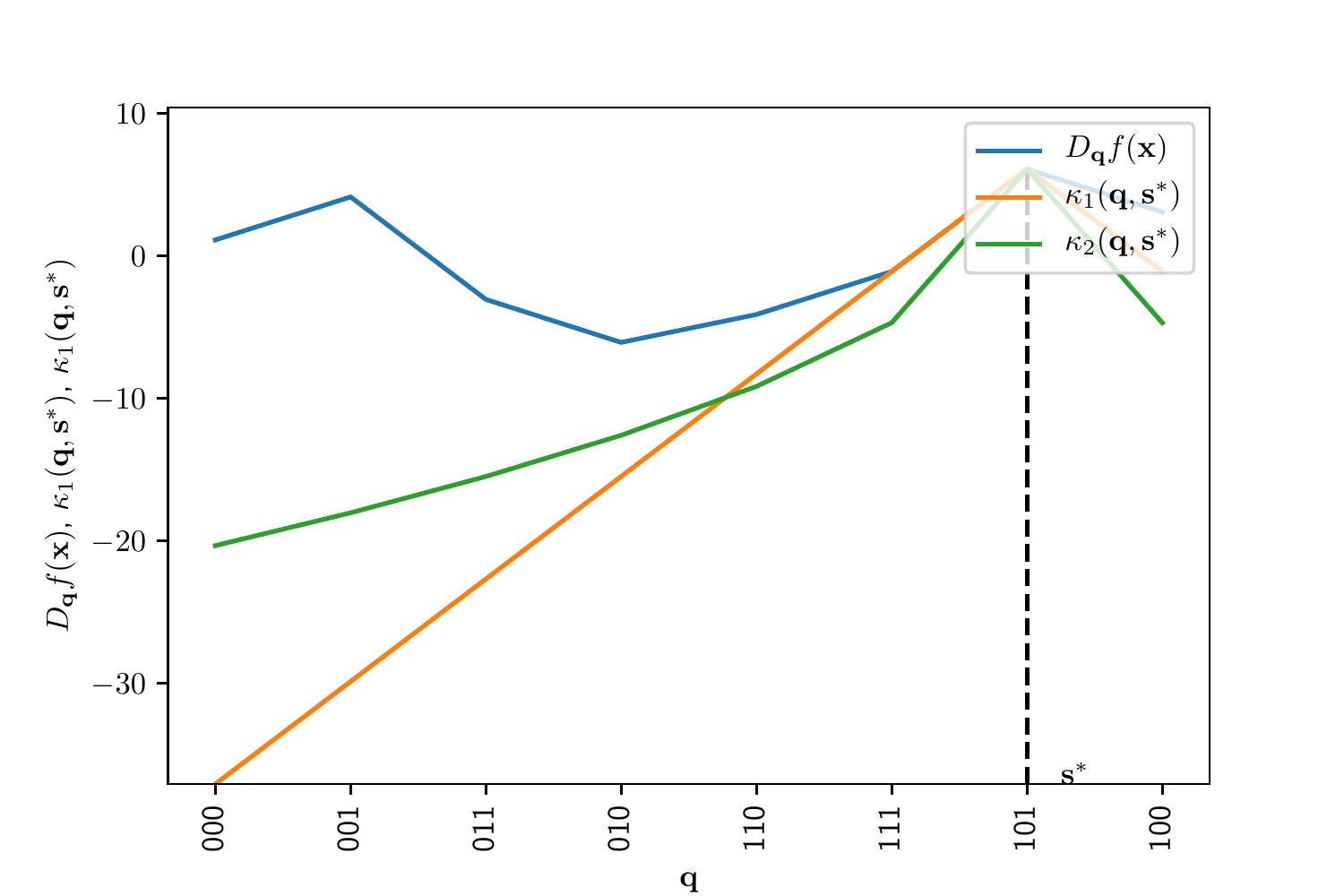} &
			\includegraphics[]{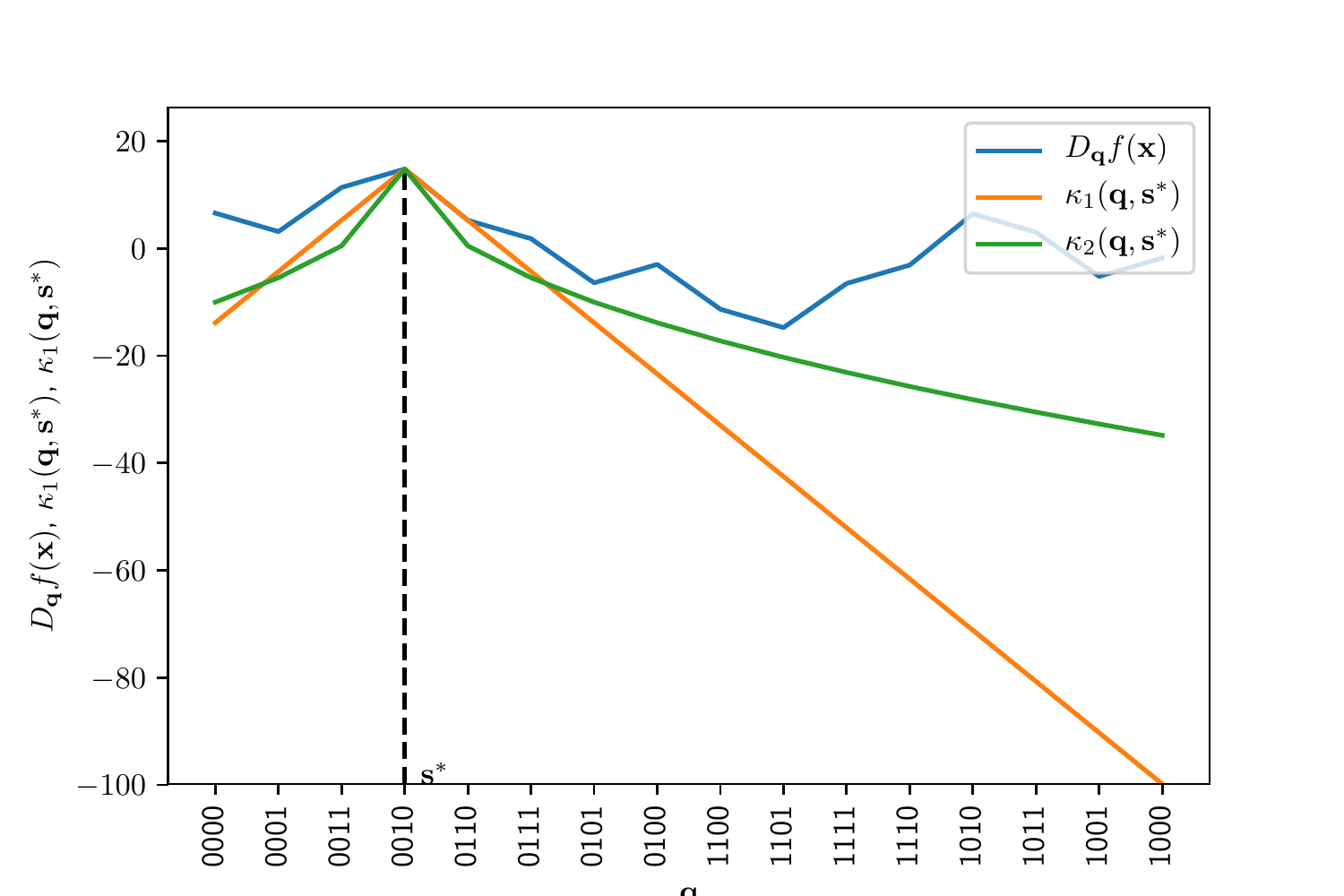}\\ 
			&\bf \Huge (c) & \bf \Huge (d) \\
			\bottomrule
		\end{tabular}
	}
	\caption{\small The directional derivative $D_{\vq}f(\vx)$ of some function $f$ at a point $\vx$ in the direction of a binary vector $\vq \in \calH \equiv \{-1, +1\}^n$ is locally smooth around the gradient sign vector, $\vq^*=\sgn(\nabla_\vx f(\vx))\in \calH$, when $\calH$ is ordered over one coordinate as a sequence of Gray codes. The plots show the local smoothness property---with two semi-metrics $\kappa_1(\cdot, \cdot)$ and $\kappa_2(\cdot,\cdot)$---of the directional derivative of functions $f$ of the form $\mathbf{c}^T\vx$ and $\vx^TQ\vx$ for $n\in \{3, 4\}$, as tabulated in Plots (a), (c) and (b, d), respectively. The local smoothness is evident as $D_{\vq^*}f(\vx) - D_{\vq}f(\vx)\leq \kappa_1(\vq, \vq^*)$ and $D_{\vq^*}f(\vx) - D_{\vq}f(\vx)\leq \kappa_2(\vq, \vq^*)$ for all $\vq \in \calH$. The semi-metrics $\kappa_1(\vq, \vq^*)$ and $\kappa_2(\vq, \vq^*)$ have the form $K |\texttt{rank}_{Gray}(\vq)- \texttt{rank}_{Gray}(\vq^*)|^\alpha$, where $\texttt{rank}_{Gray}(\cdot)$ refers to the rank of an $n$-binary code in the Gray ordering of $n$-binary codes (e.g., $\texttt{rank}_{Gray}([-1, -1, +1, -1])=4$), $K>0$ , and $\alpha >0$. With this property at hand, we employ the optimism in the face of uncertainty principle in \goo to maximize $D_{\vq}f(\vx)$  over $\calH$. For legibility, we replaced $-1$ with $0$ when enumerating $\calH$ on the x-axis.}
	\label{fig:goo-smoothness}
\end{figure}

For completeness, we reproduce \cite{munos2011optimistic}'s basic definitions and assumptions with respect to our notation. At the same time we show how the gradient sign estimation problem~(\eqref{eq:grad-est-obj-df}) satisfies them based on the second property of the directional derivative as follows.

\begin{customdefn}{2}[Semi-metric]
	\label{assmp:metric} We assume that $\kappa: \calH \times \calH\to \R^+$ is such that for all $\vp, \vq \in \calH$, we have $\kappa(\vp, \vq) = \kappa(\vq, \vp)$ and $\kappa(\vp, \vq)=0$ if and only if $\vp =\vq$.
\end{customdefn}

\begin{customdefn}{3}[Near-optimality dimension]
	The near-optimal dimension is the smallest $d\geq 0$ such that there exists $C >0$ such that for any $\varepsilon >0$, the maximal number of disjoint $\kappa$-balls of radius $\upsilon\varepsilon$ and center in $\calH_{\varepsilon}$ is less than $C\varepsilon^{-d}$.
\end{customdefn}

\begin{customprop}{3}[Local smoothness of $D_{\vq} L(\vx, y)$]
	\label{prop:smoothness}
	For any input/label pair $(\vx,y)$, there exists at least a global optimizer $\vq^*\in \calH$ of $D_{\vq} L(\vx, y)$ (i.e., $D_{\vq^*} L(\vx, y)= \sup_{\vq \in \calH} D_{\vq} L(\vx, y)$) and for all $\vq \in \calH$,
	$$D_{\vq^*} L(\vx, y) - D_{\vq} L(\vx, y) \leq \kappa(\vq, \vq^*)\;.$$
\end{customprop}
Refer to Figure~\ref{fig:goo-smoothness} for a pictorial proof of Property~\ref{prop:smoothness}.
\begin{assumption}[Bounded diameters]
	\label{assmp:bounded-diameter}
	There exists a decreasing a decreasing sequence $\omega(h) > 0$, such that for any depth $0 \leq h <n$, for any  cell $\calH_{h,i}$ of depth $h$, we have $\sup_{\vq \in \calH_{h,i}} \kappa(\vq_{h,i}, \vq) \leq \omega(h)$.
\end{assumption}
To see how Assumption~\ref{assmp:bounded-diameter} is met, refer to Figure~\ref{fig:goo-ilust}.
\begin{assumption}[Well-shaped cells] 
	\label{assmp:well-shaped}
	There exists $\upsilon > 0$ such that for any depth $0\leq h < n$, any cell $\calH_{h,i}$ contains a $\kappa$-ball of radius $\upsilon \omega(h)$ centered in $\vq_{h,i}$.
\end{assumption}
To see how Assumption~\ref{assmp:well-shaped} is met, refer to Figure~\ref{fig:goo-ilust}. With the above assumptions satisfied, we propose the \texttt{G}ray-code \texttt{O}ptimistic \texttt{O}ptimization (\goo), which is an instantiation of \citep[Algorithm 2]{munos2011optimistic} tailored to our optimization problem~(\eqref{eq:grad-est-obj-df}) over a 1-dimensional alignment of $\calH$ using the Gray code ordering. The pseudocode is outlined in Algorithm~\ref{alg:goo}. The following theorem bounds \goo's regret.

\begin{theorem}{Regret Convergence of \goo}
	Let us write $h(t)$ the smallest integer $h$ such that 
	\begin{equation}
	C h_{max}(t) \sum_{l=0}^{h} \omega(l)^{-d} + 1\geq t\;.
	\label{eq:iteration-bound}
	\end{equation}
	Then, with $g(\vq) = D_\vq L(\vx, y)$, the regret of \goo (Algorithm~\ref{alg:goo}) is bounded as
	$$R_t \leq \omega(\min(h(t), h_{max}(t) + 1))$$
\end{theorem}
\begin{proof}
	We have showed that our objective function~(\eqref{eq:grad-est-obj-df}) and the hierarchical partitioning of $\calH$ following the Gray code ordering confirm to Property~\ref{prop:smoothness} and Assumptions~\ref{assmp:bounded-diameter}  and~\ref{assmp:well-shaped}. The $+1$ term in \eqref{eq:iteration-bound}  is to accommodate the evaluation of node $(n, 0)$ before growing the space-partitioning tree $\calT$---see Figure~\ref{fig:goo-ilust}. The rest follows from the proof of \citep[Theorem 2]{munos2011optimistic}.
\end{proof}

\begin{algorithm}[t!]
	\caption{ \texttt{G}ray-code \texttt{O}ptimistic \texttt{O}ptimization (\goo) \\
		\textbf{Input:}\\
		\begin{tabular}{p{1.8cm}p{5.75cm}}
			\footnotesize $g:\calH \to \R$ &
			\footnotesize : the black-box linear function to be maximized over the binary  hypercube~$\calH\equiv\{-1, +1\}^n$\\	
			\footnotesize $h_{max}(t)$ &
			\footnotesize : the maximum depth function which limits the tree to grow up to $h_{max}(t)+1$ after $t$ node expansions
		\end{tabular}
	}    
	\label{alg:goo}
	\begin{algorithmic}			
		\tiny
		\STATE {\bfseries Initialization:} Set $t=1$, $\calT_t= \{(0,0)\}$ (root node). Align $\calH$ over $\calT_t$ using the Gray code ordering.
		\WHILE {True}
		\STATE $v_{max}\gets -\infty$
		\FOR {$h=0$ \textbf{to} $\min(depth(\calT_t), h_{max}(t))$} 
		\STATE Among all leaves $(h,j) \in \calL_t$ of depth $h$, select $$(h,i) \in \argmax_{(h,j)\in \calL_t} g(\vq_{h,i})$$
		\IF {$g(\vq_{h,i}) \geq v_{max}$}
		\STATE Expand this node: add to $\calT_t$ the two children $(h+1, i_k)_{1\leq k \leq 2}$	
		\STATE 			$v_{max}\gets g(\vq_{h,i}) $
		\STATE $t \gets t + 1$
		\IF {query budget is exhaused}
		\STATE \textbf{return} the best found solution $\vq(t)$.
		\ENDIF
		\ENDIF
		\ENDFOR
		\ENDWHILE
	\end{algorithmic}
\end{algorithm}

Despite being theoretically-founded, \goo is slow in practice. This is expected since it is a global search technique that considers all the $2^n$ vertices of the $n$-dimensional hypercube $\calH$. Recall that we are looking for adversarially helpful solution $\vq$ that may not be necessarily optimal. %To this end, we consider the \emph{separability} property of the directional derivative, a more useful property than its local smoothness as described in our third approach next.

\subsection*{Empirical Evaluation of the Approaches on a Set of Toy Problems}

We tested both \goo and \signhunter (along with \maur and \elim)\footnote{See Appendix C for details on these algorithms.} on a set of toy problems and found that \signhunter performs significantly better than \goo, while \maur and \elim were sensitive to the approximation error---see Figure~\ref{fig:approx-vs-exact-hd}. We remind the reader that \maur and \elim are two search strategies for $\vq^*$ using the Hamming oracle. After response $r^{(i)}$ to query $\vq^{(i)}$, \elim eliminates all binary codes $\vq\in \calH$ with $||\vq - \vq^{(i)}||_H \neq r^{(i)}$ in an iterative manner. Note that \elim is a naive technique that is not scalable with $n$. For \maur, we refer the reader to~\cite{maurer2009search}.

\begin{figure*}[t]
	\centering 
	\resizebox{0.49\textwidth}{!}{
		\begin{tabular}{c||c|c|}
			& \bf \Huge Hamming Distance Trace & \bf \Huge Directional Derivative Trace \\
			\toprule
			\toprule
			\put(-20,-3){
				\rotatebox{90}{\bf \Huge Noiseless Hamming Oracle
				}}
				&\includegraphics[]{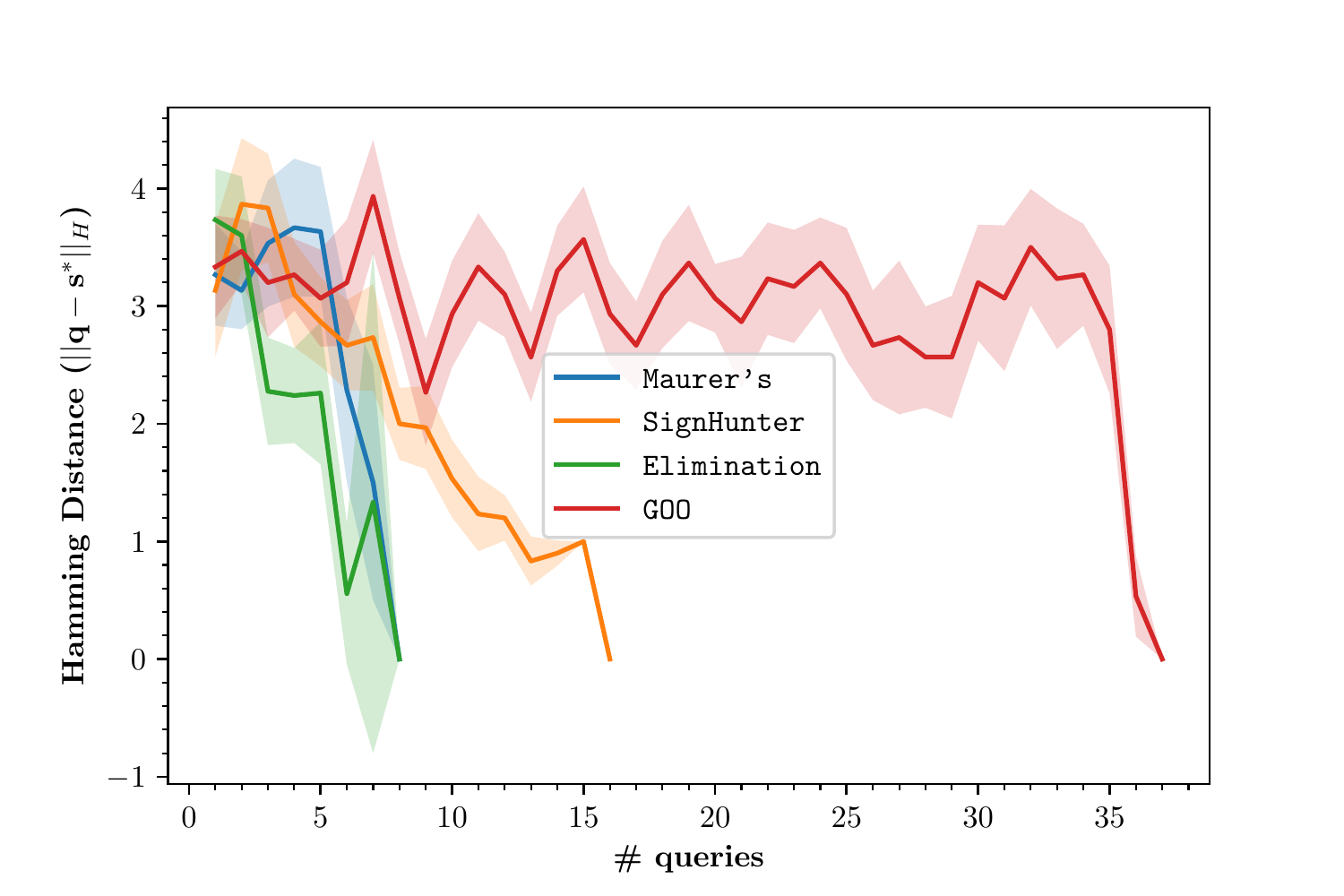} &
				\includegraphics[]{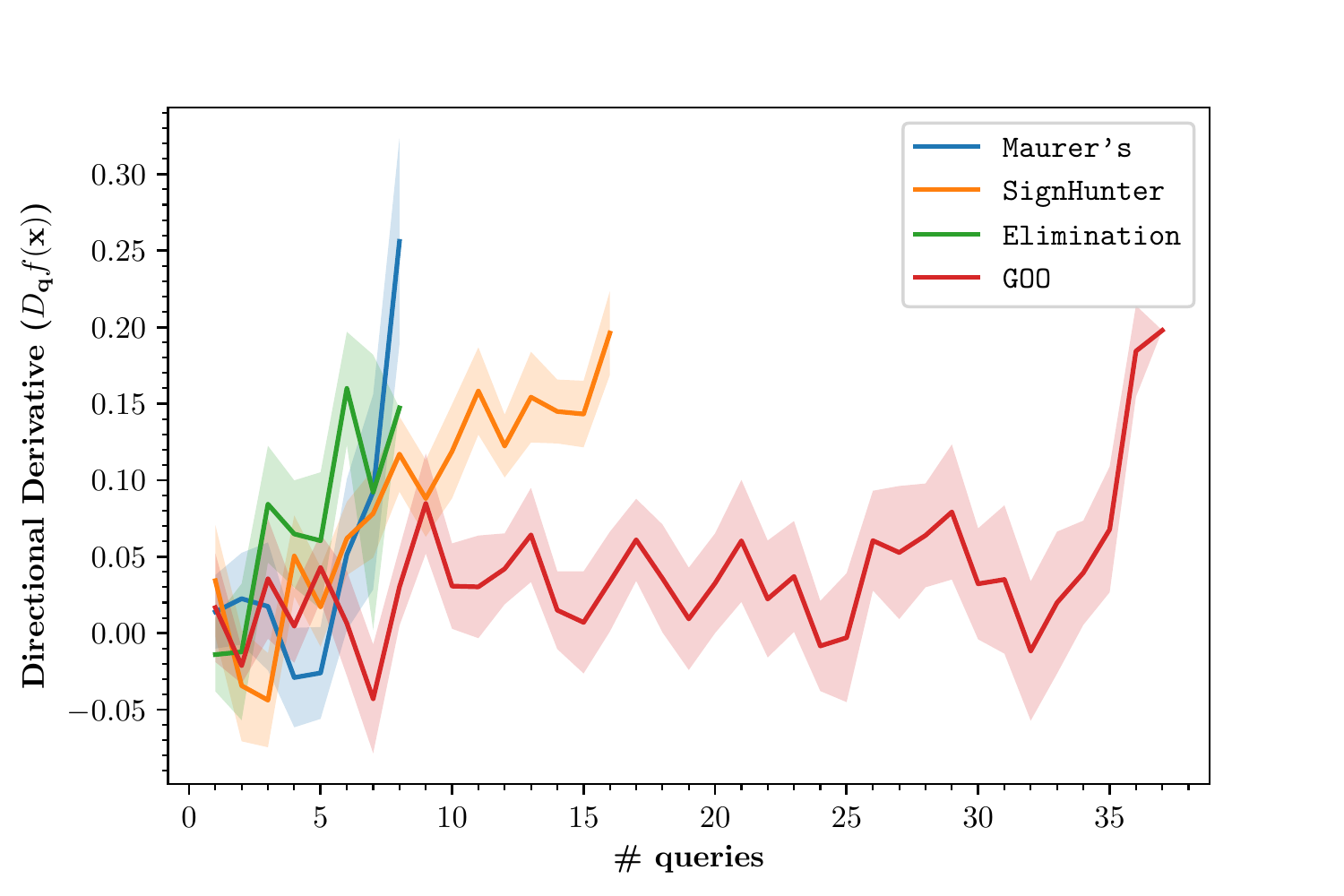}\\
				&\bf \Huge (a) & \bf \Huge (b) \\
				\midrule
				\put(-20,5){
					\rotatebox{90}{\bf \Huge Noisy Hamming Oracle
					}} 
					&\includegraphics[]{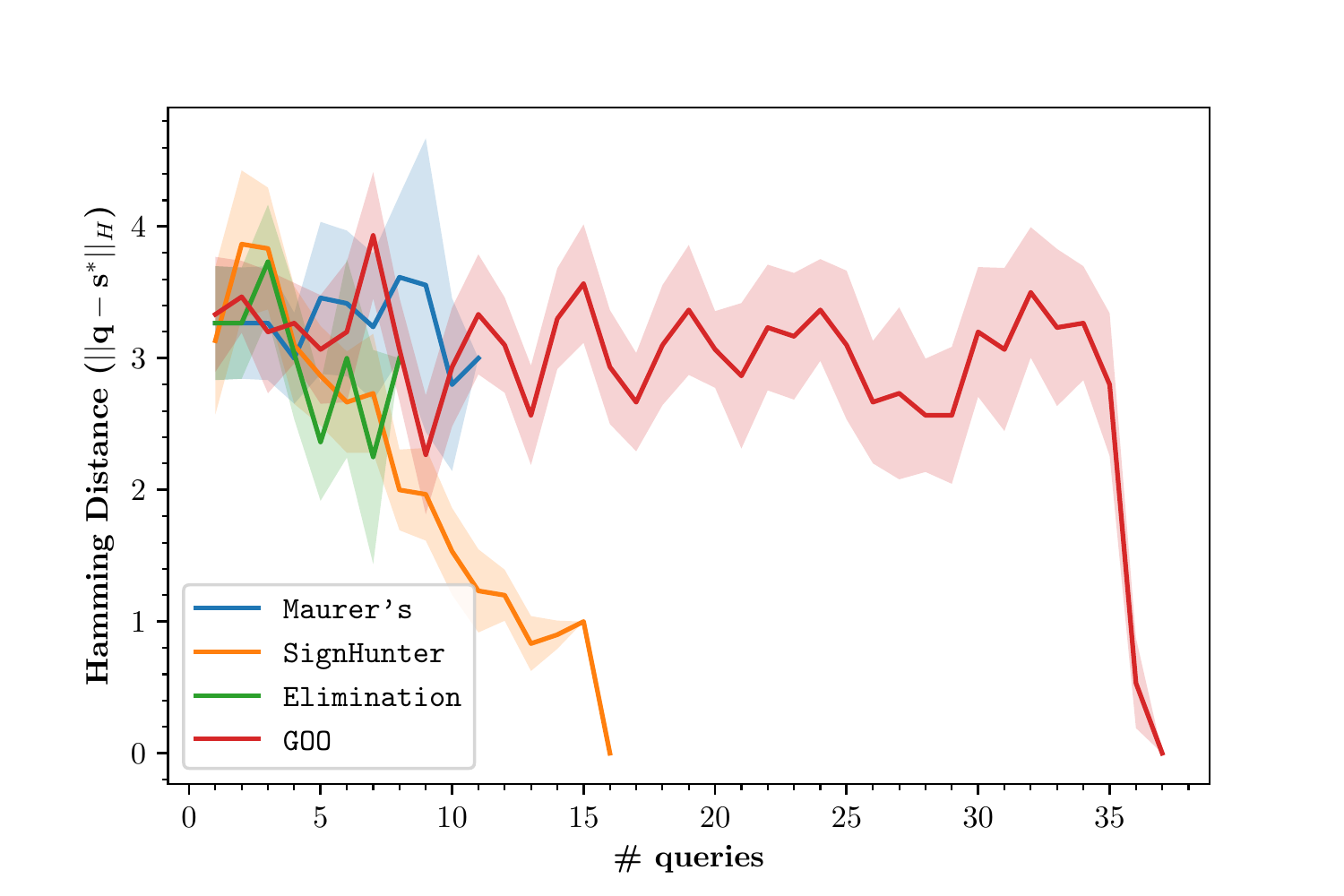} &
					\includegraphics[]{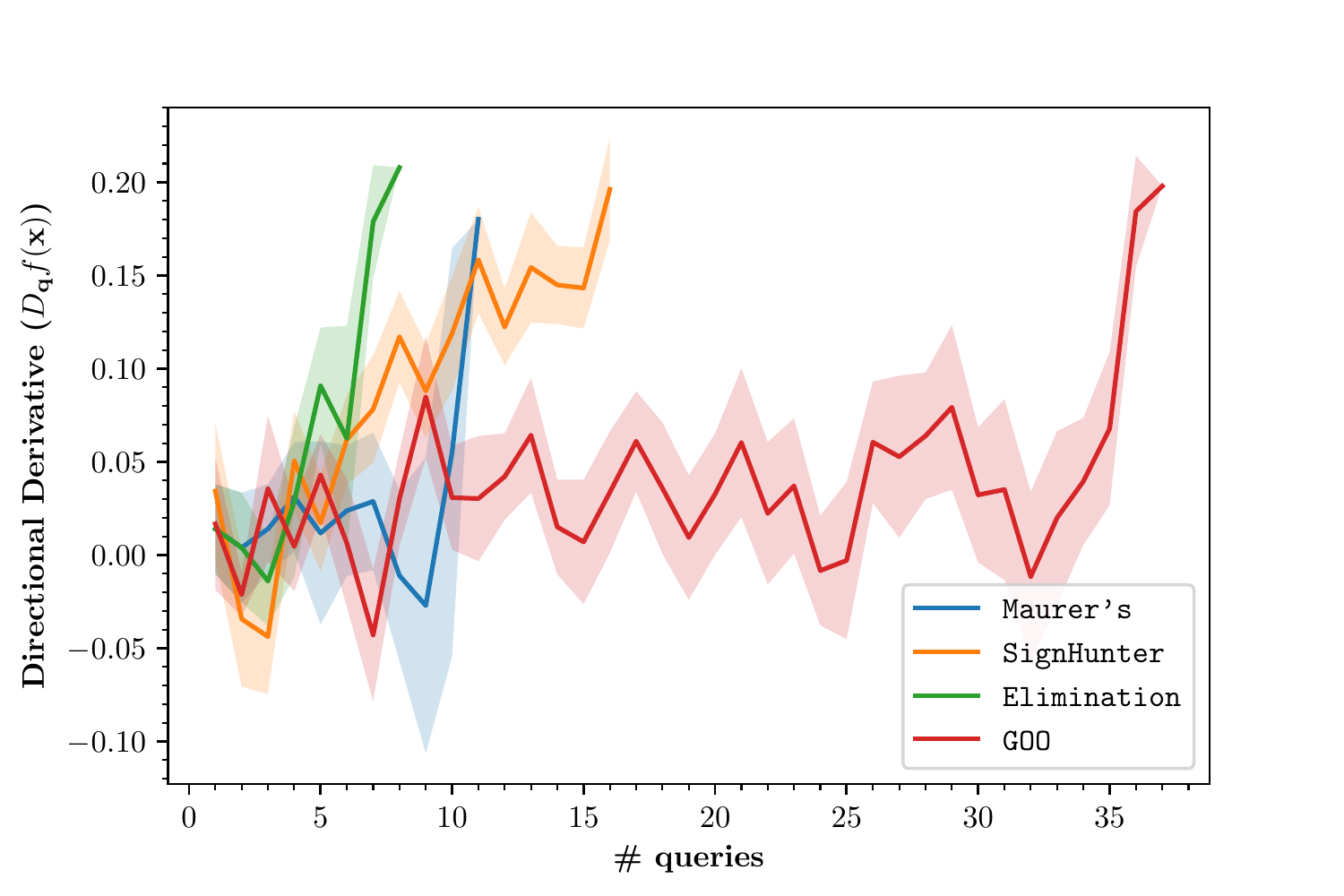}\\
					&\bf \Huge (c) &  \bf \Huge (d) \\
					\bottomrule
				\end{tabular}
			}
			\caption{ \small \emph{Noiseless vs. Noisy Hamming Oracle:} The trace  of the Hamming distance (first column, where the lower the trace the better) and directional derivative (second column, where the higher the trace the better) values of \elim and \maur queries, when given access to a noiseless/ideal (first row) and noisy Hamming oracles (second row)---through a directional derivative approximation as discussed in Approach 2---for a synthetic function $f$ of the form $\vx^TQ\vx$ with $n=7$. We expect the traces to go up and down as they explore the $2^n$ search space. The end of an algorithm's trace represents the value of the Hamming distance (directional derivative) for the first column (for the second column) at the algorithm's solution. For comparison, we also plot \goo and \signhunter's traces. Note that the performance of \goo and \signhunter is the same in both noiseless and noisy cases as both algorithms operate directly on the directional derivative approximation rather than the noiseless/noisy Hamming oracle. In the case of noiseless Hamming oracle, both \elim and \maur finds the optimal vector $\vq^* \in \calH\equiv \{-1, +1\}^7$ with \# queries $\leq 7$---their traces end at most at $8$ just to show that the algorithm's solution achieves a zero Hamming distance as shown in Plot (a), which corresponds to the maximum directional derivative in Plot (b). With a noisy Hamming oracle, these algorithm break as shown in Plots (c) and (d): taking more than $n$ queries and returning sub-optimal solutions---e.g., \maur returns on average a three-Hamming-distance solution. On the other hand, \goo and \signhunter achieve a zero Hamming distance in both cases at the expense of being less query efficient. While being theoretically-founded, \goo is slow as it employs a global search over the $2^n$ space. Despite \signhunter's local search, it converges to the optimal solution after $2\times7$ queries in accordance with Theorem~\ref{thm:signhunter}.
				The solid curves indicate the corresponding averaged trace surrounded by a $95\%$-confidence bounds using $30$ independent runs. For convenience, we plot the symmetric bounds where in fact they should be asymmetric in Plots (a) and (c) as the Hamming distance's range is $\mathbb{Z}^+_0.$}
			\label{fig:approx-vs-exact-hd}
		\end{figure*}

\begin{figure*}[t]
	\centering
	\resizebox{0.9\textwidth}{!}{
		\begin{tabular}{cc}
			\includegraphics[width=0.49\textwidth]{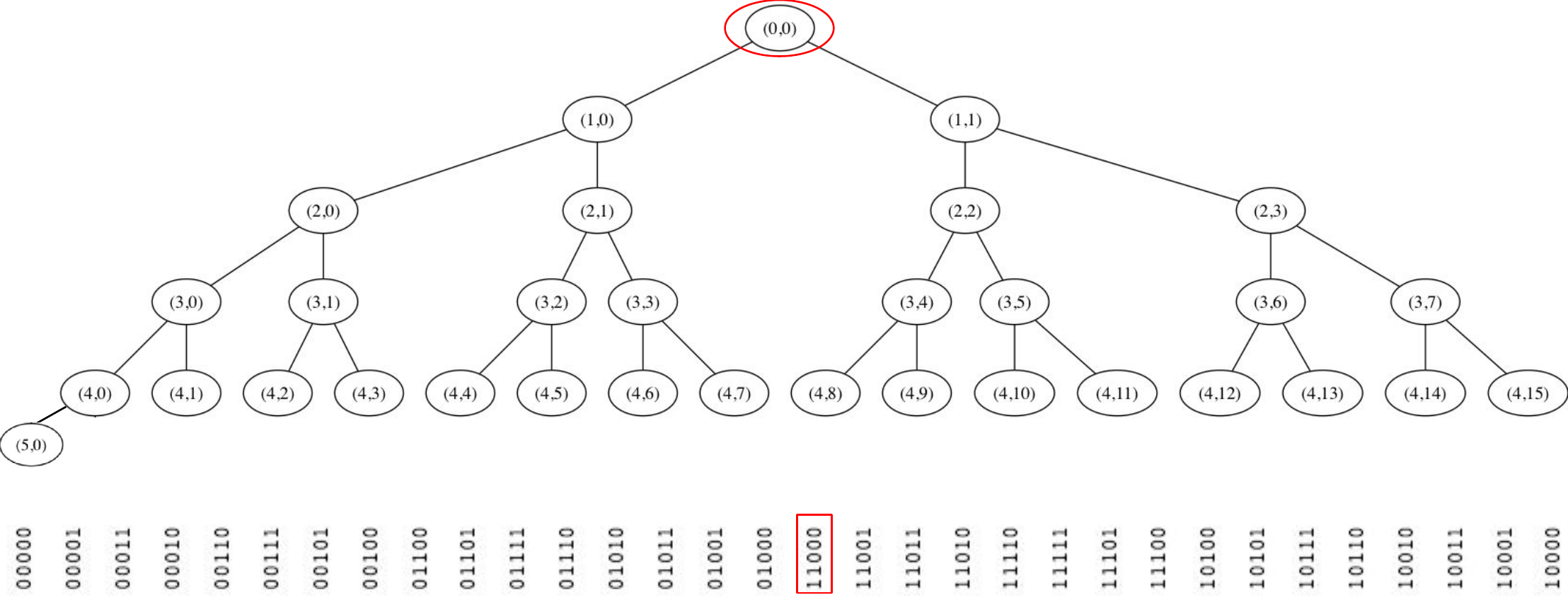} &
			\includegraphics[width=0.49\textwidth]{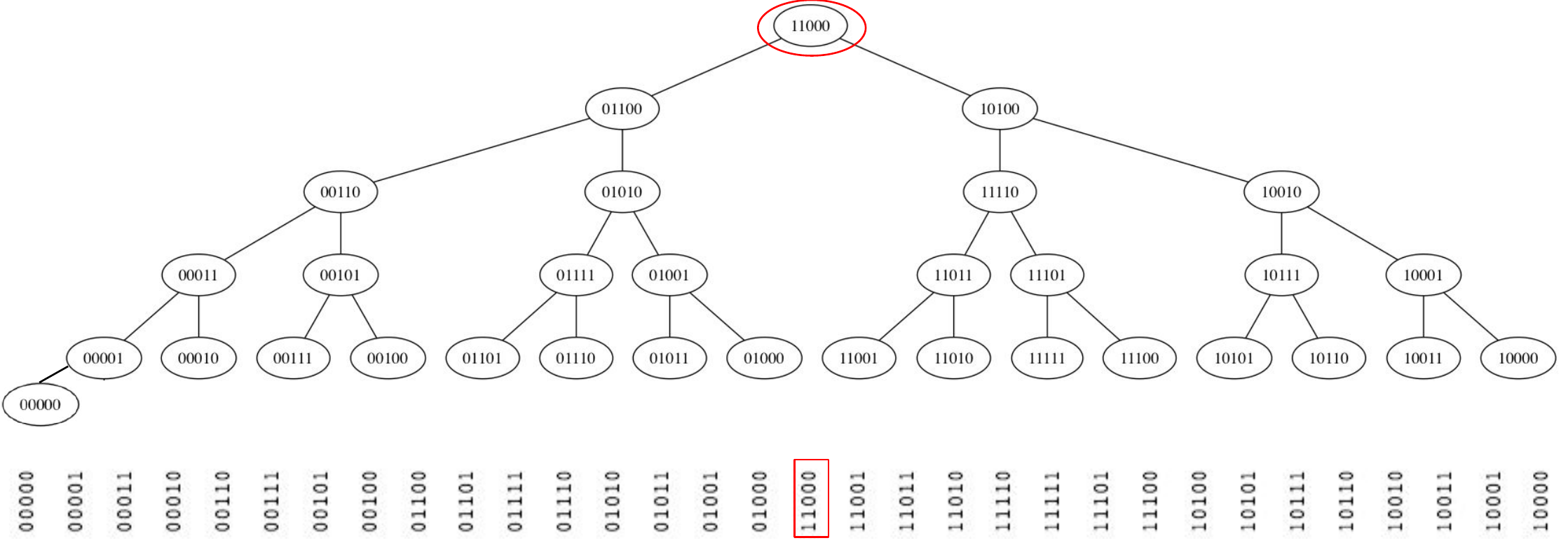}\\
			\large (a) $(h,i)$ view & \large (b) $\vq_{h,i}$ view
		\end{tabular}
	}
	\caption{\small 
		Illustration of the proposed Gray-ordering based partitioning (fully expanded) tree $\calT$ of the search space $\calH=\{-1, +1\}^n$---with $n=5$---used in the \texttt{G}ray-code \texttt{O}ptimistic \texttt{O}ptimization (\goo). The plots are two different views of the same tree. Plot (a) displays the node name $(h,i)$, while Plot (b) shows its representative binary code $\vq_{h,i} \in \calH$. For brevity, we replaced $-1$s with $0$s. The red oval and rectangle highlights the tree's root and its corresponding binary code, respectively. Consider the node $(2,1)$ whose representative code is $\vq_{2,1}=[-1, +1, -1, +1, -1]$ and its corresponding subspace $\calH_{2,1} = \{\vq_{2,1}, \vq_{3,2},  \vq_{3,3} ,  \vq_{4,4},  \vq_{4,5},
		\vq_{4,6}, \vq_{4,7} \}$. The same reasoning applies to the rest of the nodes. To maintain a valid binary partition tree,
		one can ignore the anomaly leaf node $(5,0)$, this corresponds to the code $\vq_{5,0} = [-1, -1, -1, -1, -1]$, which in practice can be evaluated prior to building the tree. Let us consider the nodes at depth $h=2$,
		observe that for all $i \in \{0, 1, 2, 3\}$ 1) $|\calH_{2, i}|=7$ ; 2) $\vq_{2,i}$ is centered around the other $6$ members of $\calH_{2, i}$ in the Gray code ordering; and that 3) $\calH_{2, i}$ constitutes a contiguous block of codes along the 1-dimensional alignment shown below the tree.  Thus, it suffices to define a semi-metric based on the corresponding indices of the codes along this alignment. For a given depth $h$, the index of any code $\vq \in \calH_{h,i}$ is at most $\omega(h)=\sup_{j}|\calH_{h+1, j}|$, which establishes Assumption 2. Assumption 3 follows naturally from the fact that nodes at a given depth $h$ partition the search space $\calH$ equally (e.g., $|\calH_{2, i}|=7$ for all $i$).
	} 
	\label{fig:goo-ilust}
\end{figure*}
\clearpage
\onecolumn
\section*{Appendix C. Proofs for Theorems in the Main Paper}

\newtheorem{innercustomthm}{Theorem}
\newenvironment{customthm}[1]
{\renewcommand\theinnercustomthm{#1}\innercustomthm}
{\endinnercustomthm}

Along with the proofs, we restate the theorems for completeness.
\begin{customthm}{2}
	\label{thm:ub}
	A hidden $n$-dimensional binary code $\vq^* \in \calH$ can be retrieved exactly with no more than $n$ queries
	to the noiseless Hamming oracle $\calO$. %\footnote{The result also applies to $n$-dimensional binary codes represented in $\{0,1\}^n$.}
\end{customthm}

\begin{proof}
	The key element of this proof is that the Hamming distance between two $n$-dimensional binary codes $\vq, \vq^*\in \calH$  can be written as 
	\begin{equation}
	r = \ham{\vq - \vq^*} = \frac{1}{2}(n - \vq^T \vq^*)\;.
	\label{eq:hamming-dot}
	\end{equation}
	Let $Q$ be an $n\times n$ matrix where the $i$th row is the $i$th query code $\vq^{(i)}$. Likewise, let $r^{(i)}$ be the corresponding $i$th query response, and $\vr$ is the concatenating vector. In matrix form, we have
	$$\vq^* = Q^{-1}({n}\mathbf{1}_n-2 \vr)\;,$$
	where $Q$ is invertible if we construct linearly independent queries $\{\vq^{(i)}\}_{1\leq i \leq n}$.
\end{proof}

In Figure~\ref{fig:query-ratio}, we plot the bounds above for $n=\{1, \ldots, 10\}$, along with two search strategies for $\vq^*$ using the Hamming oracle: i) \maur~\cite{maurer2009search}; and ii) search by \elim which, after response $r^{(i)}$ to query $\vq^{(i)}$, eliminates all binary codes $\vq\in \calH$ with $||\vq - \vq^{(i)}||_H \neq r^{(i)}$ in an iterative manner. Note that \elim is a naive technique that is not scalable with $n$. %\cite{ewert2010efficient}

\begin{customthm}{3}(Optimality of \signhunter)
	\label{thm:signhunter} Given $2^{\lceil \log(n) +1\rceil}$ queries and that the directional derivative is well approximated by the finite-difference~(Eq. 2 in the main paper), \signhunter is at least as effective as \fgsm~\citep{aes2015goodfellow} in crafting adversarial examples.
\end{customthm}

\begin{proof} Recall that the $i$th coordinate of the gradient sign vector can be recovered as outlined in \eqref{eq:mc-sign}. From the definition of \signhunter, this is carried out for all the $n$ coordinates after $2^{\lceil \log(n) +1\rceil}$ queries. Put it differently, after $2^{\lceil \log(n) +1\rceil}$ queries, \signhunter has flipped every coordinate alone recovering its sign exactly as shown in \eqref{eq:mc-sign}.  Therefore, the gradient sign vector is fully recovered, and one can employ the \fgsm attack to craft an adversarial example. Note that this is under the assumption that our finite difference approximation of the directional derivative (Eq. 2 in the main paper) is good enough (or at least a rank-preserving).
\end{proof}

\begin{figure}[h!]
	\centering
	\includegraphics[width=0.45\textwidth]{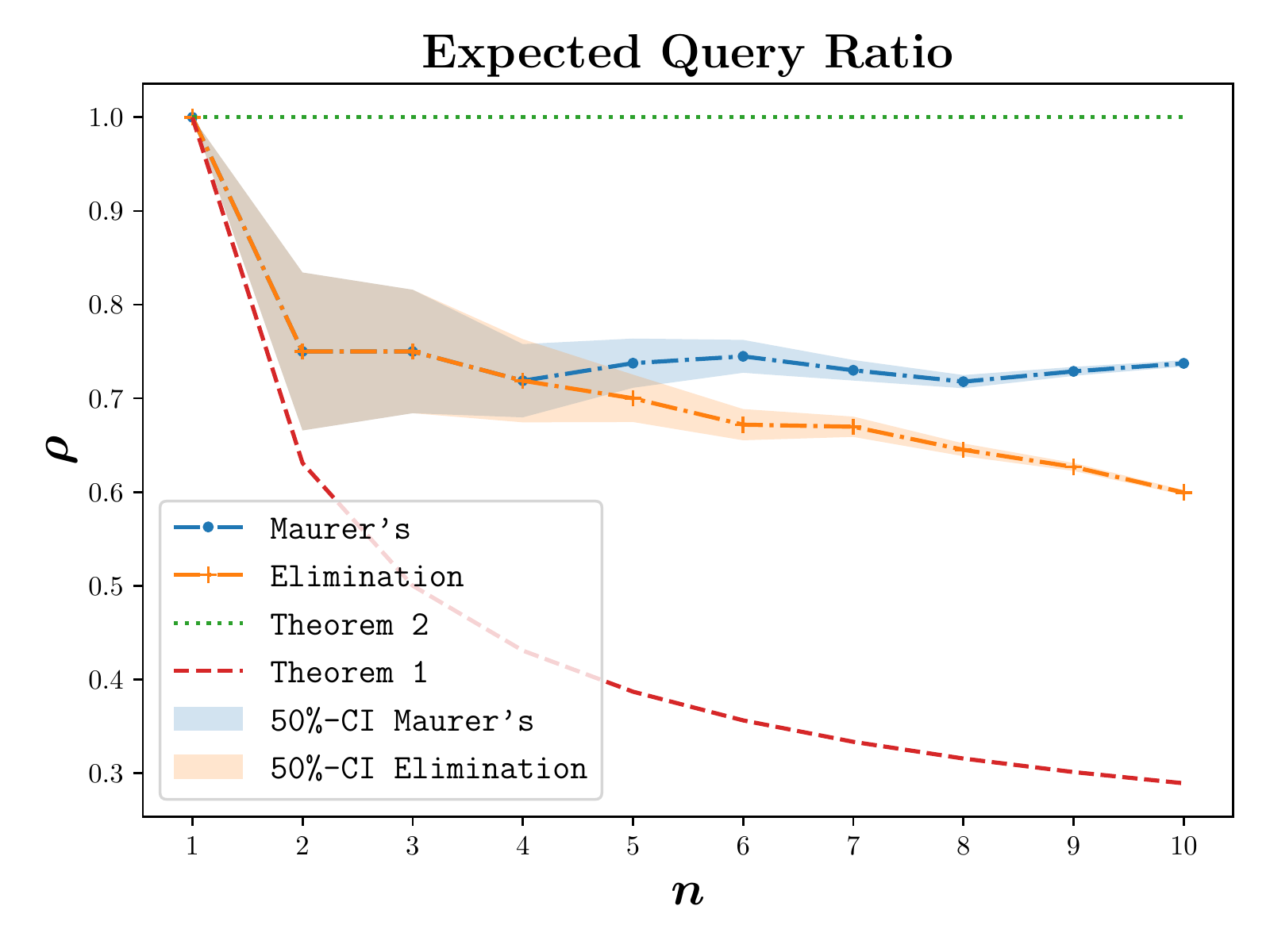}
	\caption{Expected Query Ratios for $n=\{1,\ldots, 10\}$ with the \emph{noiseless} Hamming oracle~$\calO$.}
	\label{fig:query-ratio}
\end{figure}

\cleardoublepage
\twocolumn[]
\section*{Appendix D. Experiments Setup}

This section outlines the experiments setup as follows. Figure~\ref{fig:tune} shows the performance of the considered algorithms on a synthetic concave loss function after tuning their hyperparameters. A possible explanation of \signhunter's superb performance is that the synthetic loss function is well-behaved in terms of its gradient given an image. That is, most of gradient coordinates share the same sign, since pixels tend to have the same values and the optimal value for all the pixels is the same 
$\frac{\vx_{min} + \vx_{max}}{2}$.  Thus, \signhunter will recover the true gradient sign with as few queries as possible (recall the example in Section~4.1 of the main paper). Moreover, given the structure of the synthetic loss function, the optimal loss value is always at the boundary of the perturbation region. The boundary is where \signhunter samples its perturbations. Tables~\ref{tbl:nes-param},~\ref{tbl:zo-param},~\ref{tbl:bandit-param}, and~\ref{tbl:sign-param} outline the algorithms' hyperparameters, while Table~\ref{tbl:general-setup} describes the general setup for the experiments.

\begin{figure}[h!]
	\centering
	\resizebox{0.49\textwidth}{!}{
		\begin{tabular}{cc}
			\includegraphics[width=0.2\textwidth]{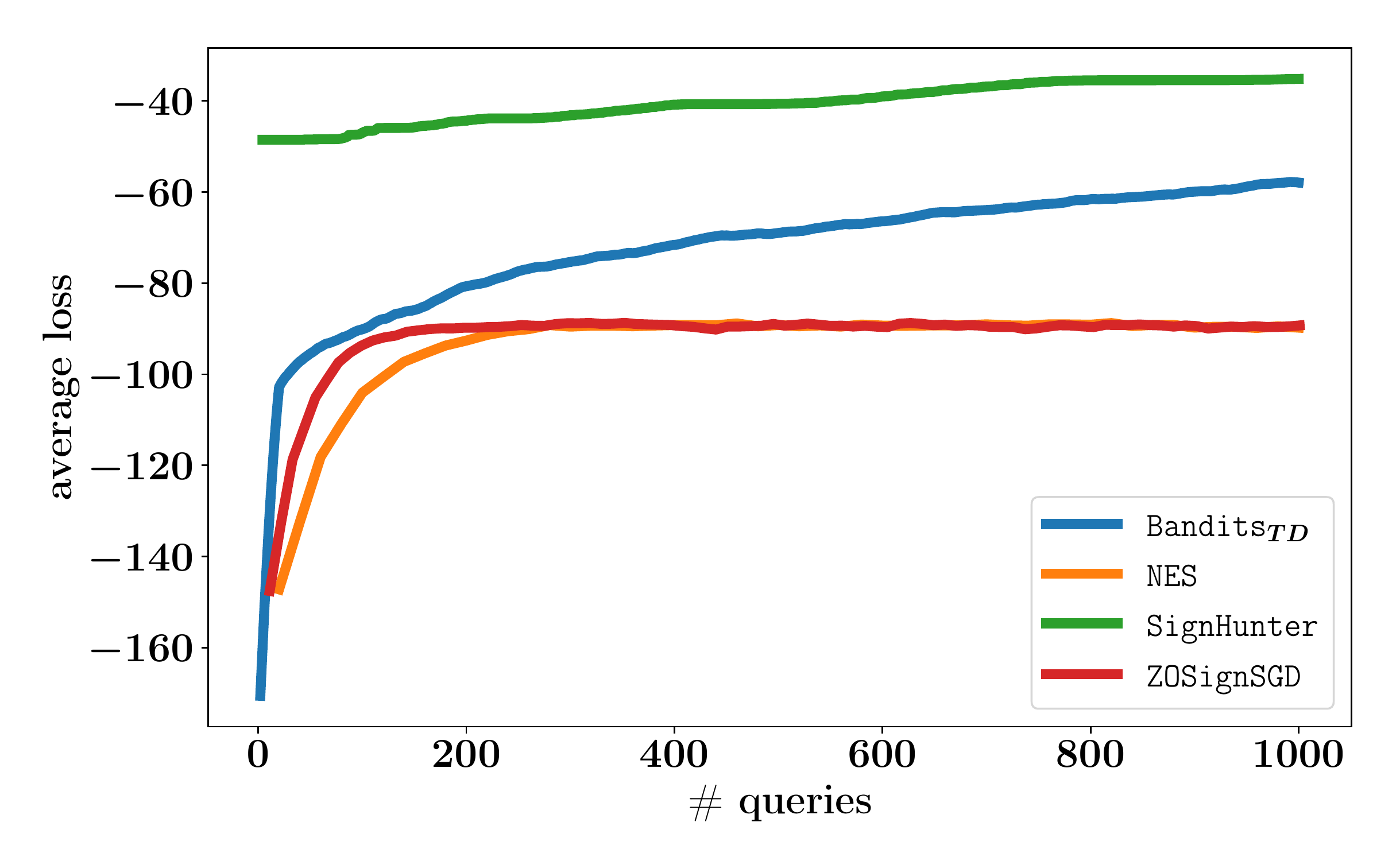} & 
			\includegraphics[width=0.2\textwidth]{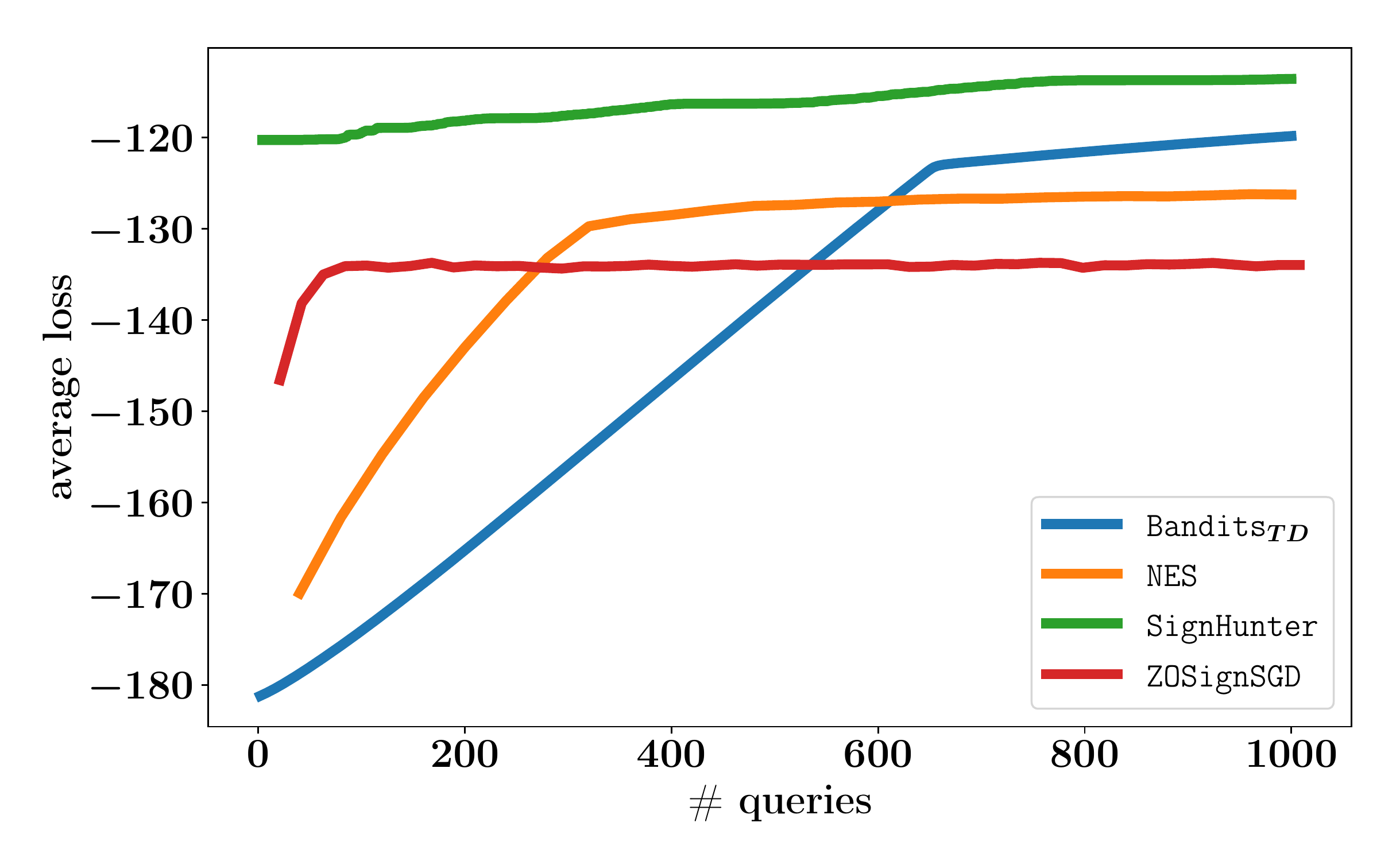} \\
			\tiny \textbf{(a) \mnist~$\linf$} & 
			\tiny \textbf{(b) \mnist~$\ltwo$} \\
			\includegraphics[width=0.2\textwidth]{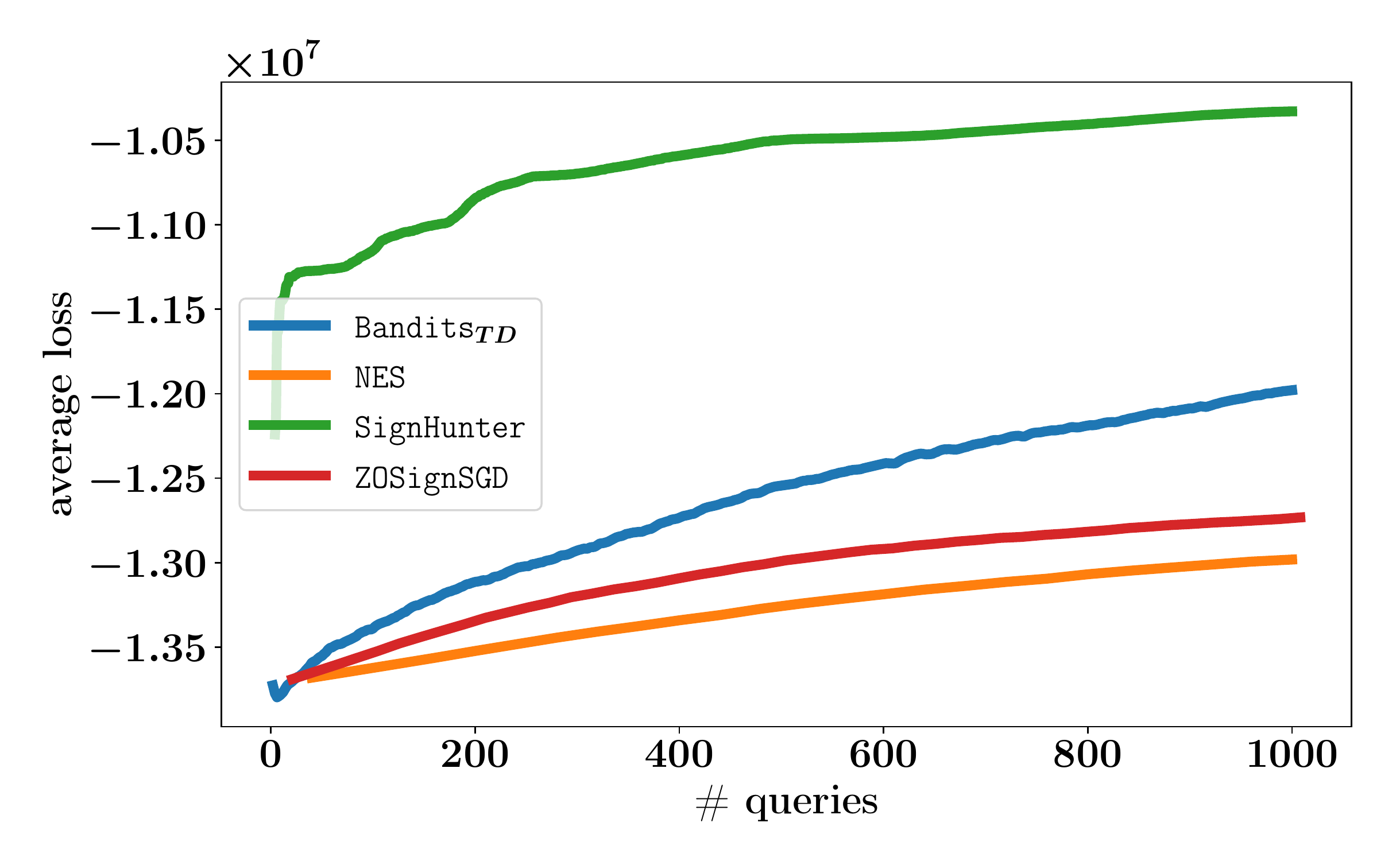} &
			\includegraphics[width=0.2\textwidth]{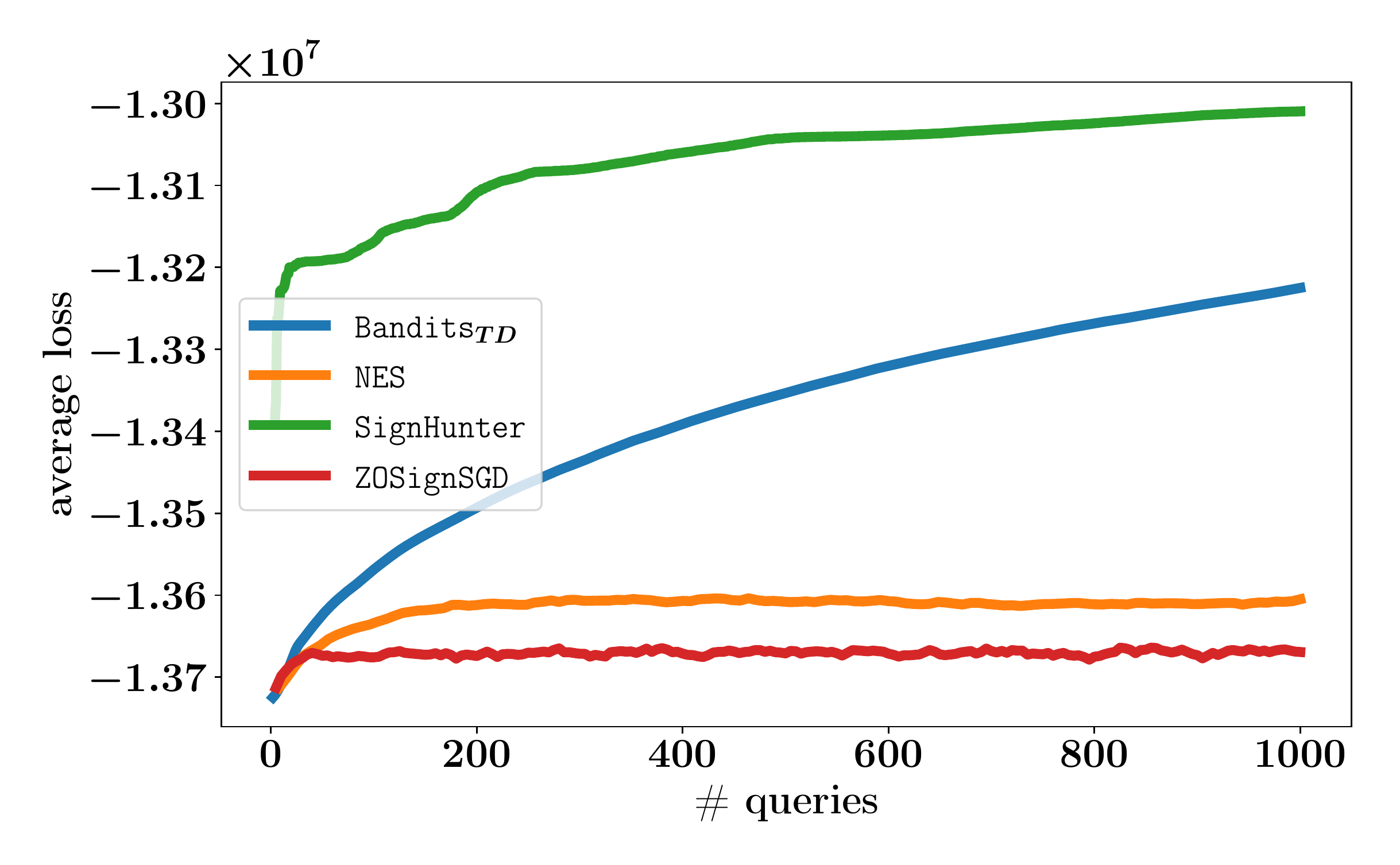} \\ 
			\tiny \textbf{(c) \cifar~$\linf$} & 
			\tiny \textbf{(d) \cifar~$\ltwo$} \\
			\includegraphics[width=0.2\textwidth]{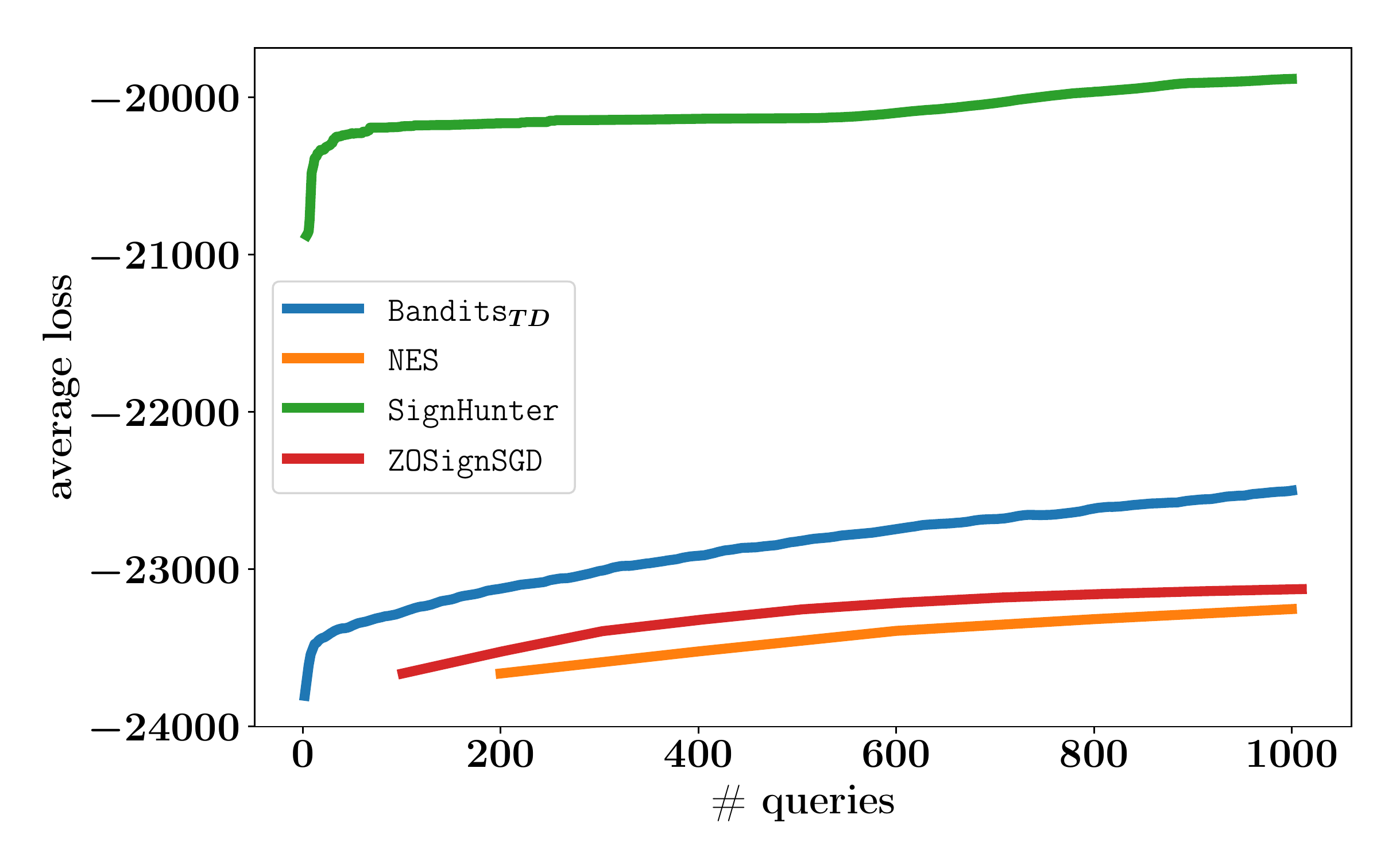} & 
			\includegraphics[width=0.2\textwidth]{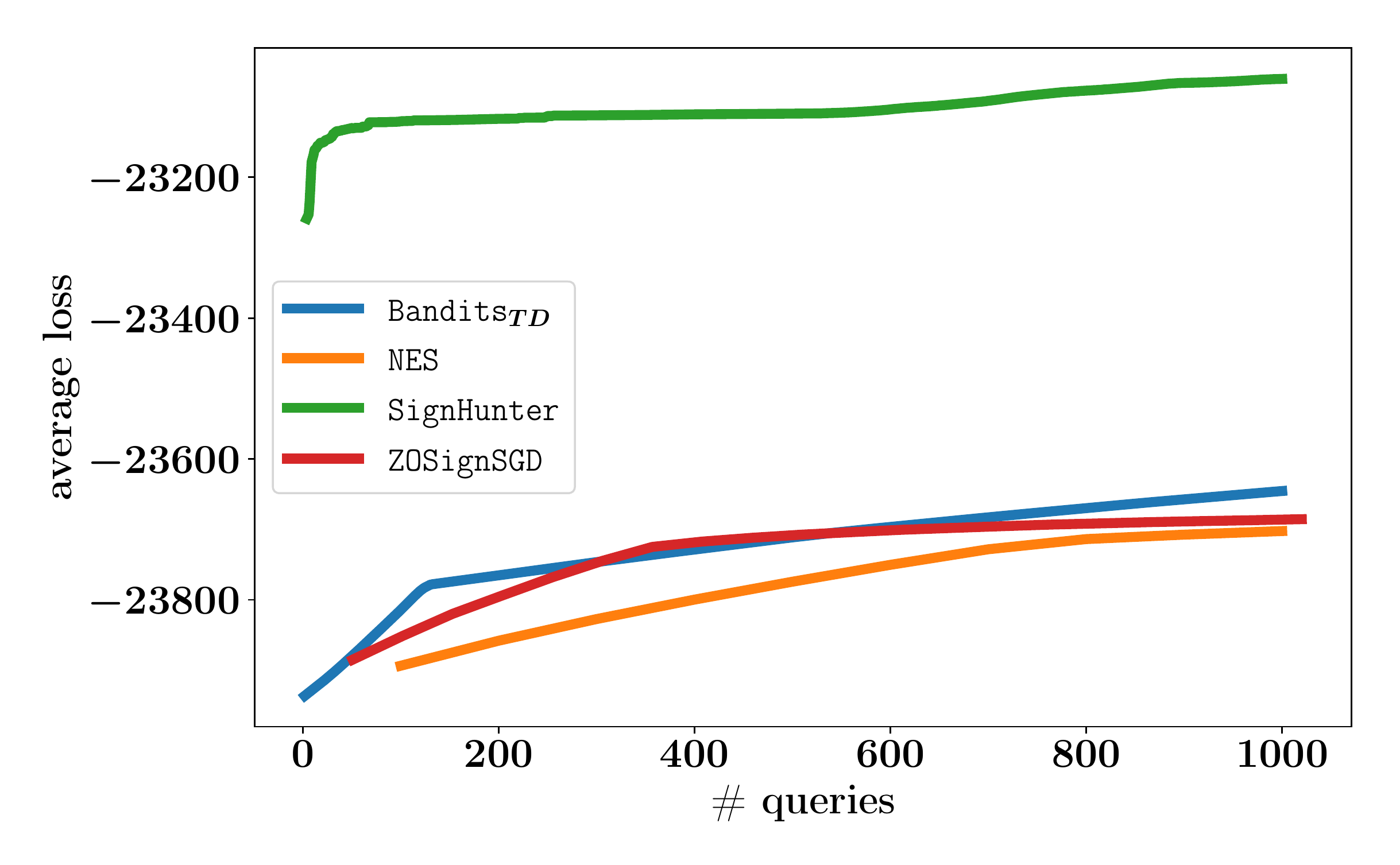} \\
			\tiny \textbf{(e) \imgnt~$\linf$} & 
			\tiny \textbf{(f) \imgnt~$\ltwo$} 
		\end{tabular}
	}
	\caption{Tuning testbed for the attacks. A synthetic loss function was used to tune the performance of the attacks over a random sample of 25 images for each dataset and $\ell_p$ perturbation constraint. The plots above show the average performance of the tuned attacks on the synthetic loss function $L(\vx, y)= - (\vx - \vx^*)^T (\vx - \vx^*)$, where $\vx^*= \frac{\vx_{min} + \vx_{max}}{2}$ using a query limit of $1000$ queries for each image. Note that in all, \bandit outperforms both \nes and \zo. Also, we observe the same behavior reported by \cite{liu2018signsgd} on the fast convergence of \zo compared to \nes. We did not  tune \signhunter; it does not have any tunable parameters.
	}
 \label{fig:tune}
\end{figure}

\begin{table}[h!]
	\caption{General setup for all the attacks}
	\label{tbl:general-setup}
	\centering
	\resizebox{0.4\textwidth}{!}{
		\begin{tabular}{lllllll}
			\toprule
			{} & \multicolumn{6}{l}{\bf{Value}} \\
			{} & \multicolumn{2}{l}{\mnist} & \multicolumn{2}{l}{\cifar} & \multicolumn{2}{l}{\imgnt} \\
			{} & $\ell_{\infty}$ &  $\ell_{2}$ &  $\ell_{\infty}$ &  $\ell_{2}$ &   $\ell_{\infty}$ &  $\ell_{2}$ \\
			\bf{Parameter}                &                 &             &                  &             &                   &             \\
			\midrule
			$\epsilon$ (allowed perturbation)  &             0.3 &           3 &               12 &         127 &              0.05 &           5 \\
			Max allowed queries                &  
			\multicolumn{6}{c}{10000}      \\
			Evaluation/Test set size       & \multicolumn{6}{c}{1000}\\
			Data (pixel value) Range & 
			\multicolumn{2}{c}{[0,1]} & 
			\multicolumn{2}{c}{[0,255]} & 
			\multicolumn{2}{c}{[0,1]} \\
			\bottomrule
		\end{tabular}
	}
\end{table}

\begin{table}[h!]
	\caption{Hyperparameters setup for \nes}
	\centering
	\resizebox{0.4\textwidth}{!}{
		\begin{tabular}{lllllll}
			\toprule
			{} & \multicolumn{6}{l}{\bf{Value}} \\
			{} & \multicolumn{2}{l}{\mnist} & \multicolumn{2}{l}{\cifar} & \multicolumn{2}{l}{\imgnt} \\
			{} & $\ell_{\infty}$ & $\ell_{2}$ &  $\ell_{\infty}$ & $\ell_{2}$ &   $\ell_{\infty}$ & $\ell_{2}$ \\
			\bf{Hyperparameter}                                    &                 &            &                  &            &                   &            \\
			\midrule
			$\delta$ (finite difference probe)                     &             0.1 &        0.1 &             2.55 &       2.55 &               0.1 &        0.1 \\
			$\eta$ (image $\ell_p$ learning rate)                  &             0.1 &          1 &                2 &        127 &              0.02 &          2 \\
			$q$ (number of finite difference estimations per step) &              10 &         20 &               20 &          4 &               100 &         50 \\
			\bottomrule
		\end{tabular}
	}
	\label{tbl:nes-param}
\end{table}

\begin{table}[h!]
	\caption{Hyperparameters setup for \zo}
	\centering
	\resizebox{0.4\textwidth}{!}{
		\begin{tabular}{lllllll}
			\toprule
			{} & \multicolumn{6}{l}{\bf{Value}} \\
			{} & \multicolumn{2}{l}{\mnist} & \multicolumn{2}{l}{\cifar} & \multicolumn{2}{l}{\imgnt} \\
			{} &  $\ell_{\infty}$ &       $\ell_{2}$ &  $\ell_{\infty}$ &       $\ell_{2}$ &   $\ell_{\infty}$ &       $\ell_{2}$ \\
			\bf{Hyperparameter}                                    &                  &                  &                  &                  &                   &                  \\
			\midrule
			$\delta$ (finite difference probe)                     &              0.1 &              0.1 &             2.55 &             2.55 &               0.1 &              0.1 \\
			$\eta$ (image $\ell_p$ learning rate)                  &              0.1 &              0.1 &                2 &                2 &              0.02 &            0.004 \\
			$q$ (number of finite difference estimations per step) &               10 &               20 &               20 &                4 &               100 &               50 \\
			\bottomrule
		\end{tabular}
	}
	\label{tbl:zo-param}
\end{table}

\begin{table}[h!]
	\caption{Hyperparameters setup for \bandit}
	\label{tbl:bandit-param}
	\centering
	\resizebox{0.4\textwidth}{!}{
		\begin{tabular}{lllllll}
			\toprule
			{} & \multicolumn{6}{l}{\bf{Value}} \\
			{} & \multicolumn{2}{l}{\mnist} & \multicolumn{2}{l}{\cifar} & \multicolumn{2}{l}{\imgnt} \\
			{} & $\ell_{\infty}$ &    $\ell_{2}$ &  $\ell_{\infty}$ &    $\ell_{2}$ &   $\ell_{\infty}$ &    $\ell_{2}$ \\
			\bf{Hyperparameter}                                 &                 &               &                  &               &                   &               \\
			\midrule
			$\eta$ (image $\ell_p$ learning rate)               &            0.03 &          0.01 &                5 &            12 &              0.01 &           0.1 \\
			$\delta$ (finite difference probe)                  &             0.1 &           0.1 &             2.55 &          2.55 &               0.1 &           0.1 \\
			$\tau$ (online convex optimization learning rate) &           0.001 &        0.0001 &           0.0001 &         1e-05 &            0.0001 &           0.1 \\
			Tile size (data-dependent prior)                    &               8 &            10 &               20 &            20 &                50 &            50 \\
			$\zeta$ (bandit exploration)                        &            0.01 &           0.1 &              0.1 &           0.1 &              0.01 &           0.1 \\
			\bottomrule
		\end{tabular}
	}
\end{table}

\begin{table}[h!]
	\caption{Hyperparameters setup for \signhunter}
	\label{tbl:sign-param}
	\centering
	\resizebox{0.4\textwidth}{!}{
		\begin{tabular}{lllllll}
			\toprule
			{} & \multicolumn{6}{l}{\bf{Value}} \\
			{} & \multicolumn{2}{l}{\mnist} & \multicolumn{2}{l}{\cifar} & \multicolumn{2}{l}{\imgnt} \\
			{} & $\ell_{\infty}$ &  $\ell_{2}$ &  $\ell_{\infty}$ &  $\ell_{2}$ &   $\ell_{\infty}$ &  $\ell_{2}$ \\
			\bf{Hyperparameter}                &                 &             &                  &             &                   &             \\
			\midrule
			$\delta$ (finite difference probe) &             0.3 &         3 &             12 &        127 &               0.05 &         5 \\
			\bottomrule
		\end{tabular}
	}
\end{table}

\clearpage
\cleardoublepage
\onecolumn
\section*{Appendix E.  Results of Adversarial Black-Box Examples Generation}

This section shows results of our experiments in crafting adversarial black-box examples in the form of tables and performance traces, namely Figures~\ref{fig:mnist-res},~\ref{fig:cifar-res}, and~\ref{fig:imgnet-res}; and Tables~\ref{tbl:mnist_res},~\ref{tbl:cifar_res}, and~\ref{tbl:imgnet_res}.

\begin{table}[h!]
	\caption{Summary of attacks effectiveness on \mnist under $\linf$ and $\ltwo$ perturbation constraints, and with a query limit of $10,000$ queries. The \emph{Failure Rate} $\in [0, 1]$ column lists the fraction of failed attacks over $1000$ images. The \emph{Avg. \# Queries} column reports the average number of queries made to the loss oracle only over successful attacks.}
	\label{tbl:mnist_res}
	\centering
	\resizebox{0.4\textwidth}{!}{
		\begin{tabular}{lllll}
			\toprule
			{} & \multicolumn{2}{l}{\bf{Failure Rate}} & \multicolumn{2}{l}{\bf{Avg. \# Queries}} \\
			{} &  $\bm\ell_\infty$ & $\bm\ell_2$ &       $\bm\ell_\infty$ & $\bm\ell_2$ \\
			\bf{Attack}             &                   &             &                        &             \\
			\midrule
			\texttt{Bandits$_{TD}$} &            $0.68$ &      $0.59$ &               $328.00$ &    $673.16$ \\
			\texttt{NES}            &            $0.63$ &      $0.63$ &               $235.07$ &    $\bf 361.42$ \\
			\texttt{SignHunter}     &            $\bf 0.00$ &      $\bf 0.04$ &               $\bf 11.06$ &   $1064.22$ \\
			\texttt{ZOSignSGD}      &            $0.63$ &      $0.75$ &               $ 157.00$ &    $881.08$ \\
			\bottomrule
		\end{tabular}
	}
\end{table}

\begin{table}[h!]
	\caption{Summary of attacks effectiveness on \cifar under $\linf$ and $\ltwo$ perturbation constraints, and with a query limit of $10,000$ queries. The \emph{Failure Rate} $\in [0, 1]$ column lists the fraction of failed attacks over $1000$ images. The \emph{Avg. \# Queries} column reports the average number of queries made to the loss oracle only over successful attacks.}
	\label{tbl:cifar_res}
	\centering
	\resizebox{0.4\textwidth}{!}{
		\begin{tabular}{lllll}
			\toprule
			{} & \multicolumn{2}{l}{\bf{Failure Rate}} & \multicolumn{2}{l}{\bf{Avg. \# Queries}} \\
			{} &  $\bm\ell_\infty$ & $\bm\ell_2$ &       $\bm\ell_\infty$ & $\bm\ell_2$ \\
			\bf{Attack}             &                   &             &                        &             \\
			\midrule
			\texttt{Bandits$_{TD}$} &            $0.95$ &       $0.39$ &               $432.24$ &       $1201.85$ \\
			\texttt{NES}            &            $0.37$ &       $0.67$ &               $312.57$ &       $\bf 496.99$ \\
			\texttt{SignHunter}     &            $\bf 0.07$ &       $\bf 0.21$ &               $\bf 121.00$ &       $692.39$\\
			\texttt{ZOSignSGD}      &            $0.37$ &      $0.80$ &               $161.28$ &       $528.35$ \\
			\bottomrule
		\end{tabular}
	}
\end{table}

\begin{table}[h!]
	\caption{Summary of attacks effectiveness on \imgnt under $\linf$ and $\ltwo$ perturbation constraints, and with a query limit of $10,000$ queries. The \emph{Failure Rate} $\in [0, 1]$ column lists the fraction of failed attacks over $1000$ images. The \emph{Avg. \# Queries} column reports the average number of queries made to the loss oracle only over successful attacks.}
	\label{tbl:imgnet_res}
	\centering
	\resizebox{0.4\textwidth}{!}{
	\begin{tabular}{lllll}
		\toprule
		{} & \multicolumn{2}{l}{\bf{Failure Rate}} & \multicolumn{2}{l}{\bf{Avg. \# Queries}} \\
		{} &  $\bm\ell_\infty$ & $\bm\ell_2$ &       $\bm\ell_\infty$ & $\bm\ell_2$ \\
		\bf{Attack}             &                   &             &                        &             \\
		\midrule
		\texttt{Bandits$_{TD}$} &            $ 0.07$ &       $\bf 0.11$ &               $1010.05$ &       $1635.55$ \\
		\texttt{NES}            &            $0.26$ &       $0.42$ &               $1536.19$ &       $1393.86$ \\
		\texttt{SignHunter}     &            $\bf 0.02$ &       $0.23$ &               $\bf 578.56$ &       $1985.55$\\
		\texttt{ZOSignSGD}      &            $0.23$ &      $0.52$ &               $1054.98$ &       $\bf 931.15$ \\
		\bottomrule
	\end{tabular}
}
\end{table}

\begin{figure*}[h!]
	\centering
	\resizebox{!}{0.34\textheight}{
		\begin{tabular}{cc}
			 $\linf$ & $\ell_2$\\
 \includegraphics[width=0.4\textwidth]{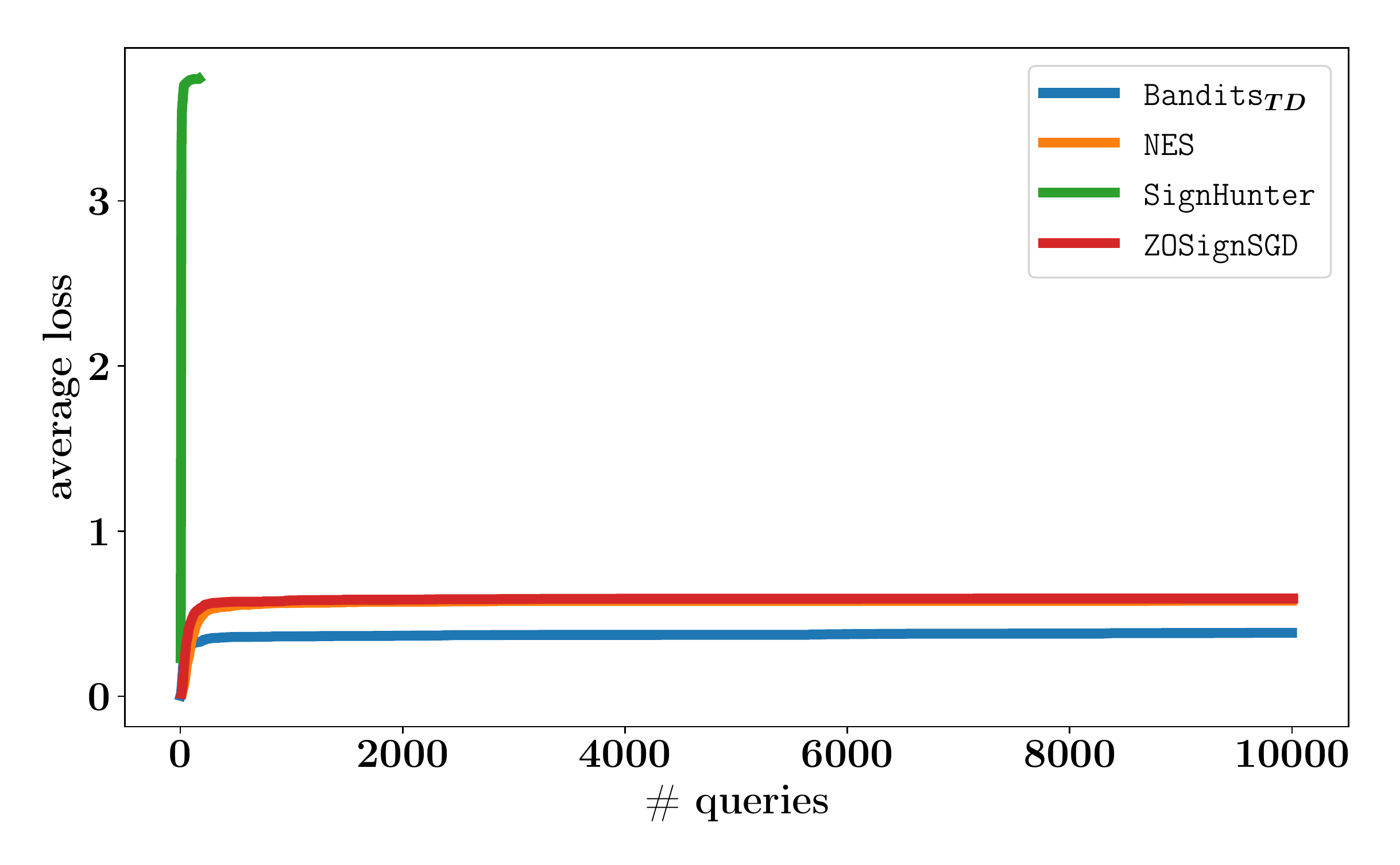} & \includegraphics[width=0.4\textwidth]{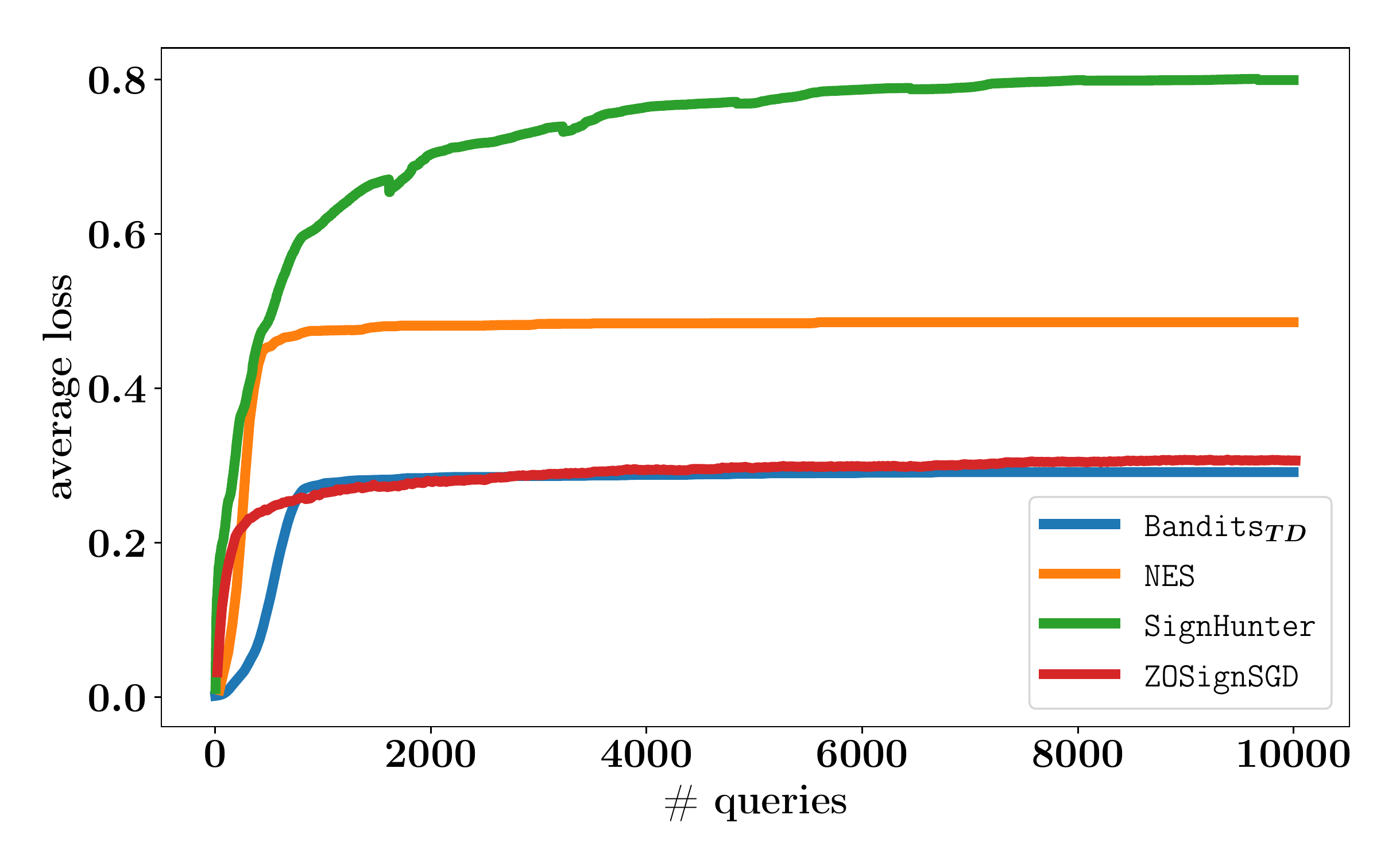} \\ \includegraphics[width=0.4\textwidth]{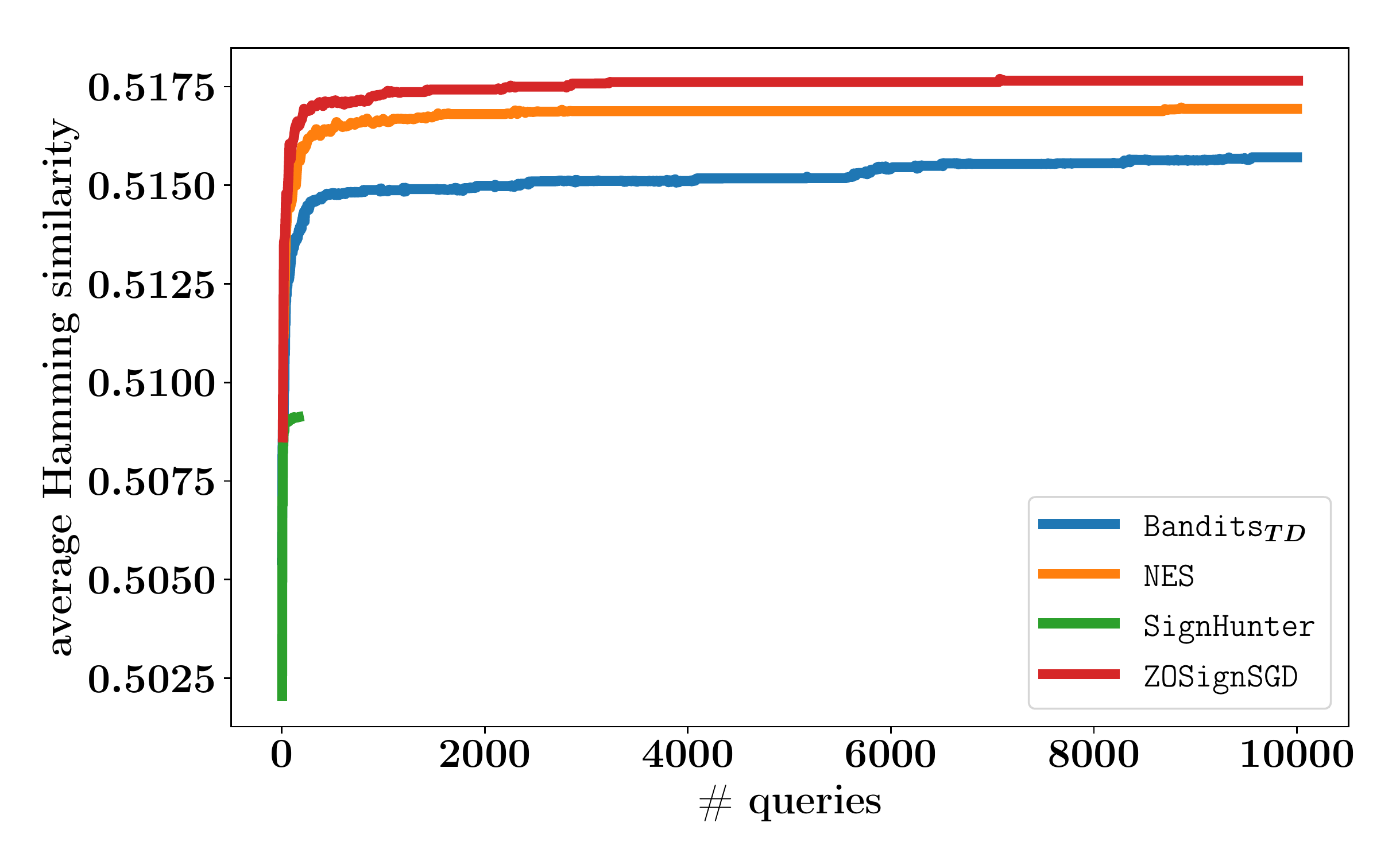} &
			\includegraphics[width=0.4\textwidth]{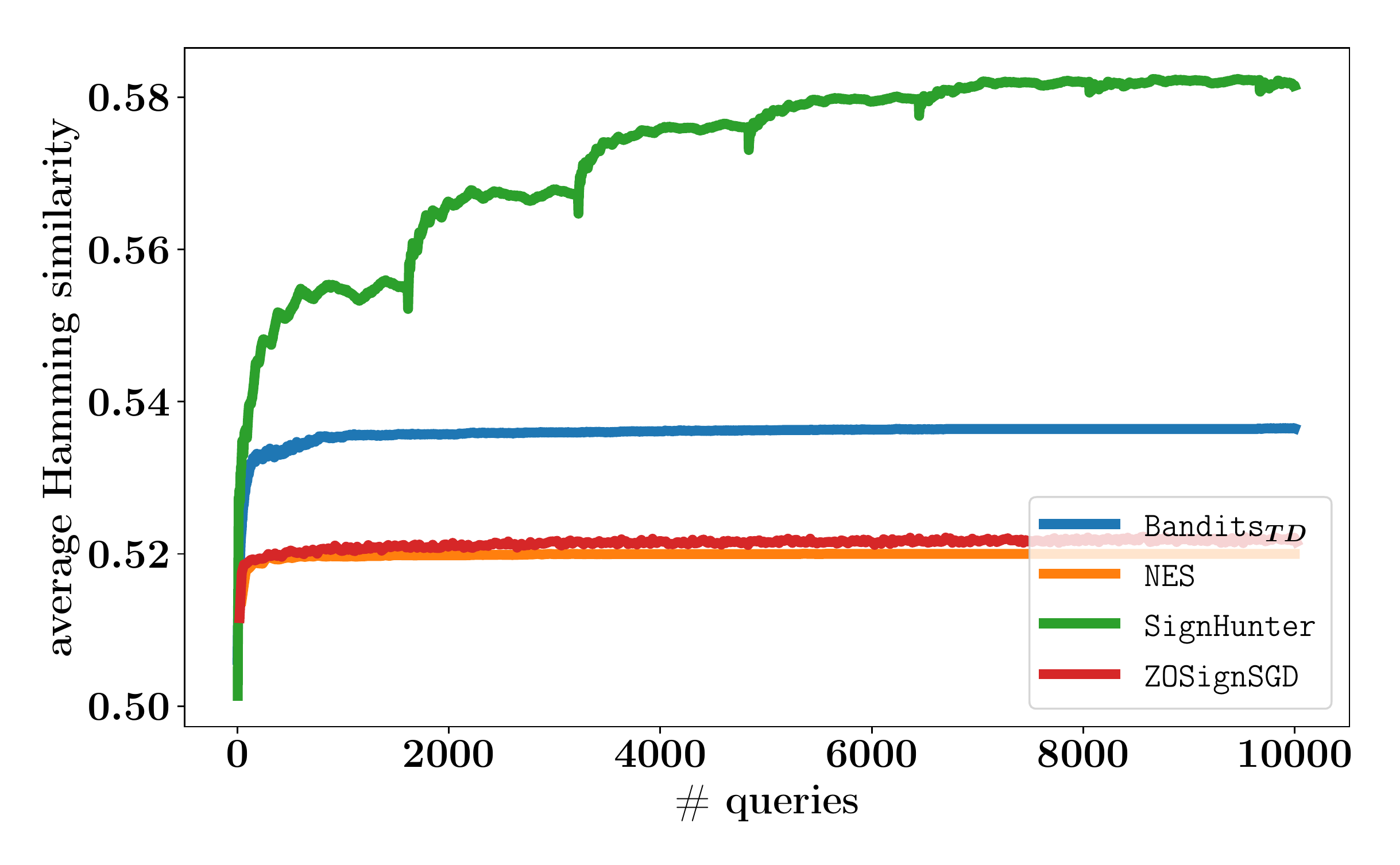}\\
			 \includegraphics[width=0.4\textwidth ]{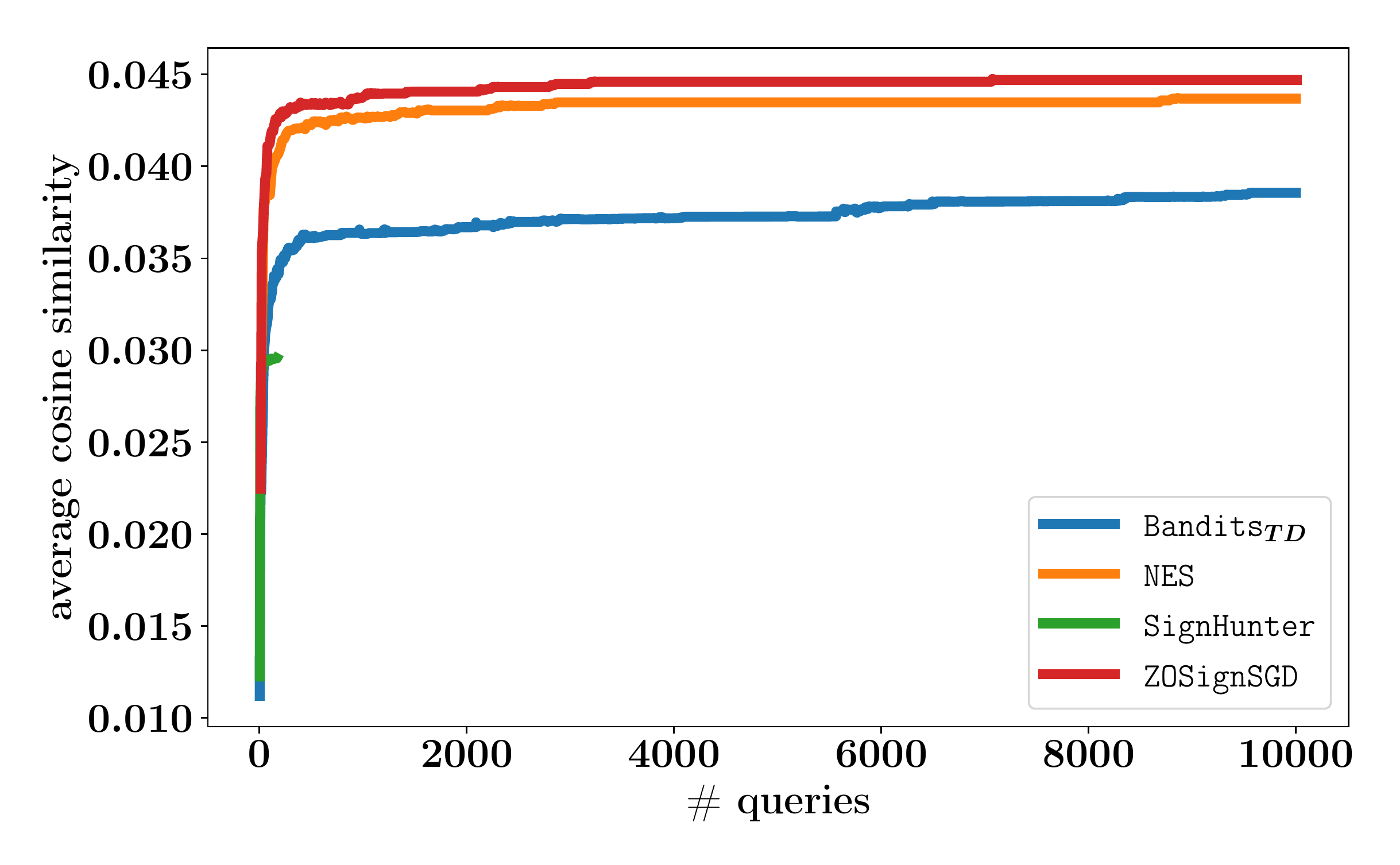} & \includegraphics[width=0.4\textwidth ]{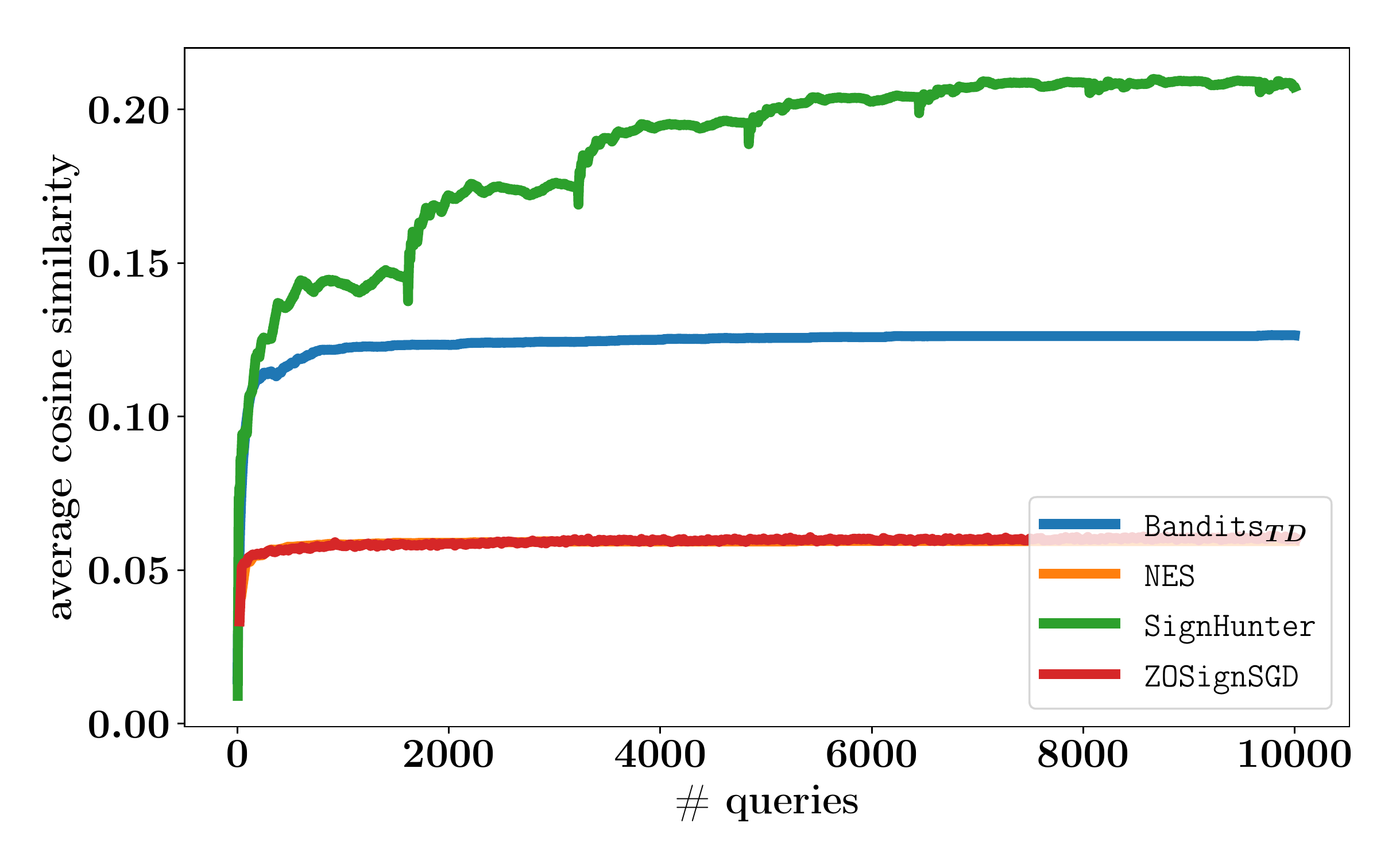} \\
 \includegraphics[width=0.4\textwidth ]{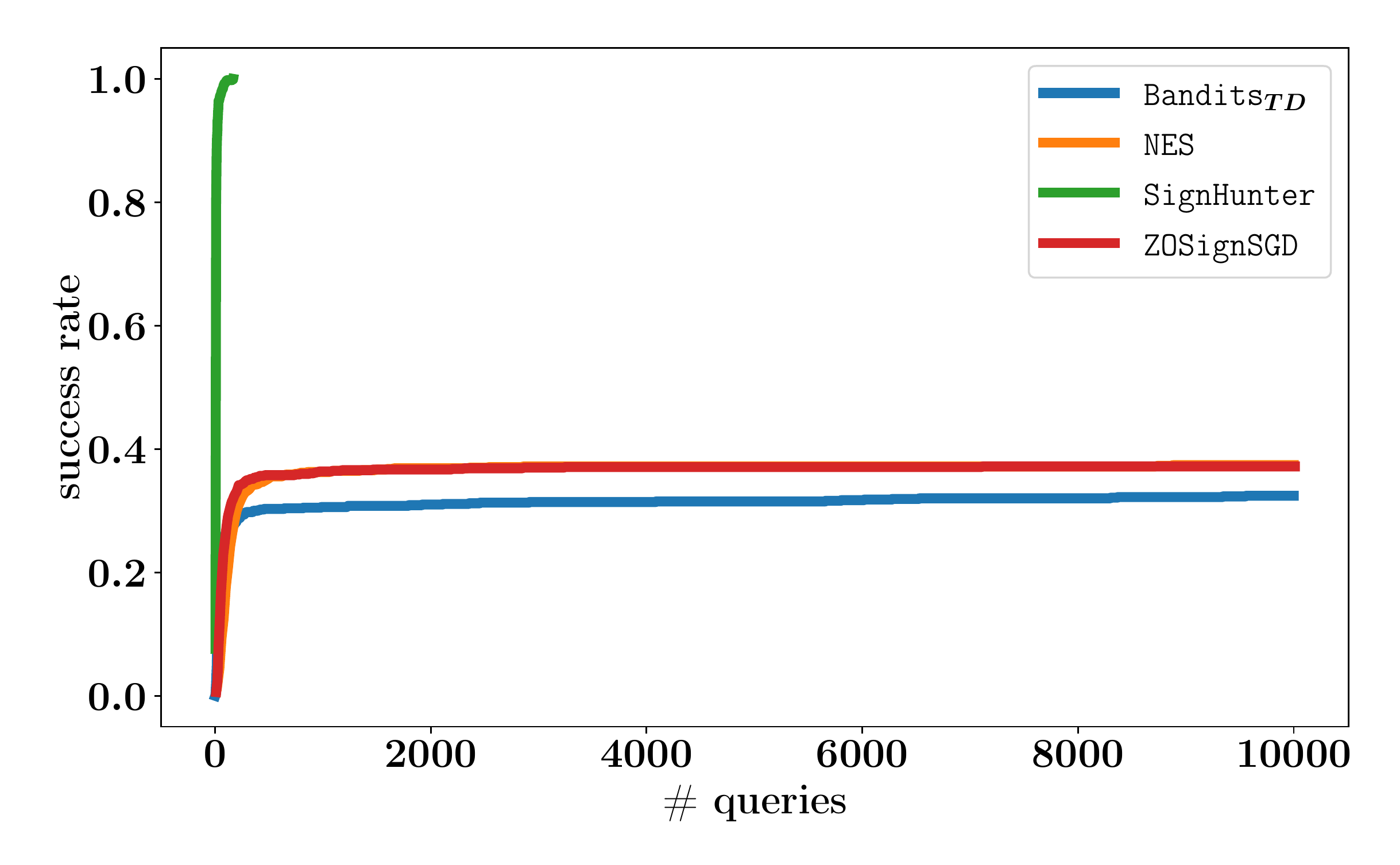} &
			\includegraphics[width=0.4\textwidth ]{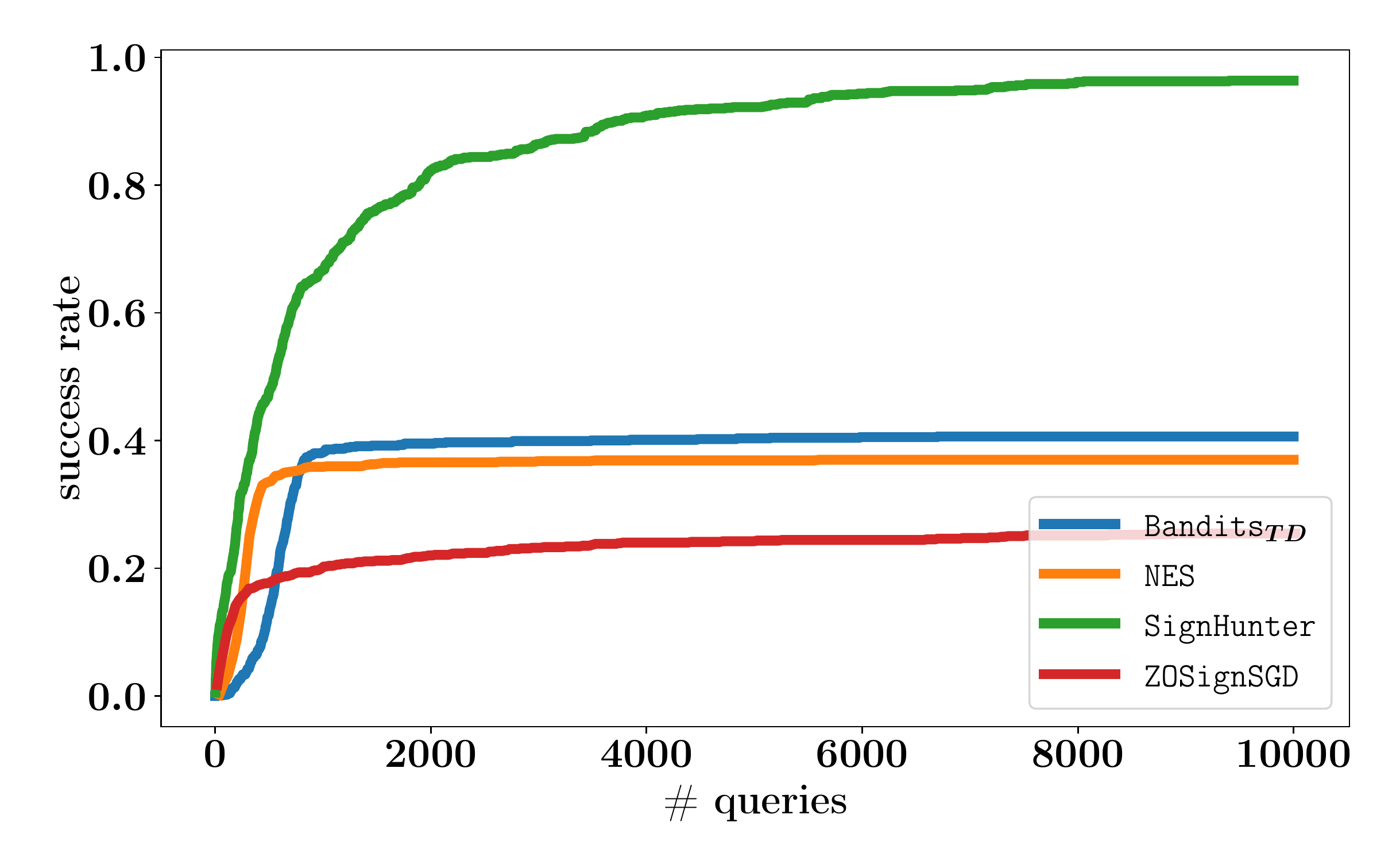} \\
 \includegraphics[width=0.4\textwidth ]{figs/mnist_sota_tbl_plots/mnist_inf_qrt_plt.pdf}  &
			\includegraphics[width=0.4\textwidth ]{figs/mnist_sota_tbl_plots/mnist_2_qrt_plt.pdf} \\
		\end{tabular}
	}
	\caption{Performance curves of attacks on \mnist for $\linf$ (first column) and $\ltwo$ (second column) perturbation constraints. Plots of  \emph{Avg. Loss} row reports the loss as a function of the number of queries averaged over all images. The \emph{Avg. Hamming Similarity} row shows the Hamming similarity of the sign of the attack's estimated gradient $\hat{\vg}$ with true gradient's sign $\vq^*$, computed as $1 - ||\sgn(\hat{\vg}) - \vq^*||_H/ n$ and averaged over all images. Likewise, plots of the \emph{Avg. Cosine Similarity} row show the normalized dot product of $\hat{\vg}$ and $\vg^*$ averaged over all images. The \emph{Success Rate} row reports the attacks' cumulative distribution functions for the number of queries required to carry out a successful attack up to the query limit of $10,000$ queries. The \emph{Avg. \# Queries} row reports the average number of queries used per successful image for each attack when reaching a specified success rate: the more effective the attack, the closer its curve is to the bottom right of the plot.
	}
\label{fig:mnist-res}
\end{figure*}

\begin{figure*}[h!]
	\centering
	\resizebox{!}{0.34\textheight}{
		\begin{tabular}{cc}
			$\linf$ & $\ell_2$\\
			\includegraphics[width=0.4\textwidth]{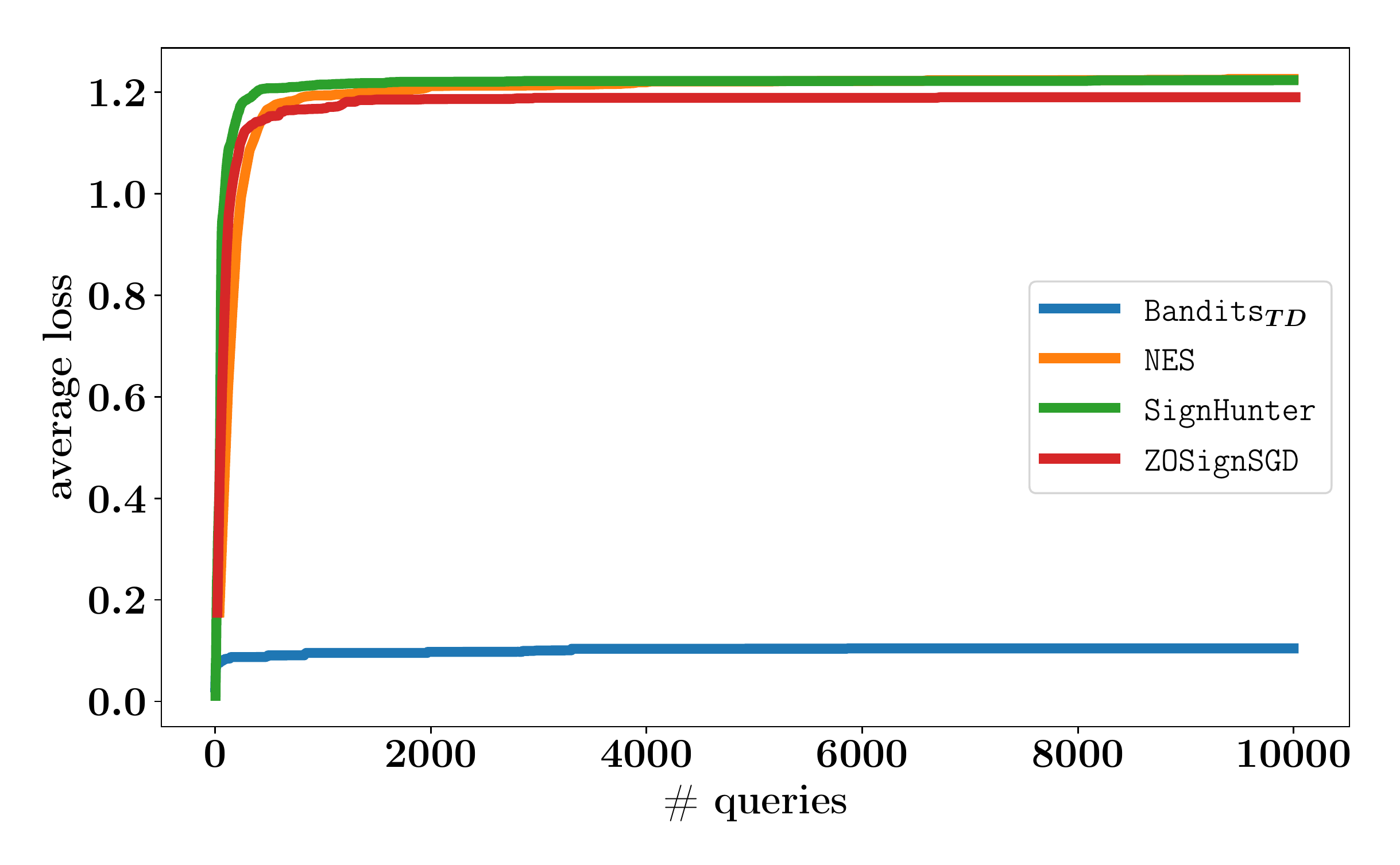} & \includegraphics[width=0.4\textwidth]{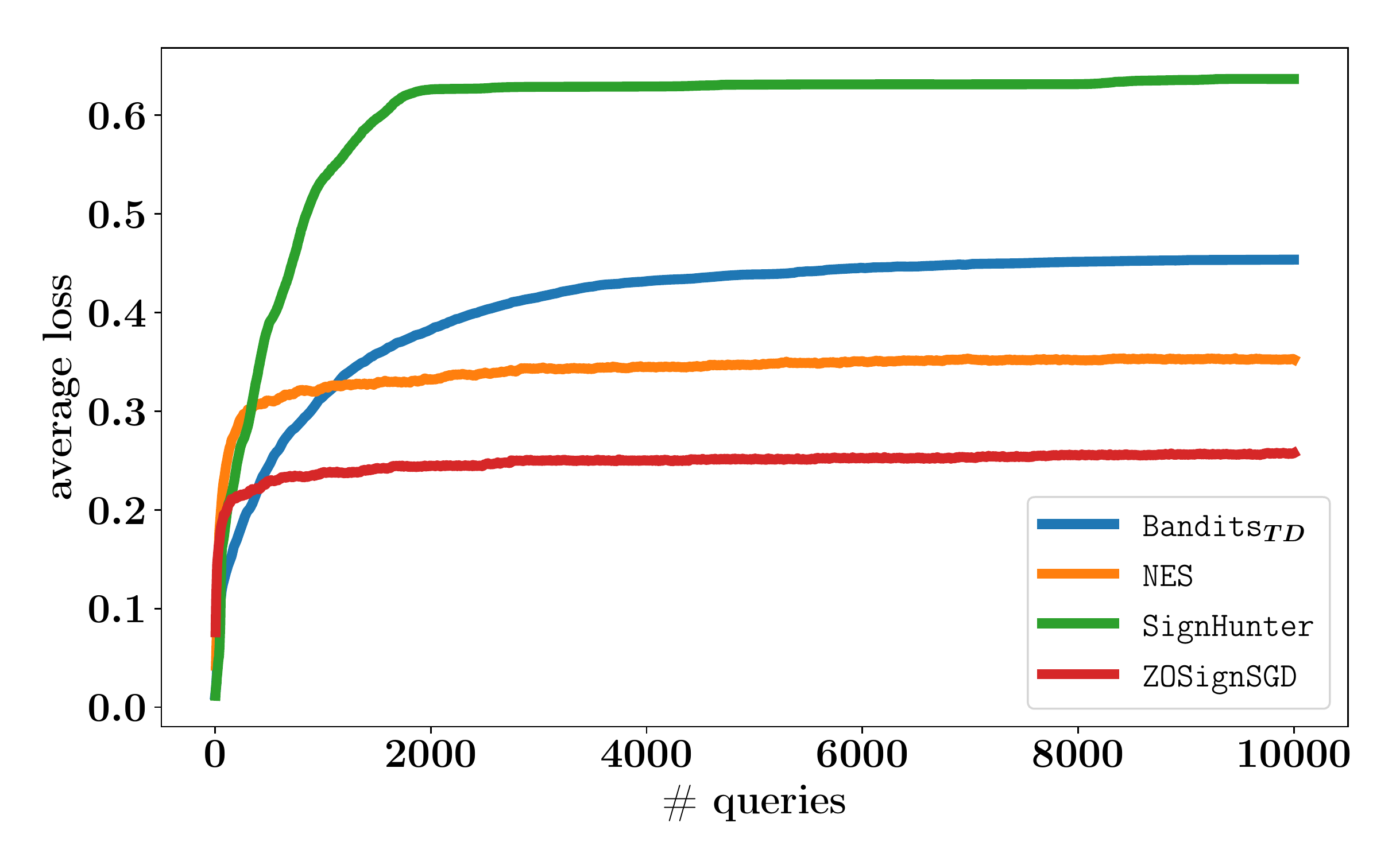} \\ \includegraphics[width=0.4\textwidth]{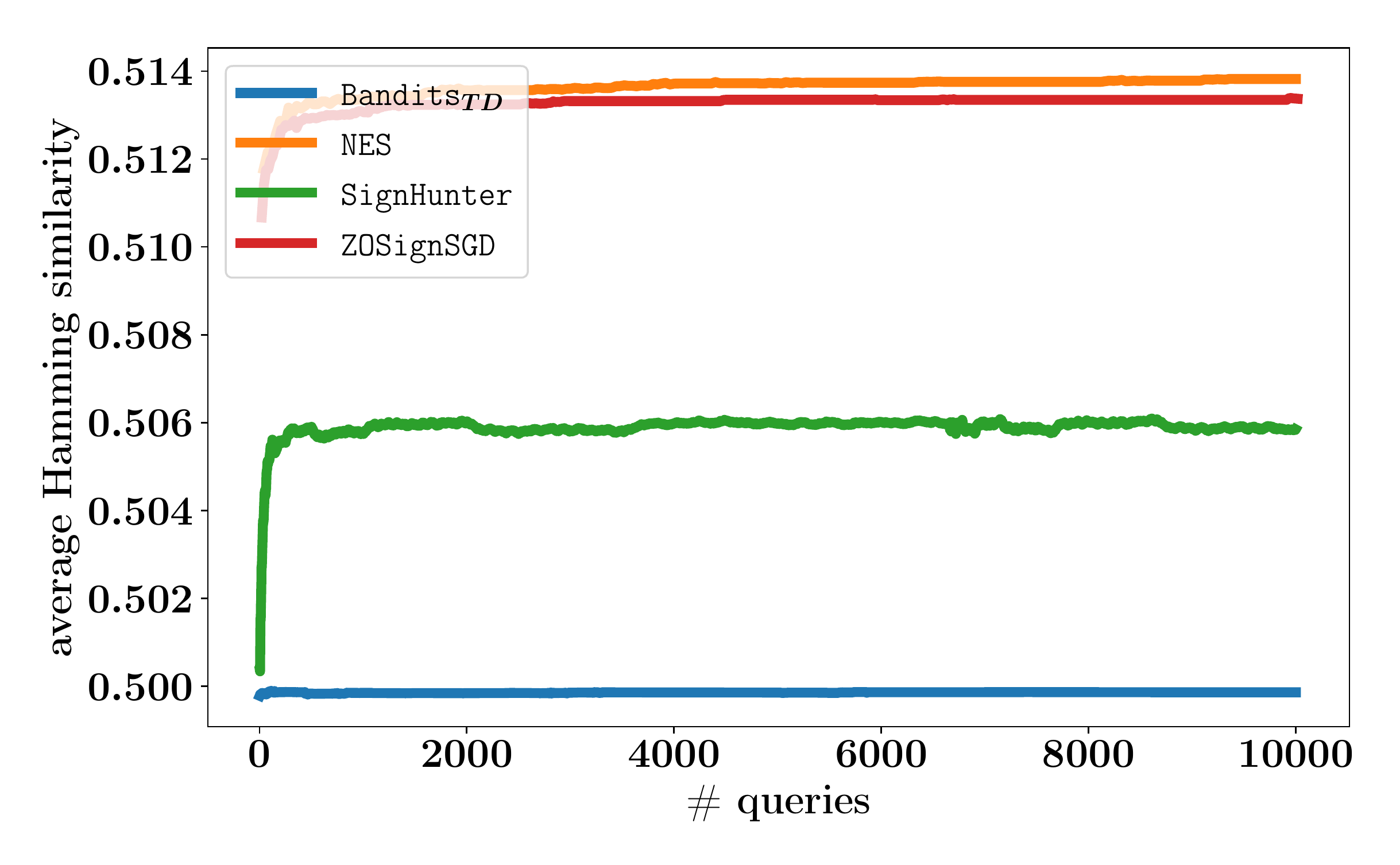} &
			\includegraphics[width=0.4\textwidth]{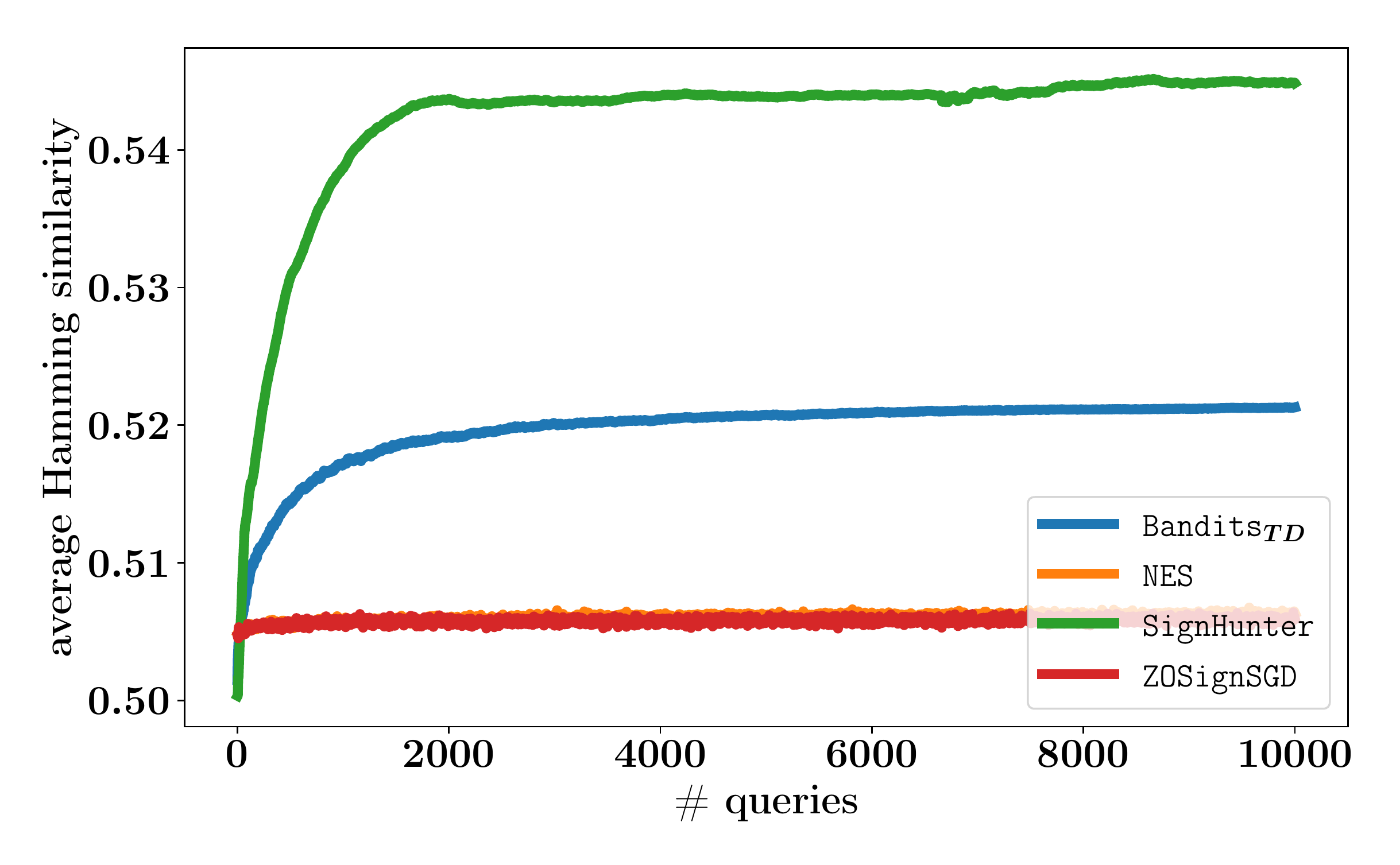}\\
			\includegraphics[width=0.4\textwidth ]{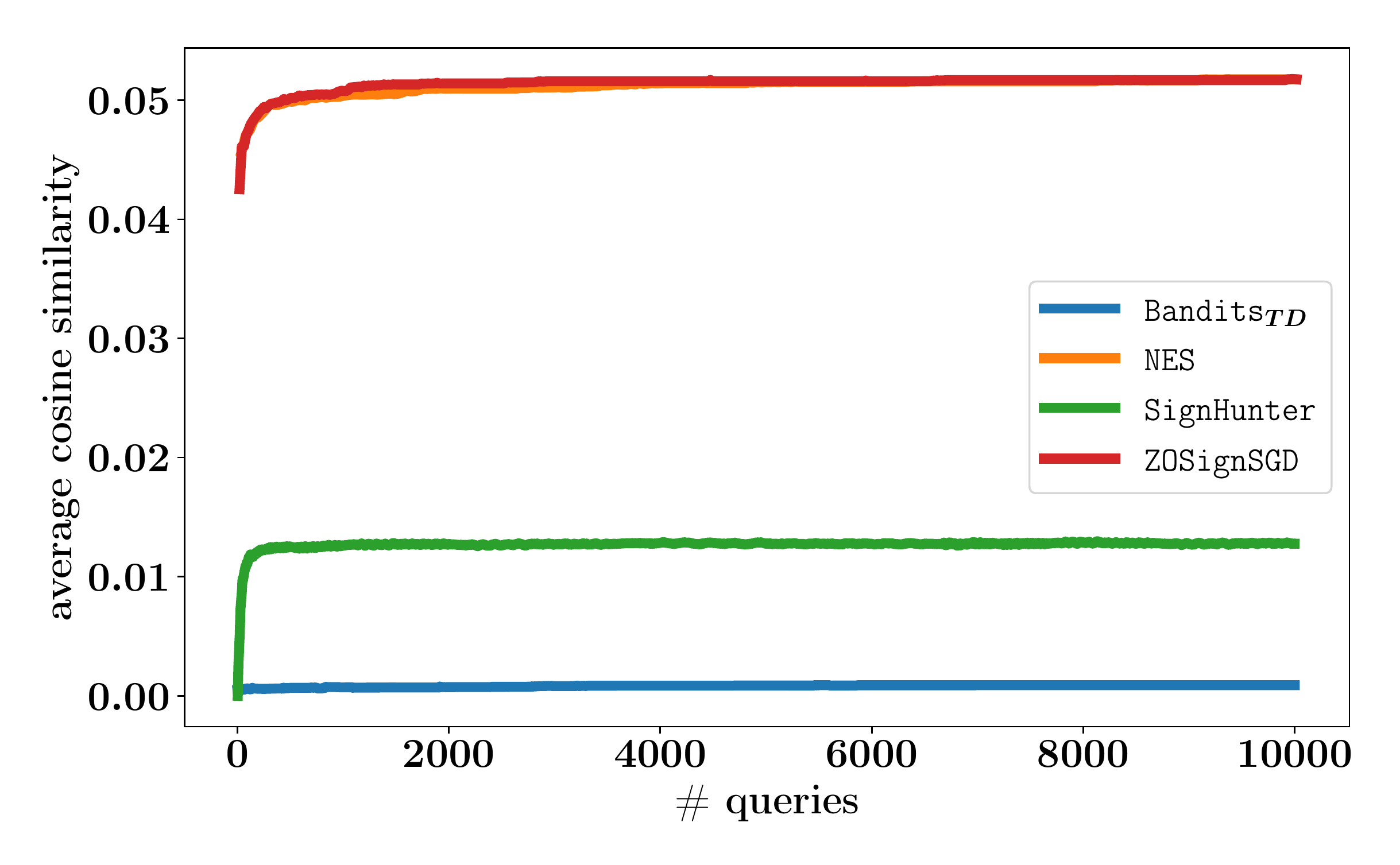} & \includegraphics[width=0.4\textwidth ]{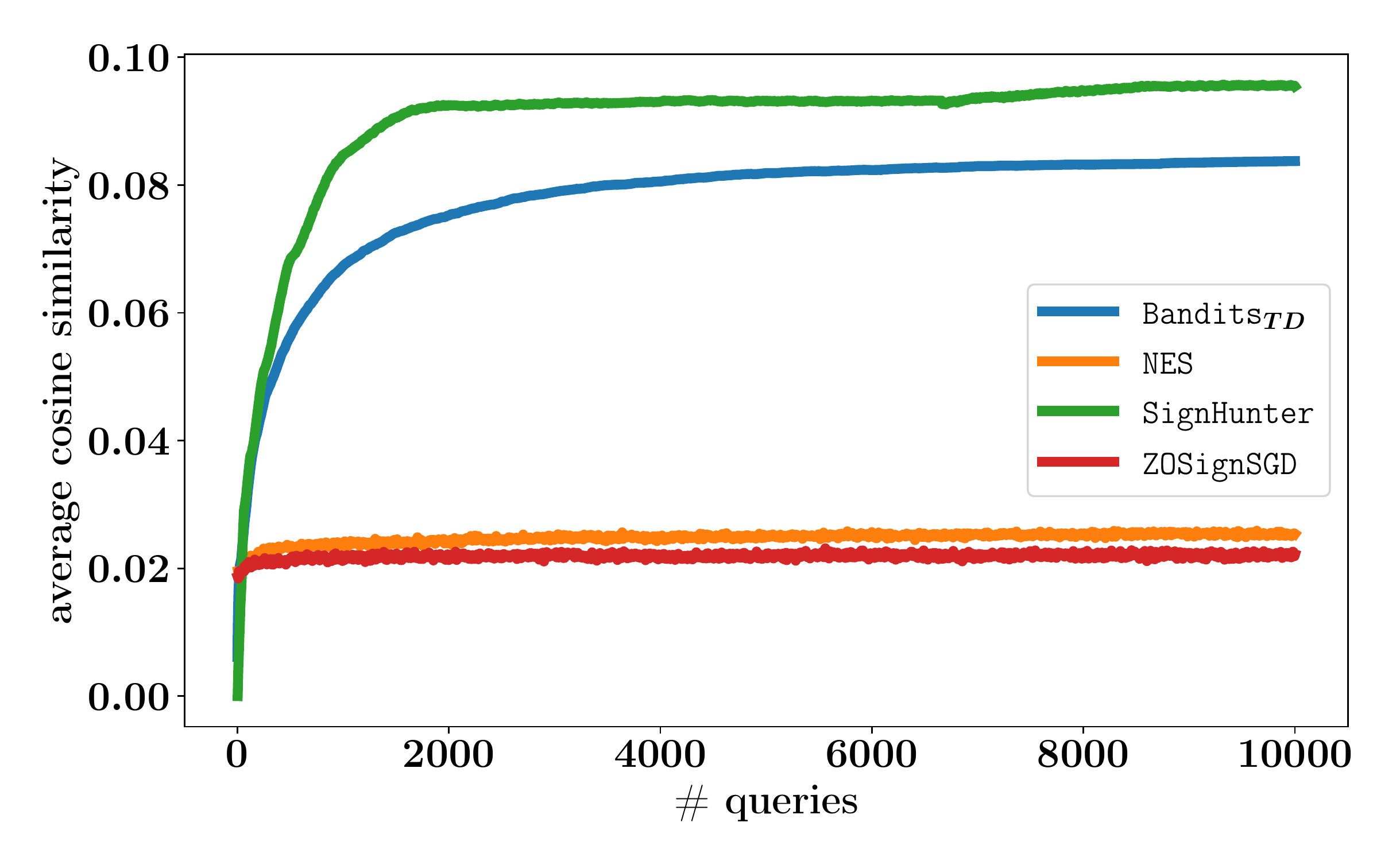} \\
			\includegraphics[width=0.4\textwidth ]{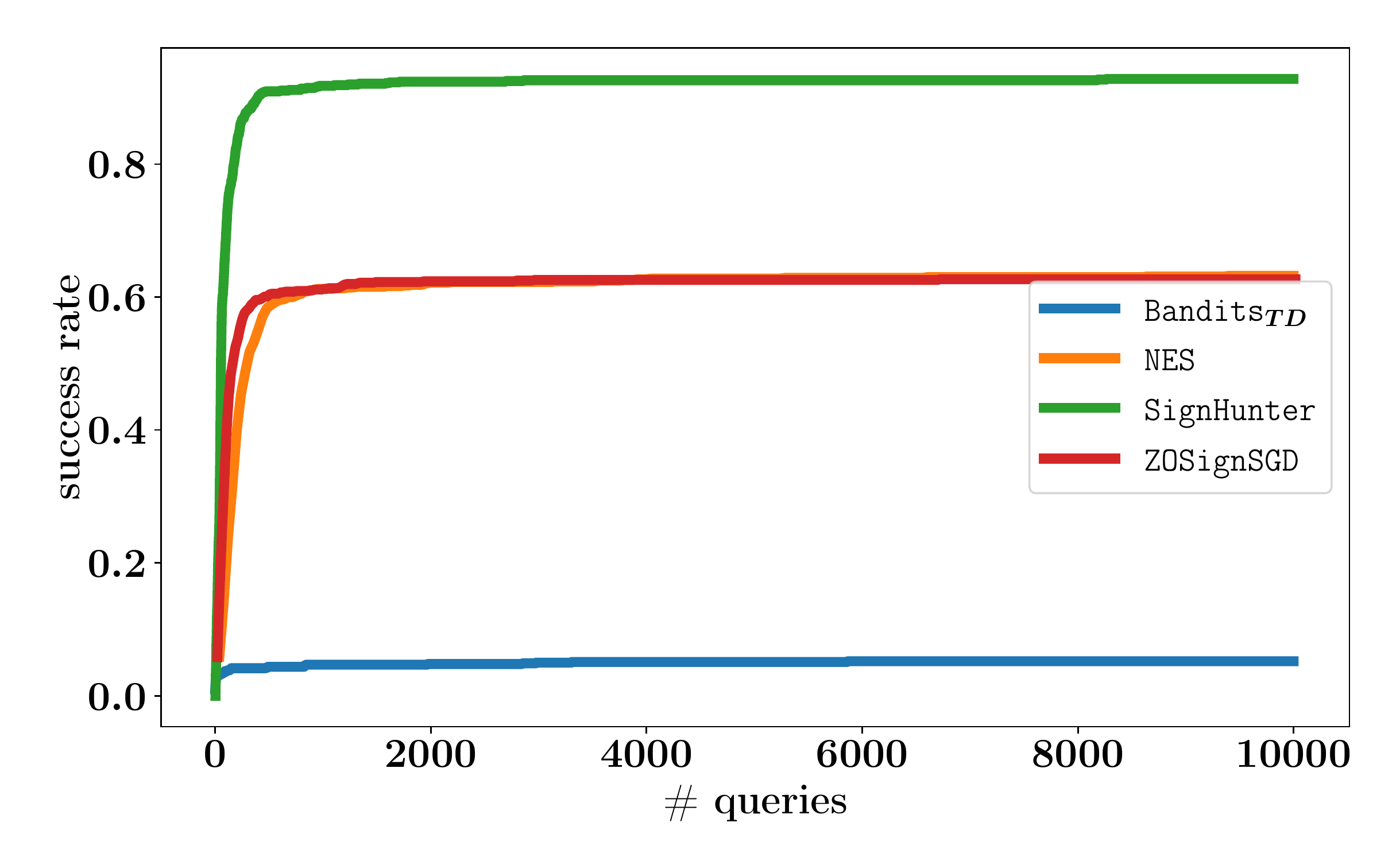} &
			\includegraphics[width=0.4\textwidth ]{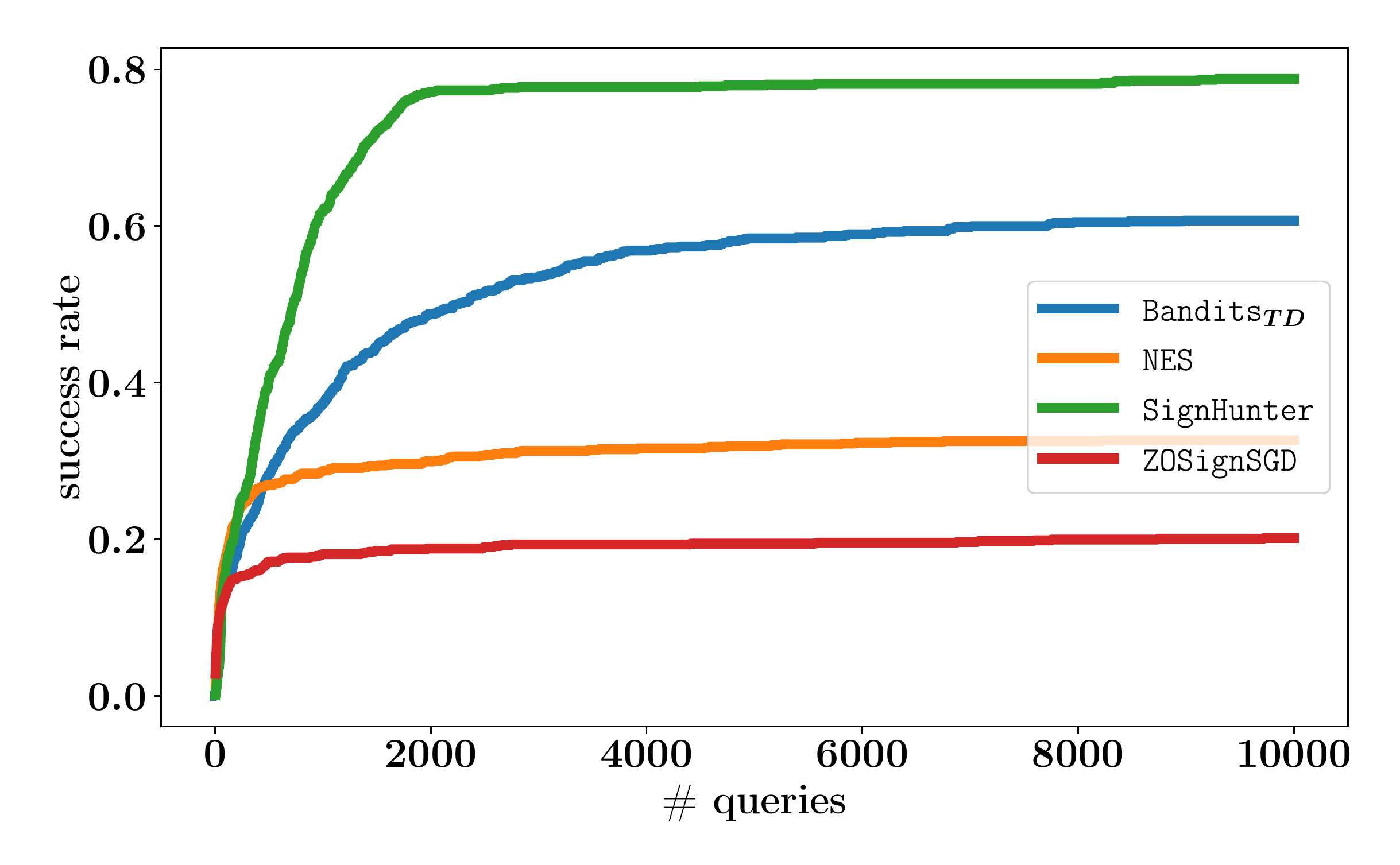} \\
			\includegraphics[width=0.4\textwidth ]{figs/cifar10_sota_tbl_plots/cifar10_inf_qrt_plt.pdf}  &
			\includegraphics[width=0.4\textwidth ]{figs/cifar10_sota_tbl_plots/cifar10_2_qrt_plt.pdf} \\
		\end{tabular}
	}
	\caption{Performance curves of attacks on \cifar for $\linf$ (first column) and $\ltwo$ (second column) perturbation constraints. Plots of  \emph{Avg. Loss} row reports the loss as a function of the number of queries averaged over all images. The \emph{Avg. Hamming Similarity} row shows the Hamming similarity of the sign of the attack's estimated gradient $\hat{\vg}$ with true gradient's sign $\vq^*$, computed as $1 - ||\sgn(\hat{\vg}) - \vq^*||_H/ n$ and averaged over all images. Likewise, plots of the \emph{Avg. Cosine Similarity} row show the normalized dot product of $\hat{\vg}$ and $\vg^*$ averaged over all images. The \emph{Success Rate} row reports the attacks' cumulative distribution functions for the number of queries required to carry out a successful attack up to the query limit of $10,000$ queries. The \emph{Avg. \# Queries} row reports the average number of queries used per successful image for each attack when reaching a specified success rate: the more effective the attack, the closer its curve is to the bottom right of the plot.
	}
\label{fig:cifar-res}
\end{figure*}

% TODO replace cifar with imagenet figs.

\begin{figure*}[h!]
	\centering
	\resizebox{!}{0.34\textheight}{
		\begin{tabular}{cc}
			$\linf$ & $\ell_2$\\
			\includegraphics[width=0.4\textwidth]{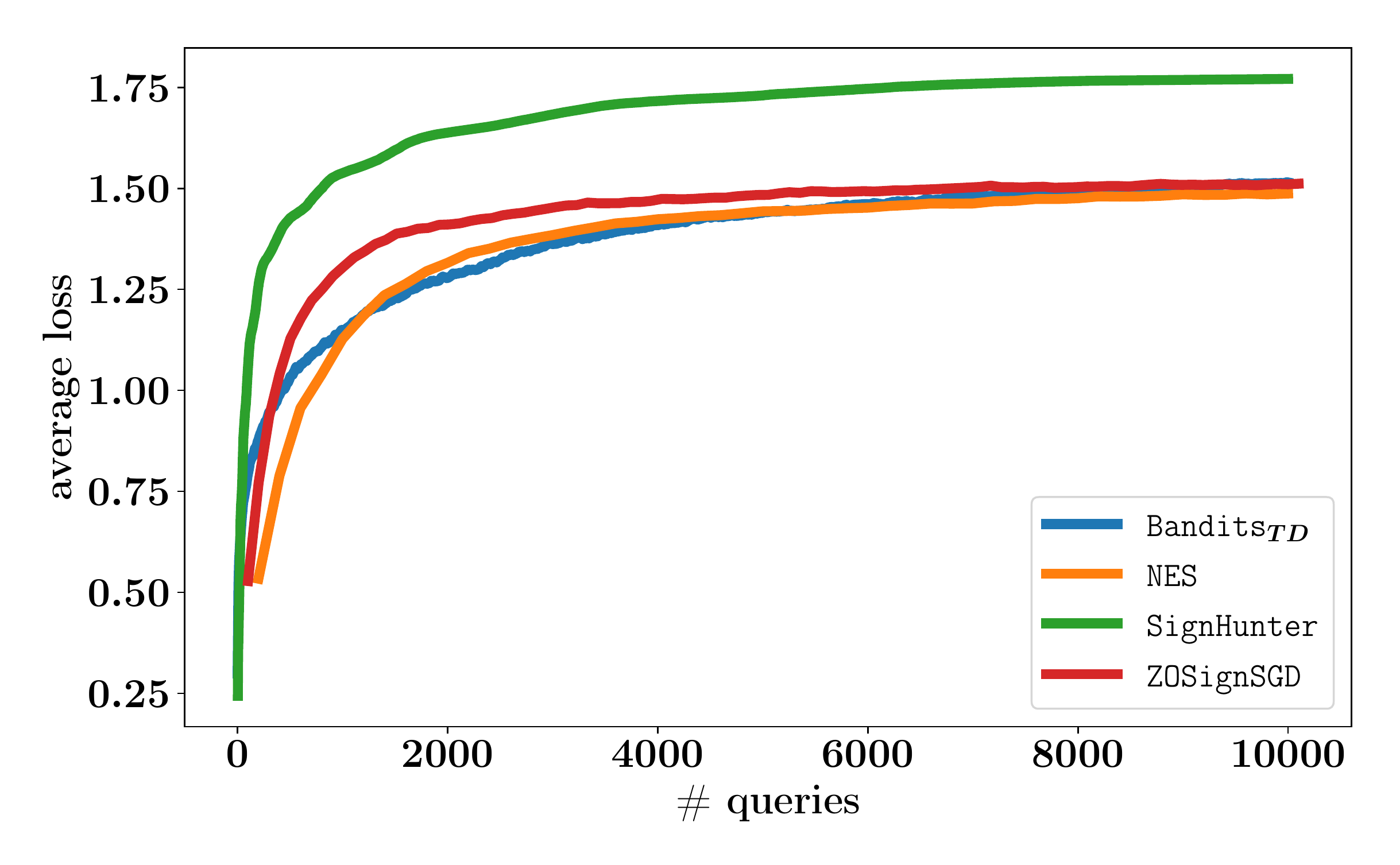} & \includegraphics[width=0.4\textwidth]{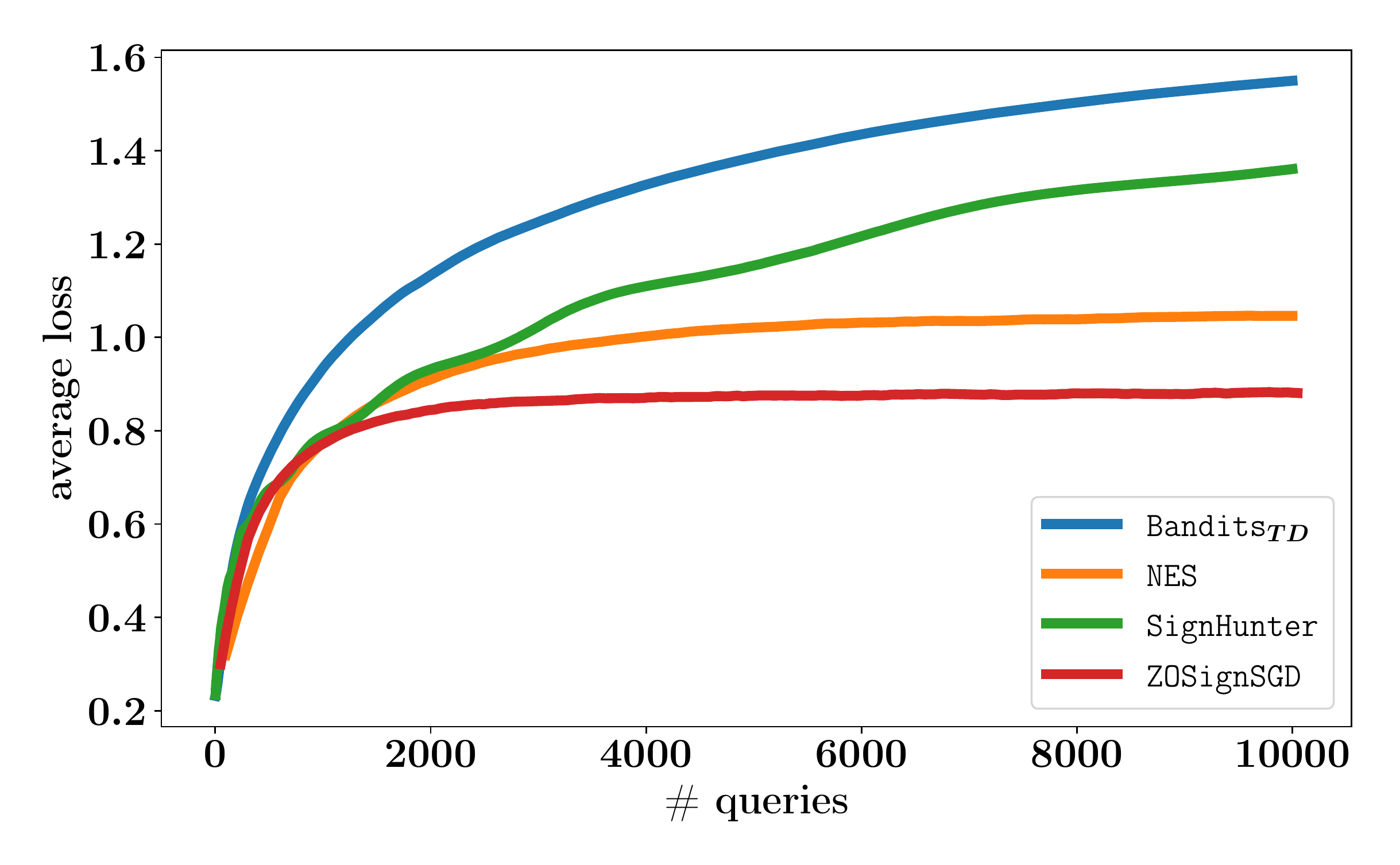} \\ \includegraphics[width=0.4\textwidth]{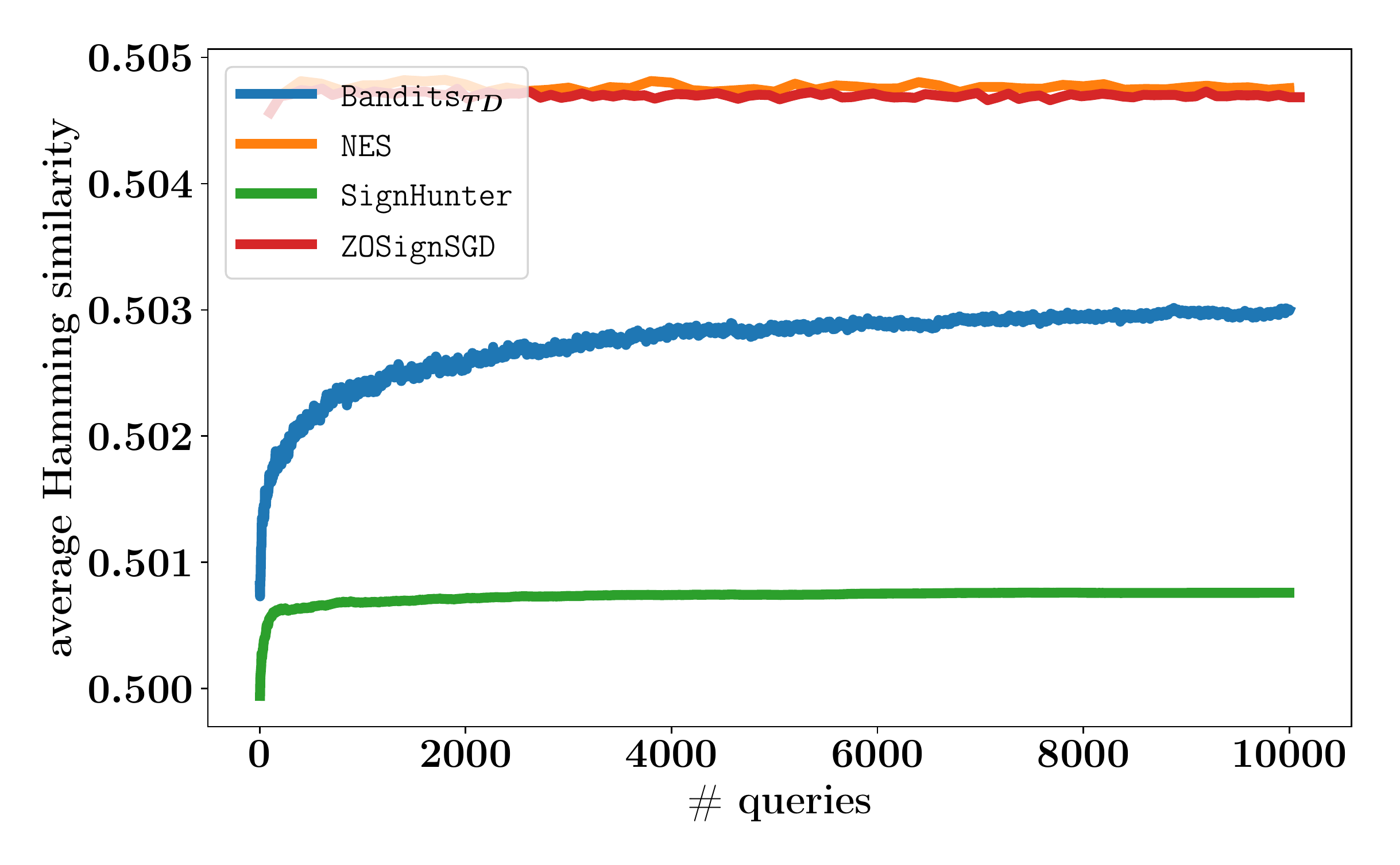} &
			\includegraphics[width=0.4\textwidth]{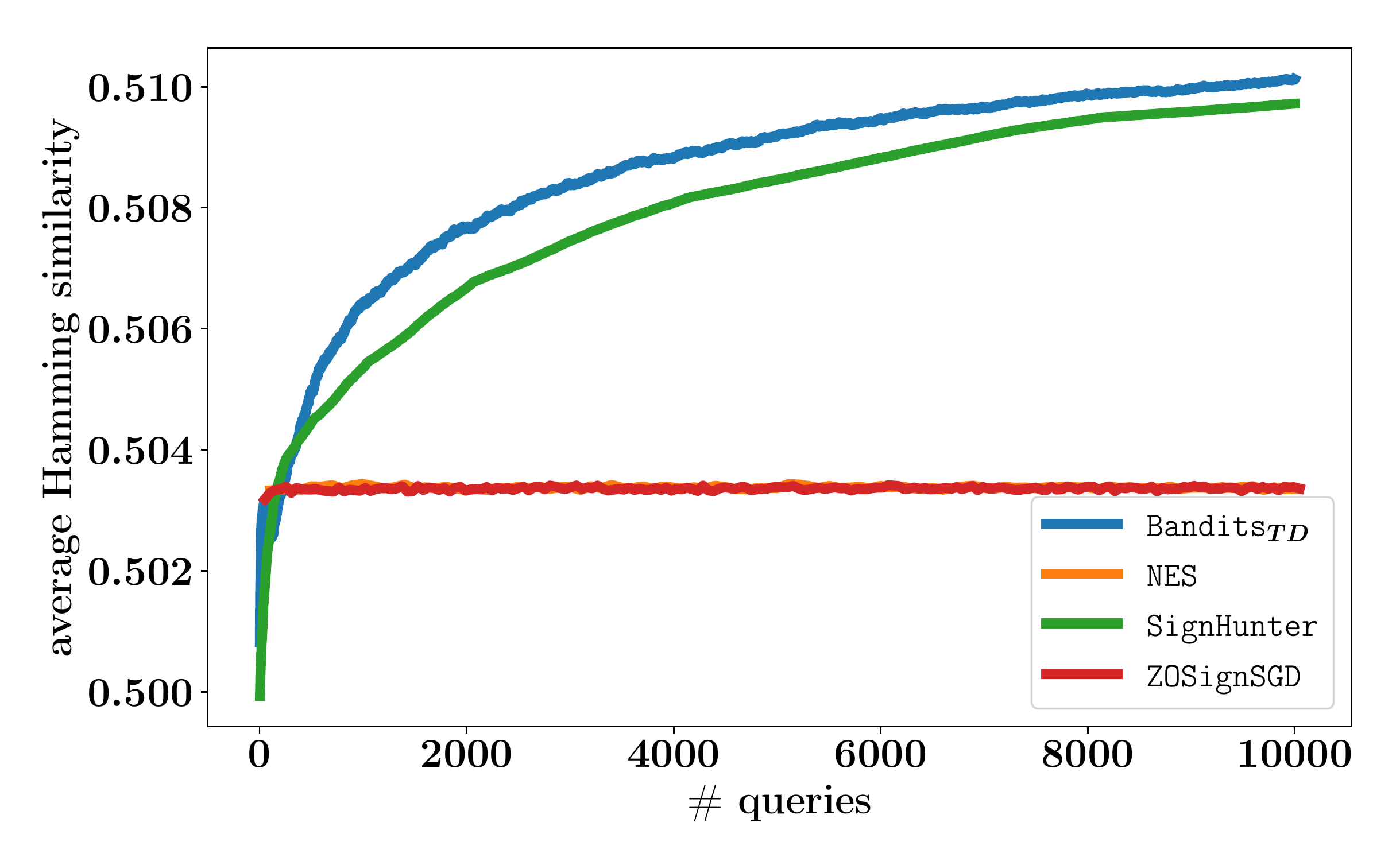}\\
			\includegraphics[width=0.4\textwidth ]{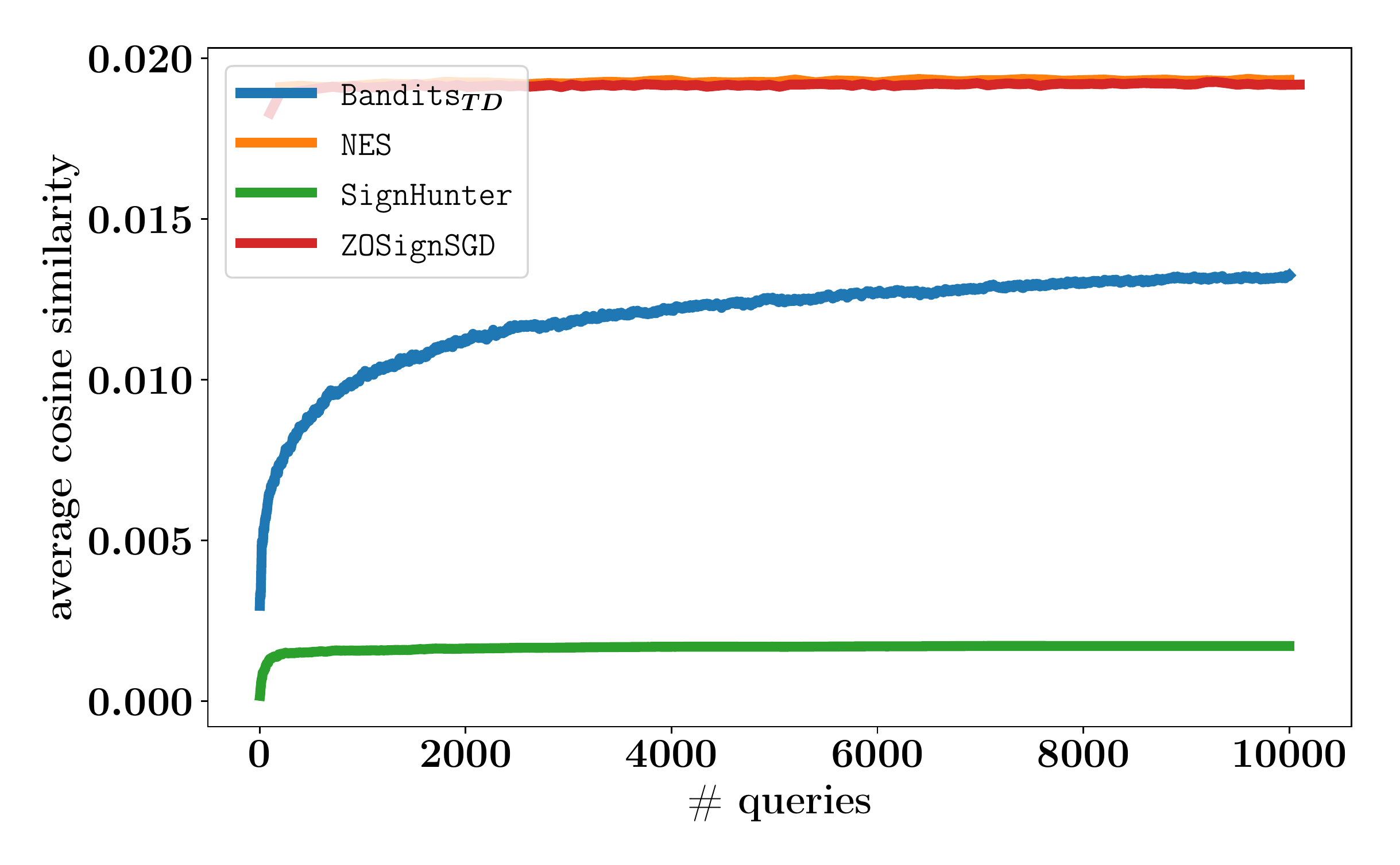} & \includegraphics[width=0.4\textwidth ]{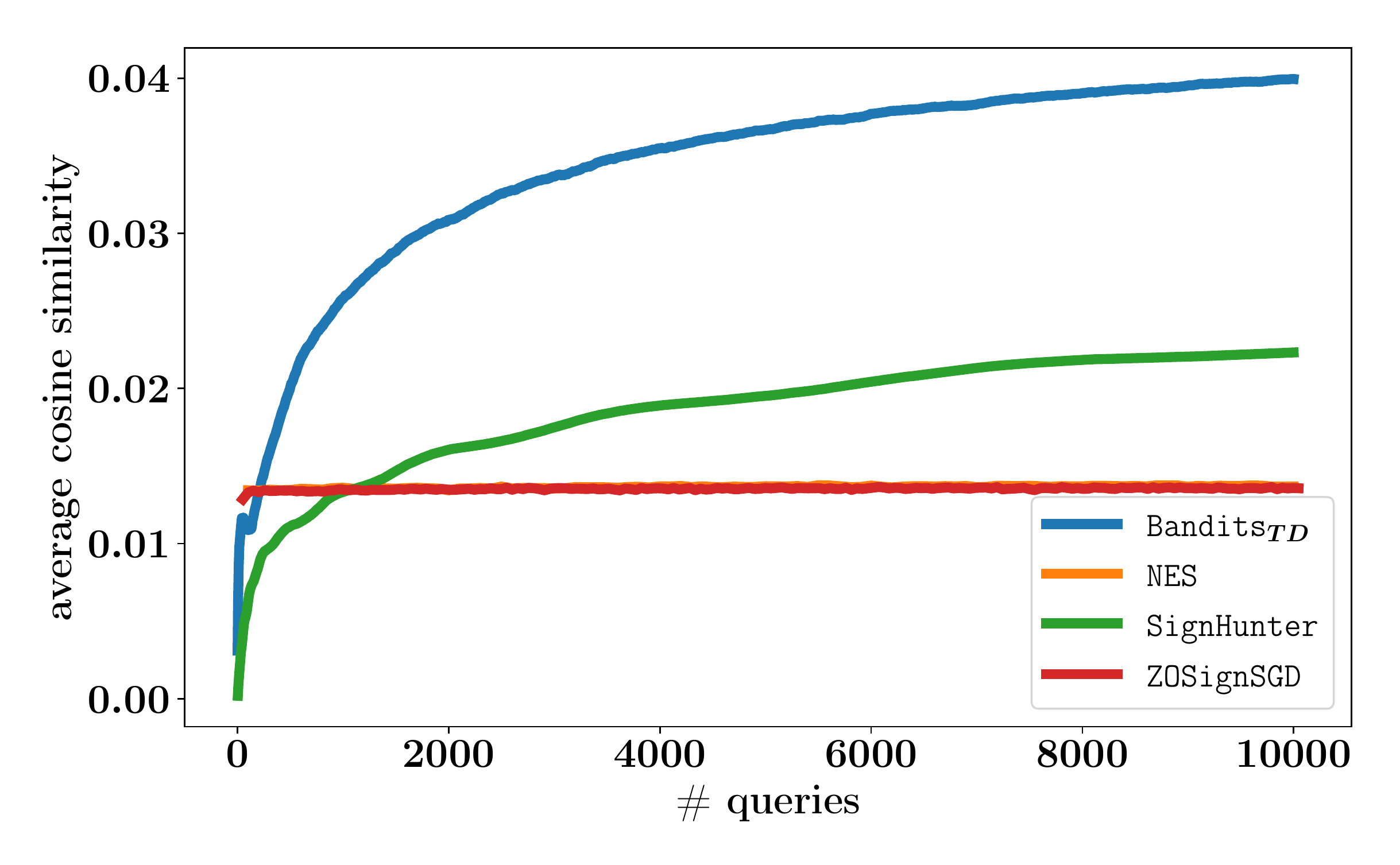} \\
			\includegraphics[width=0.4\textwidth ]{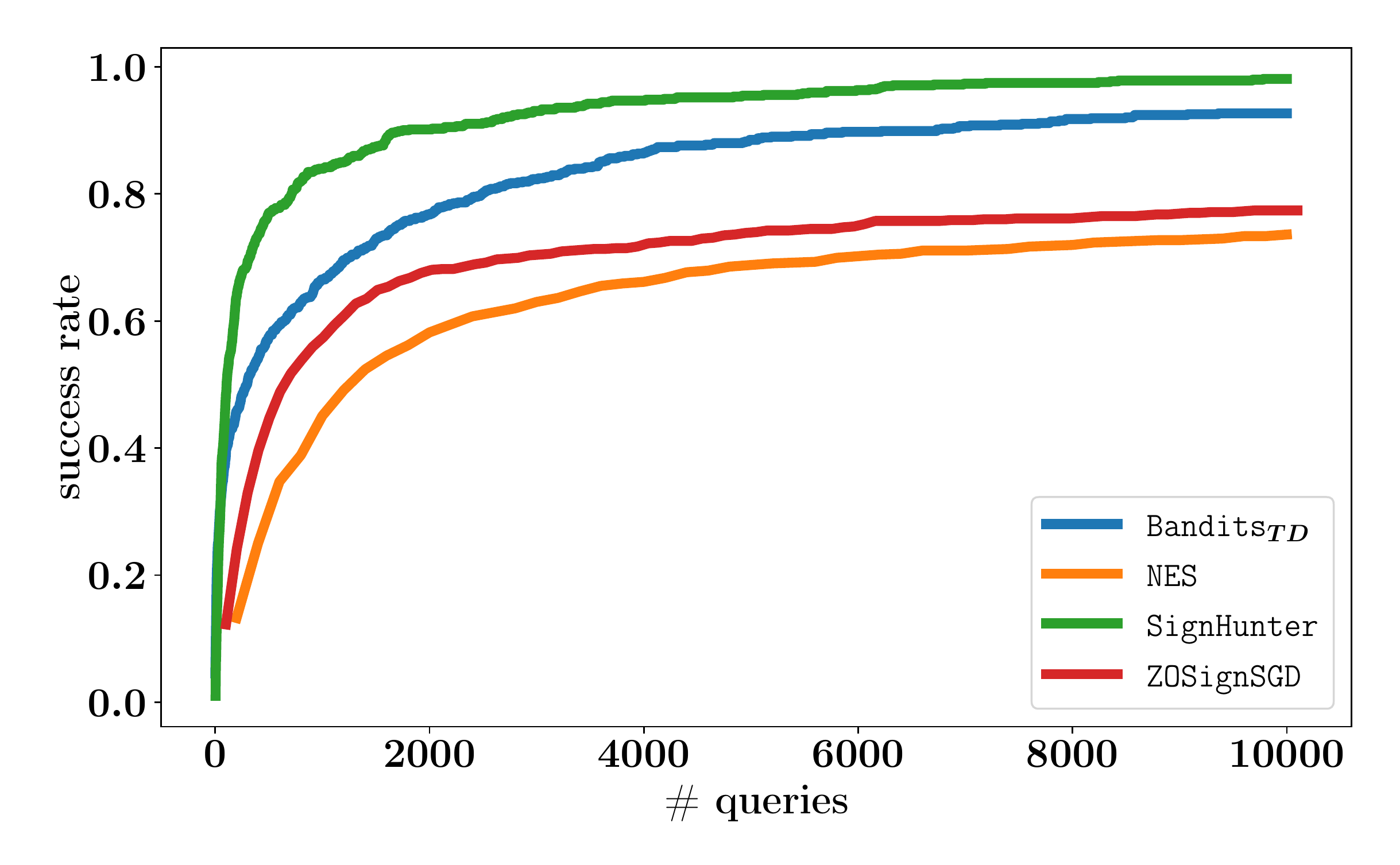} &
			\includegraphics[width=0.4\textwidth ]{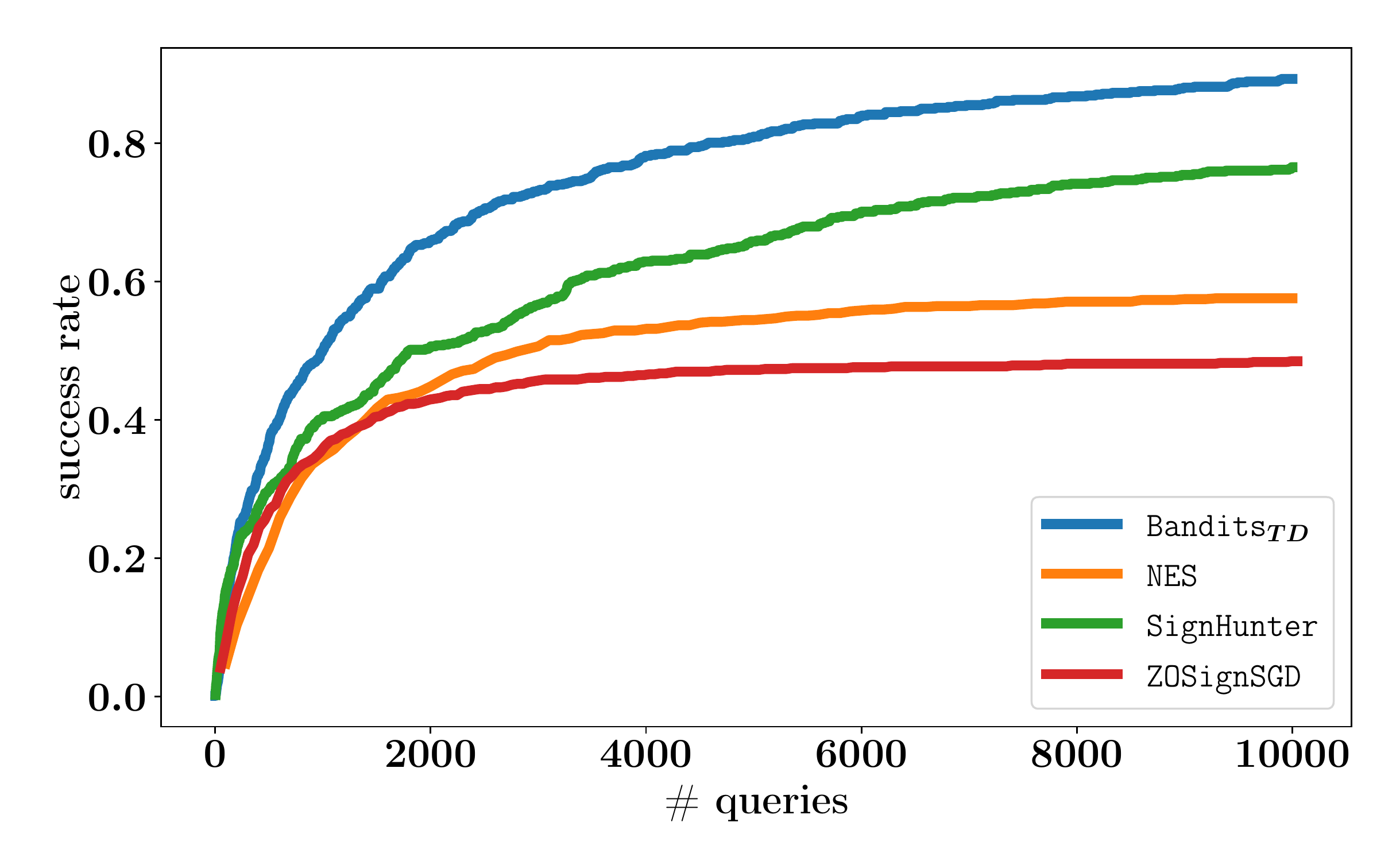} \\
			\includegraphics[width=0.4\textwidth ]{figs/imagenet_sota_tbl_plots/imagenet_inf_qrt_plt.pdf}  &
			\includegraphics[width=0.4\textwidth ]{figs/imagenet_sota_tbl_plots/imagenet_2_qrt_plt.pdf} \\
		\end{tabular}
	}
	\caption{Performance curves of attacks on \imgnt for $\linf$ (first column) and $\ltwo$ (second column) perturbation constraints. Plots of  \emph{Avg. Loss} row reports the loss as a function of the number of queries averaged over all images. The \emph{Avg. Hamming Similarity} row shows the Hamming similarity of the sign of the attack's estimated gradient $\hat{\vg}$ with true gradient's sign $\vq^*$, computed as $1 - ||\sgn(\hat{\vg}) - \vq^*||_H/ n$ and averaged over all images. Likewise, plots of the \emph{Avg. Cosine Similarity} row show the normalized dot product of $\hat{\vg}$ and $\vg^*$ averaged over all images. The \emph{Success Rate} row reports the attacks' cumulative distribution functions for the number of queries required to carry out a successful attack up to the query limit of $10,000$ queries. The \emph{Avg. \# Queries} row reports the average number of queries used per successful image for each attack when reaching a specified success rate: the more effective the attack, the closer its curve is to the bottom right of the plot.
	}
\label{fig:imgnet-res}
\end{figure*}

\cleardoublepage
 
\section*{Appendix F. Public Black-Box Challenge Results}

\begin{figure*}[h!]
	\centering
	\resizebox{!}{0.32\textheight}{
		\begin{tabular}{ccc}
			Madry's Lab (\mnist) & Madry's Lab (\cifar) & Ensemble Adversarial Training (\imgnt) \\
			\includegraphics[width=0.4\textwidth]{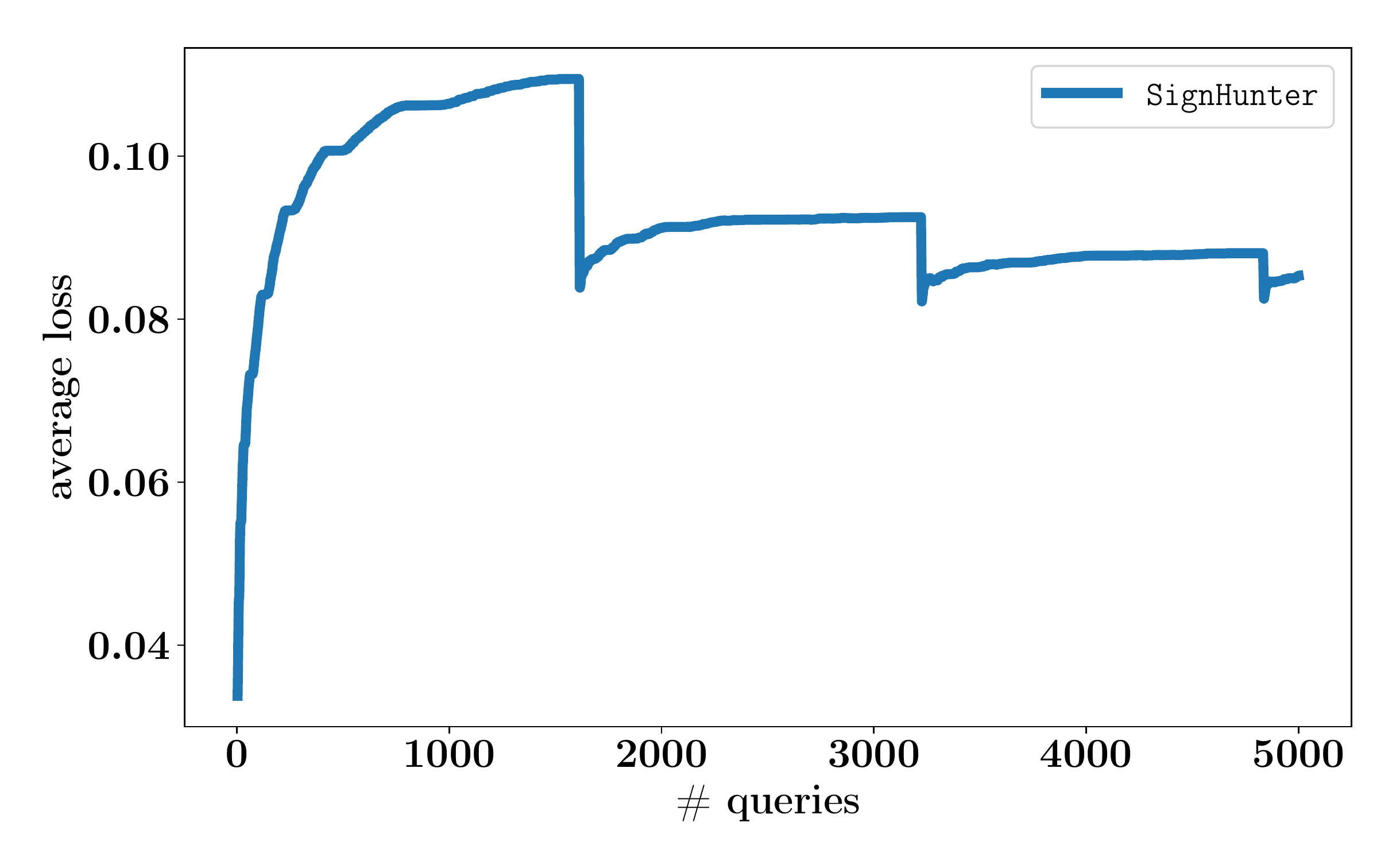} & \includegraphics[width=0.4\textwidth]{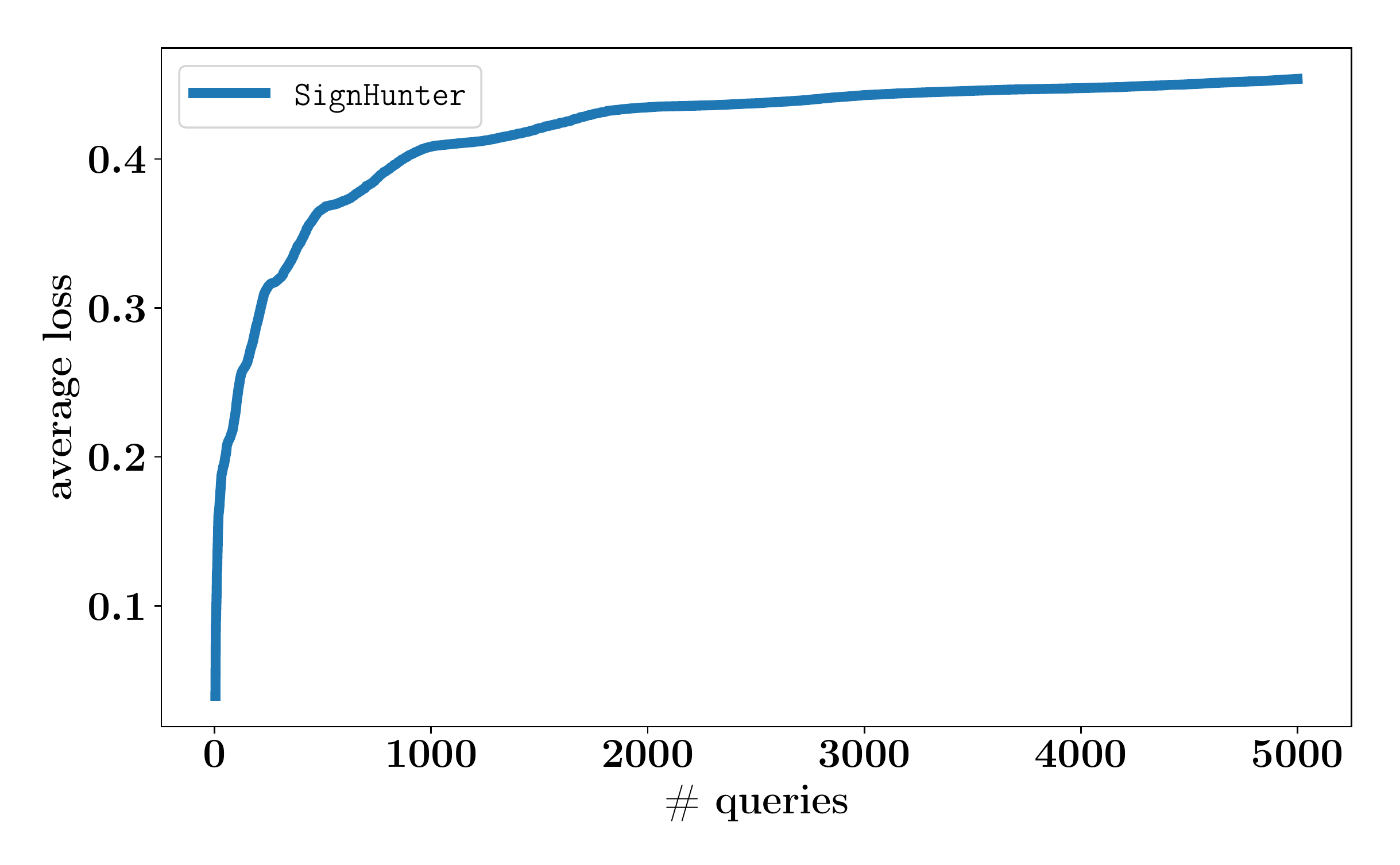} & \includegraphics[width=0.4\textwidth]{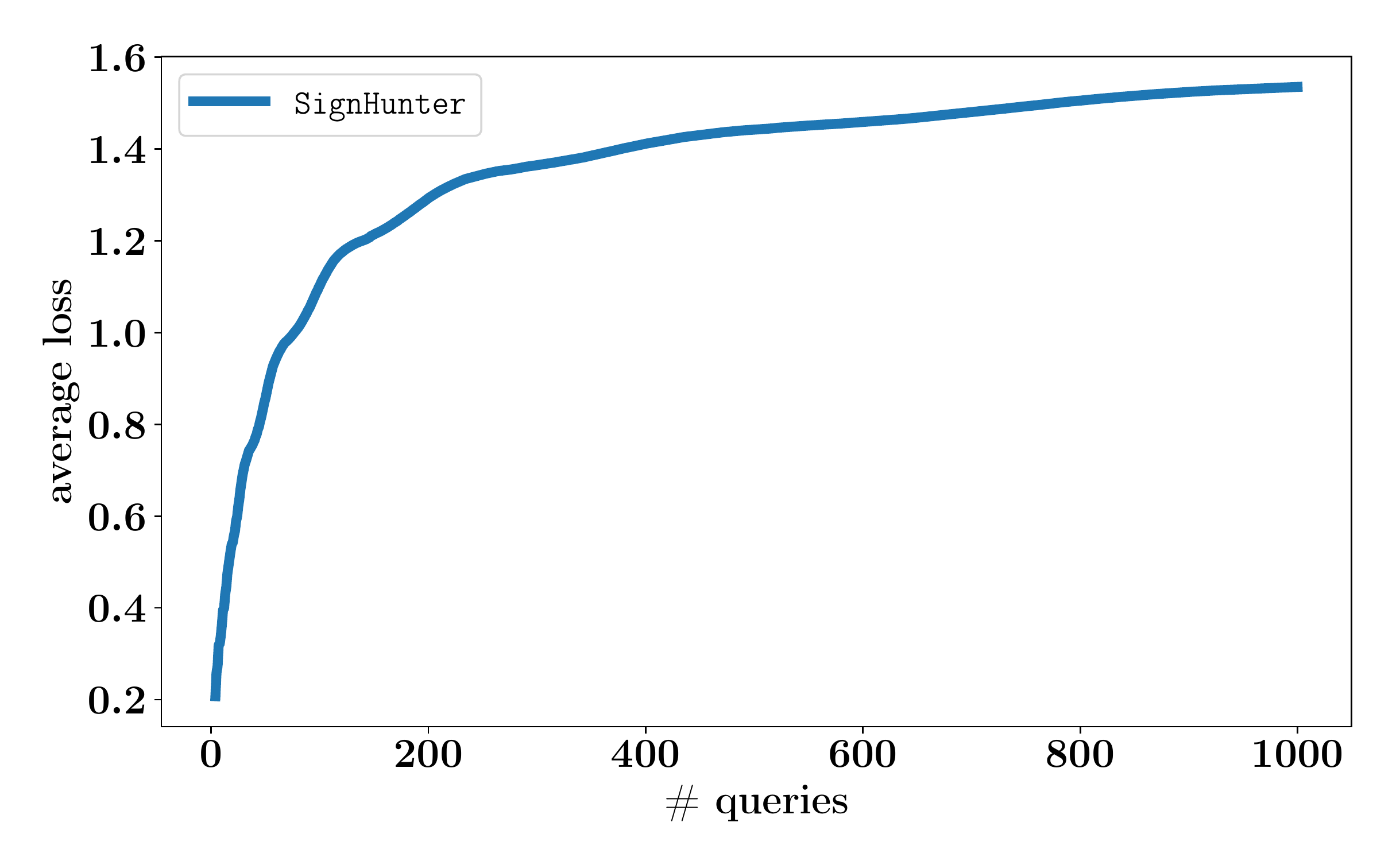} \\ \includegraphics[width=0.4\textwidth]{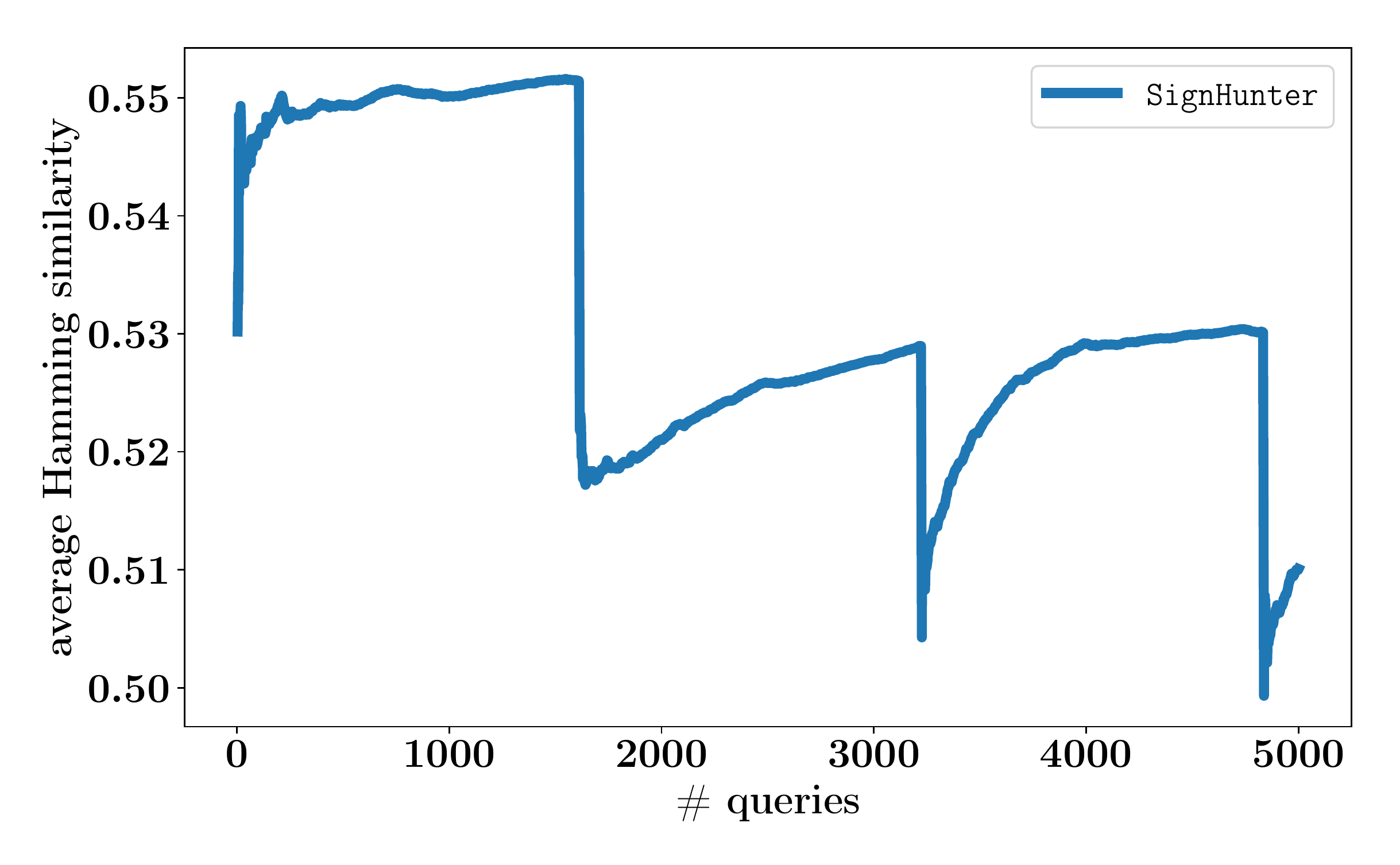} &
			\includegraphics[width=0.4\textwidth]{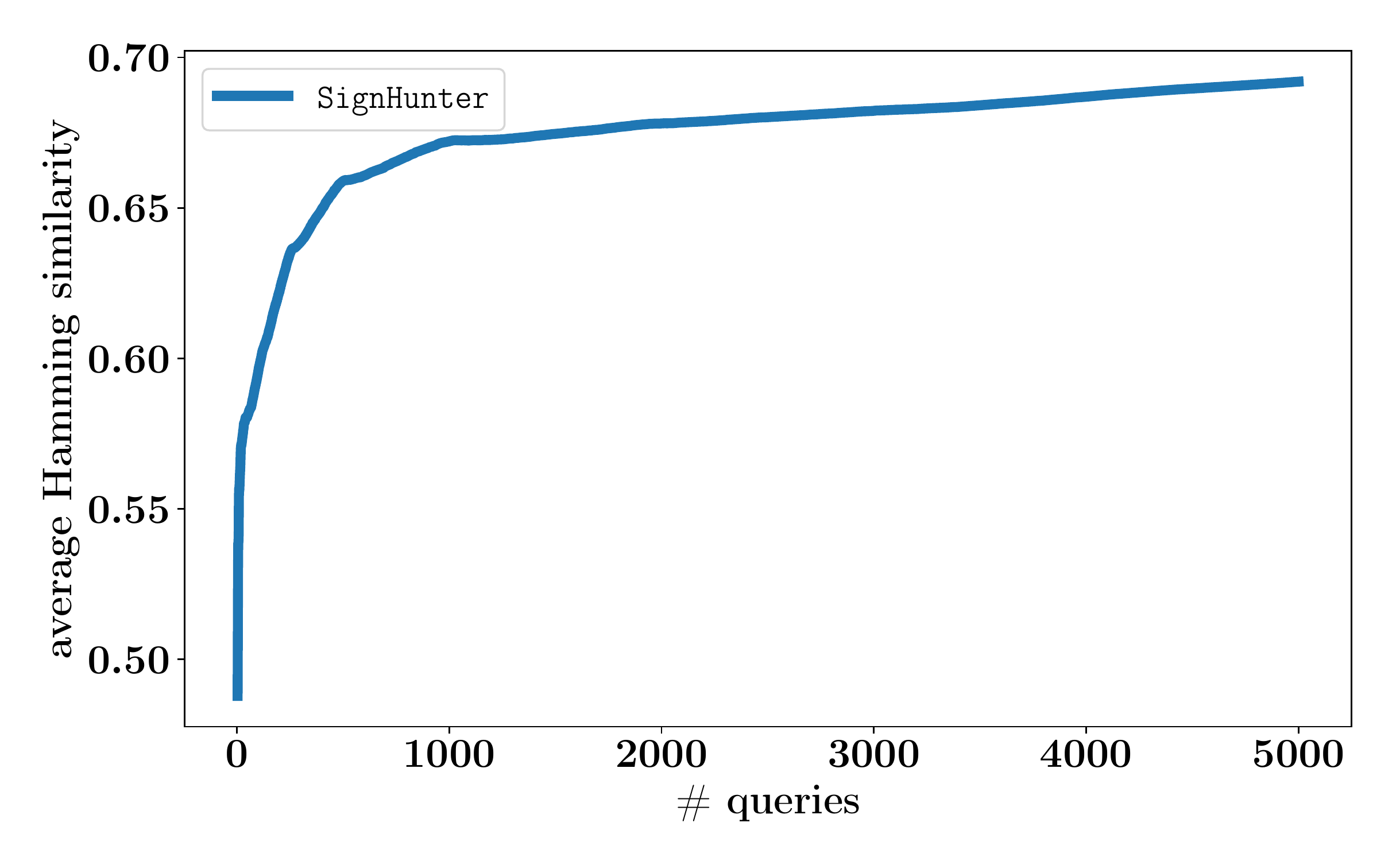} &
			\includegraphics[width=0.4\textwidth]{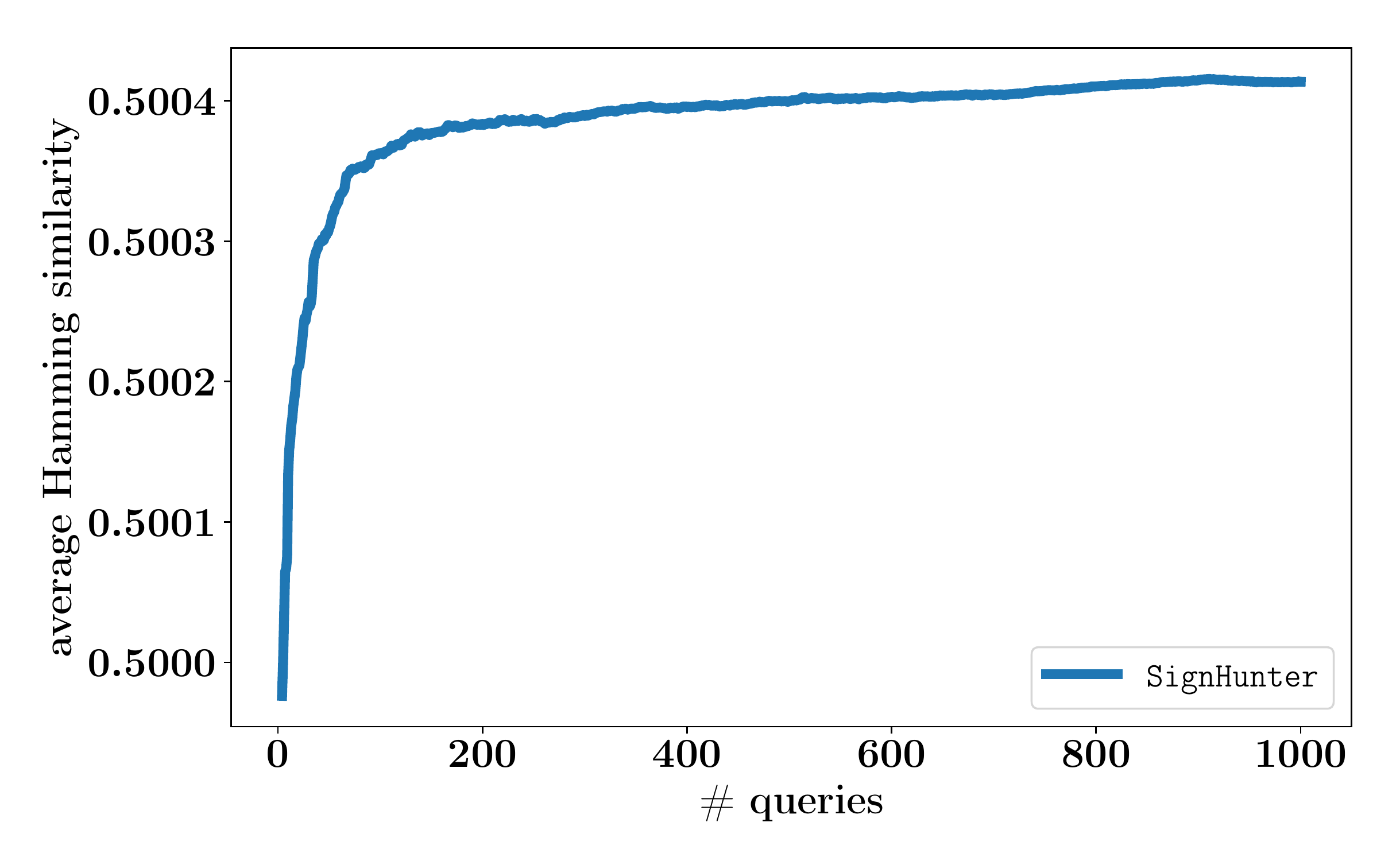}\\
			\includegraphics[width=0.4\textwidth ]{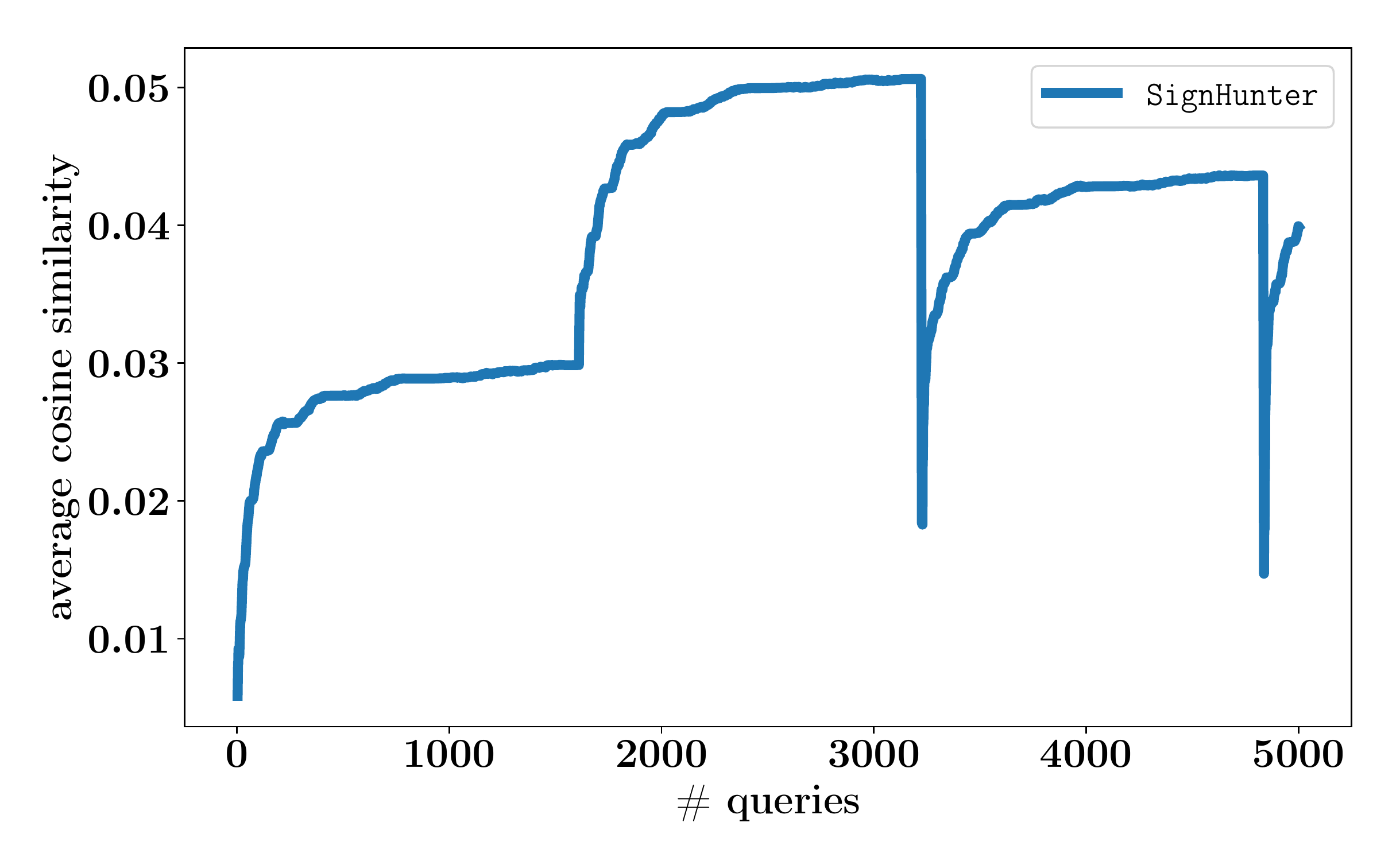} & \includegraphics[width=0.4\textwidth ]{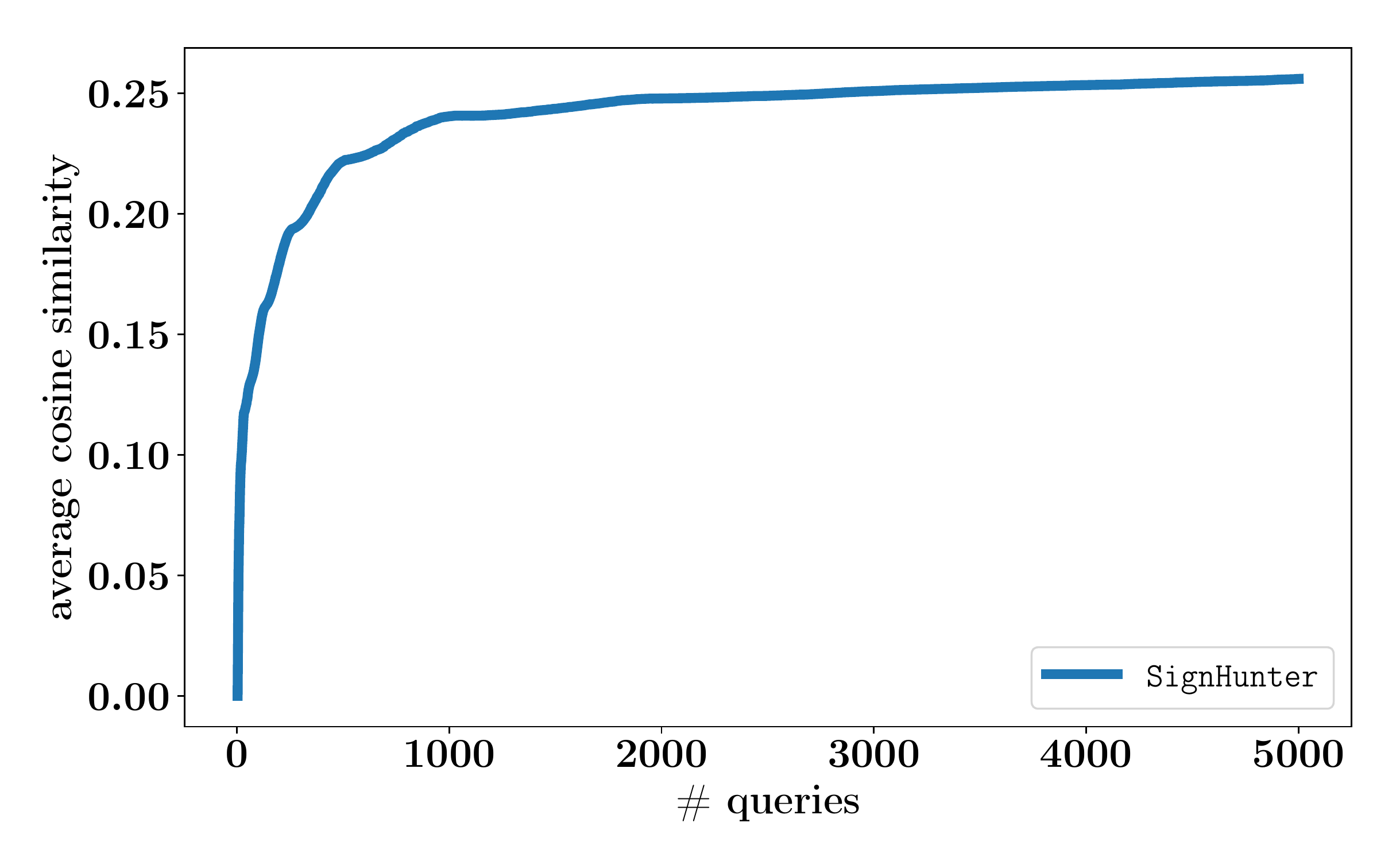} & \includegraphics[width=0.4\textwidth ]{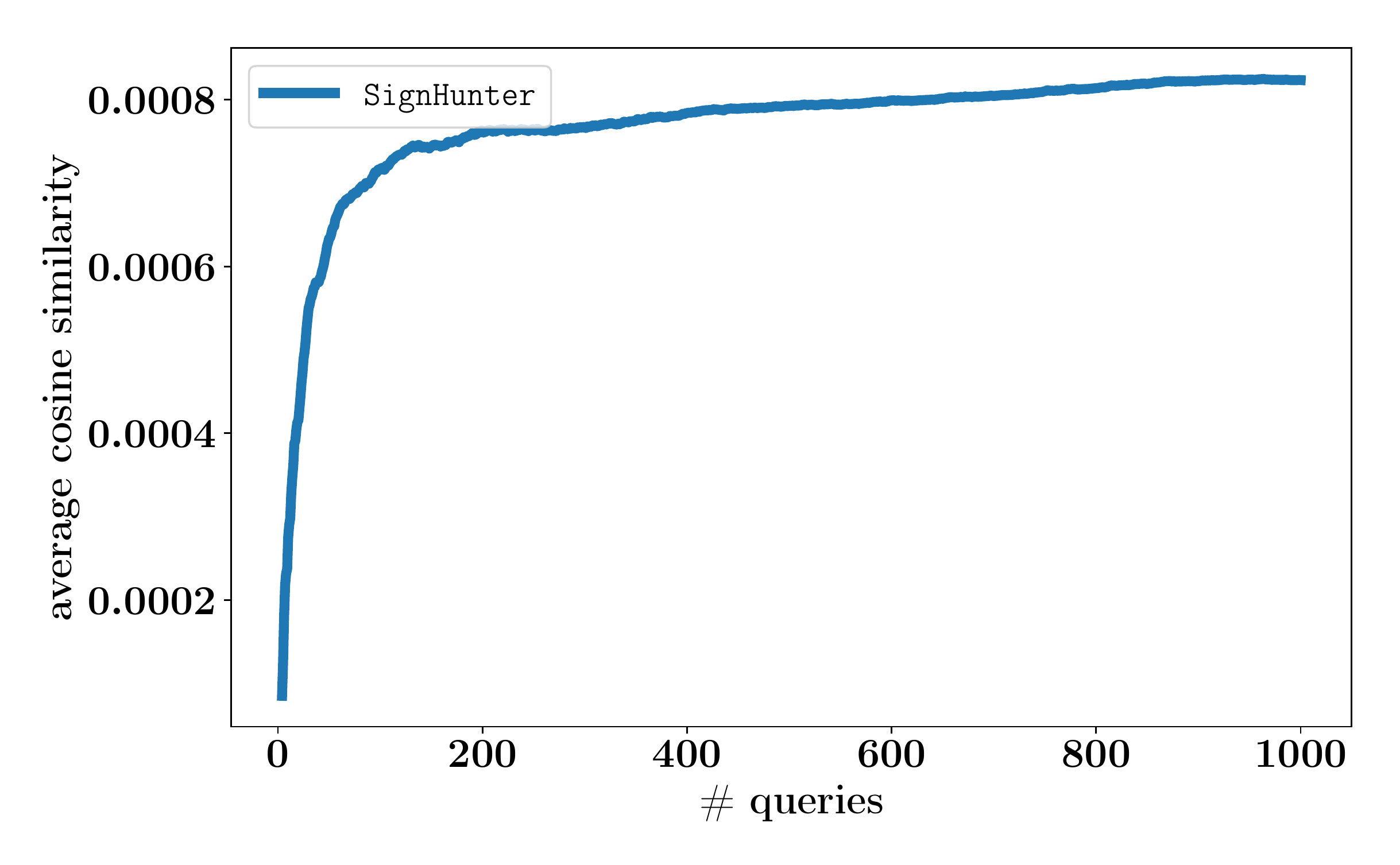} \\
			\includegraphics[width=0.4\textwidth ]{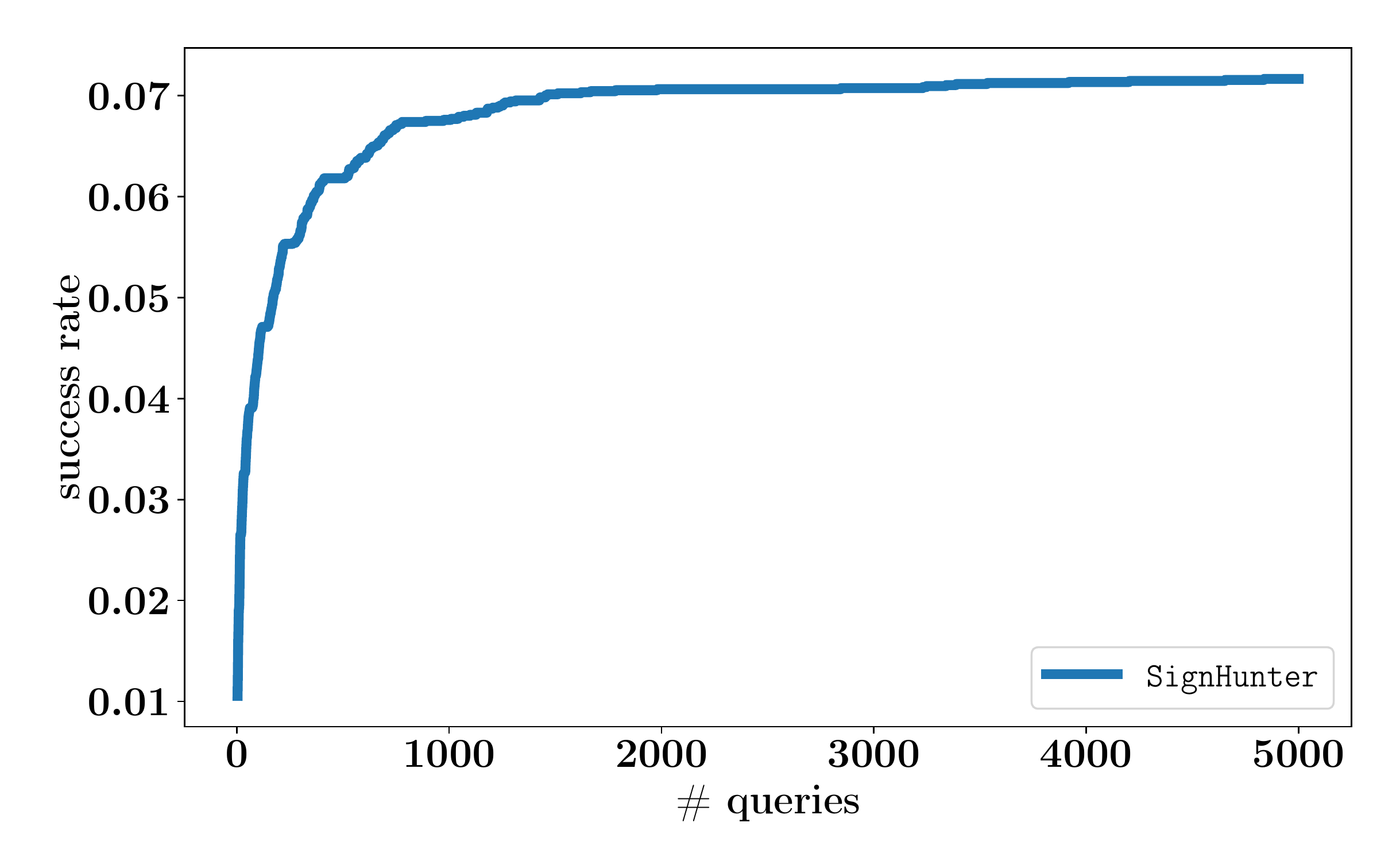} &
			\includegraphics[width=0.4\textwidth ]{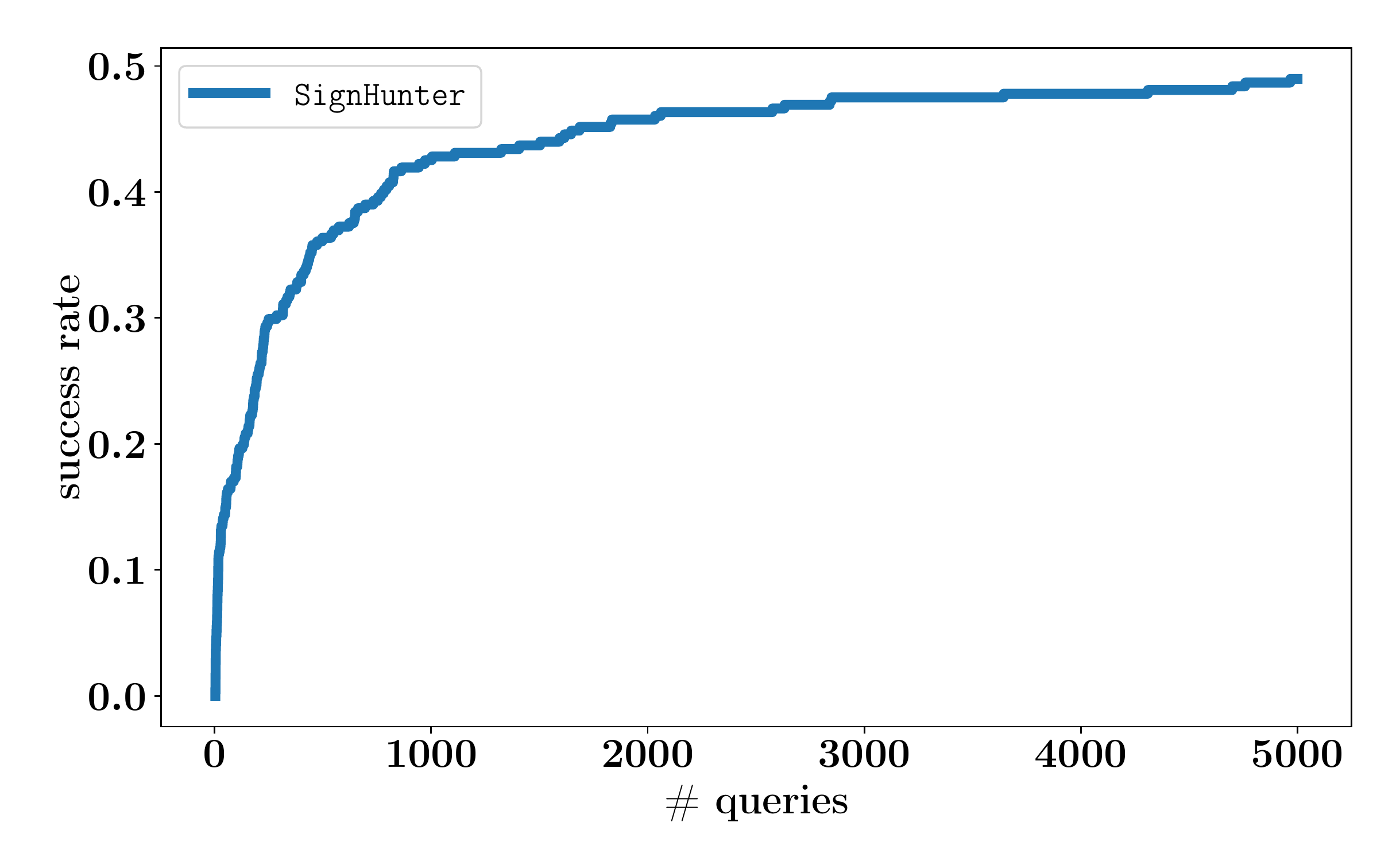} &
			\includegraphics[width=0.4\textwidth ]{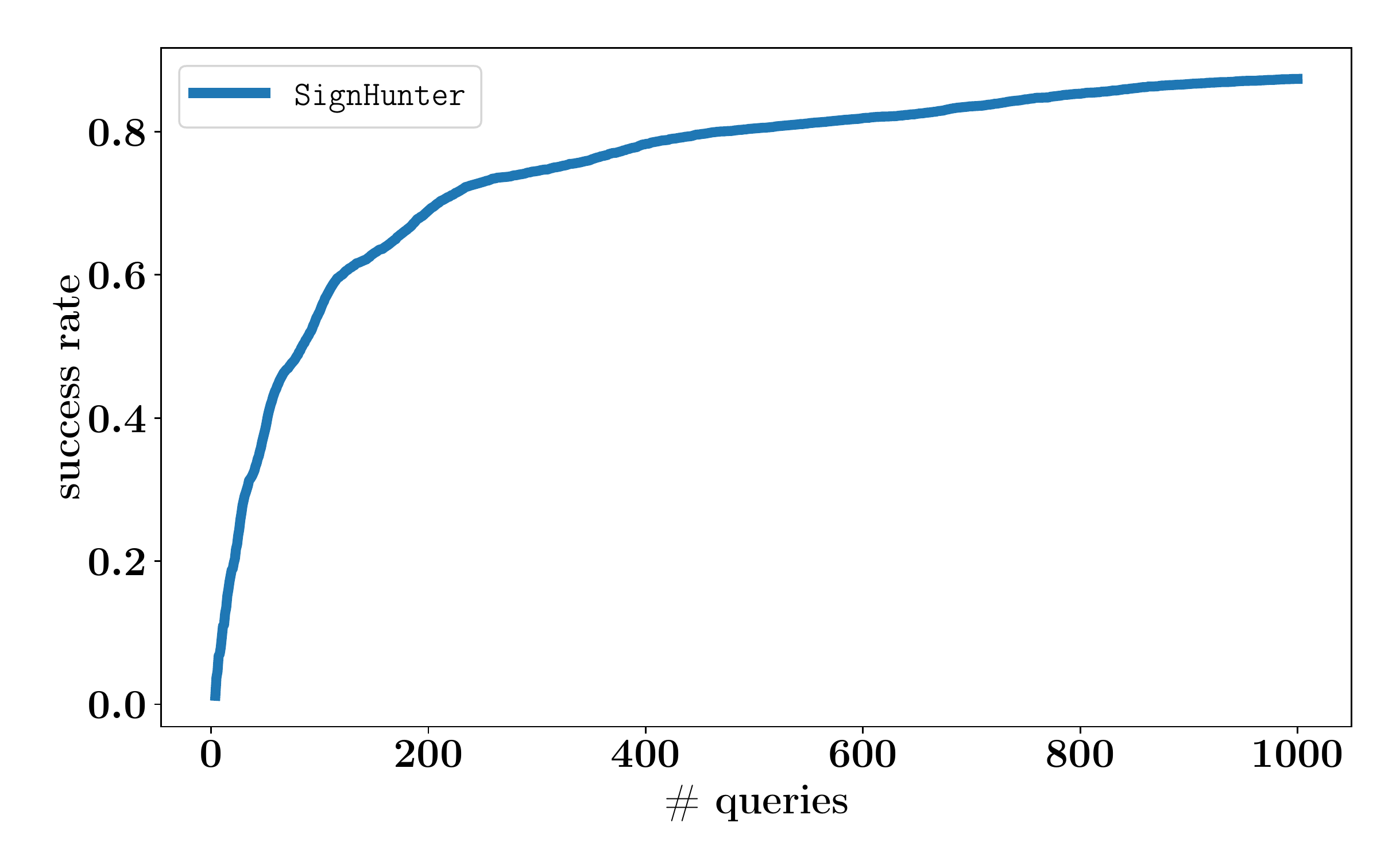} \\
			\includegraphics[width=0.4\textwidth ]{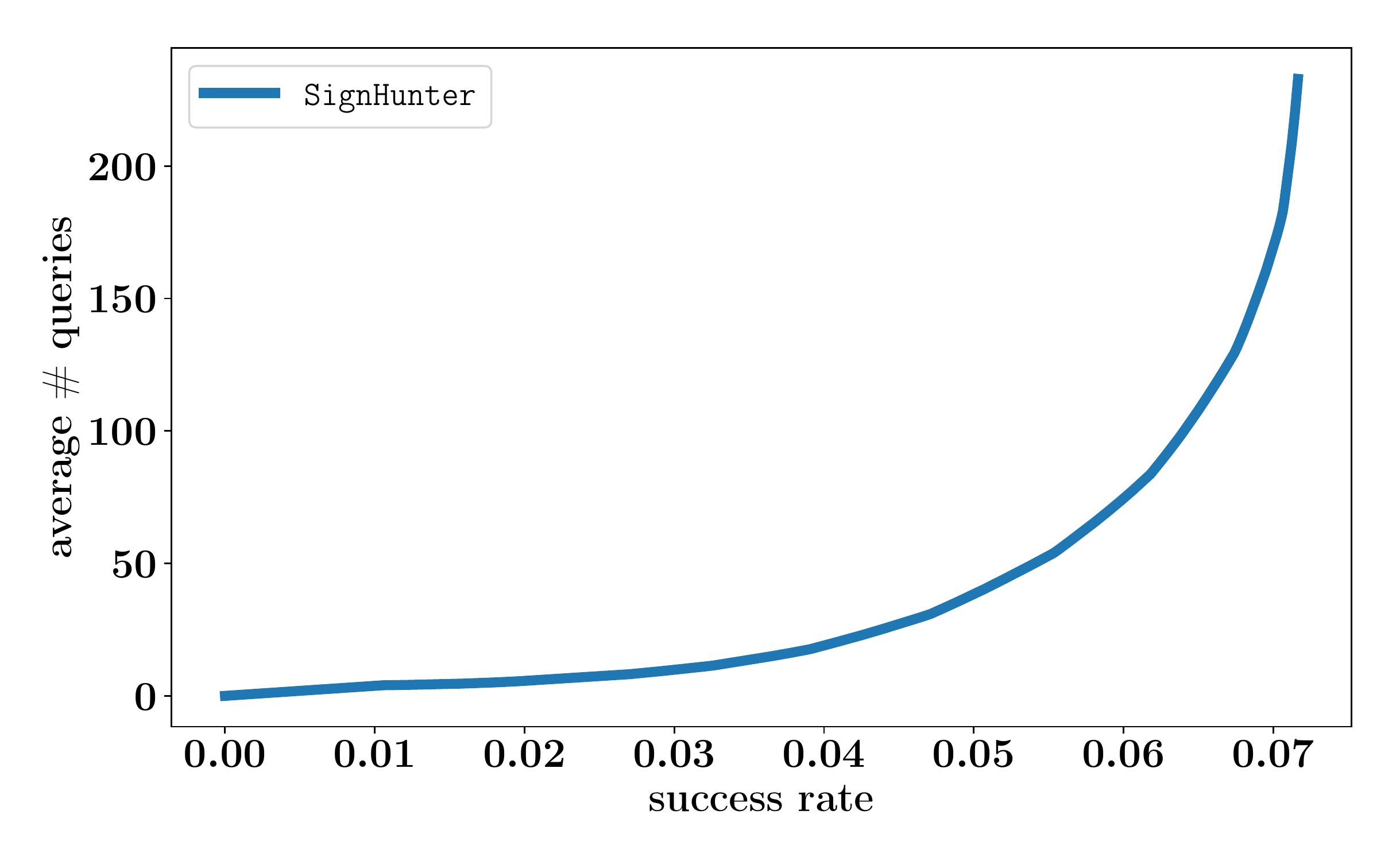}  &
			\includegraphics[width=0.4\textwidth ]{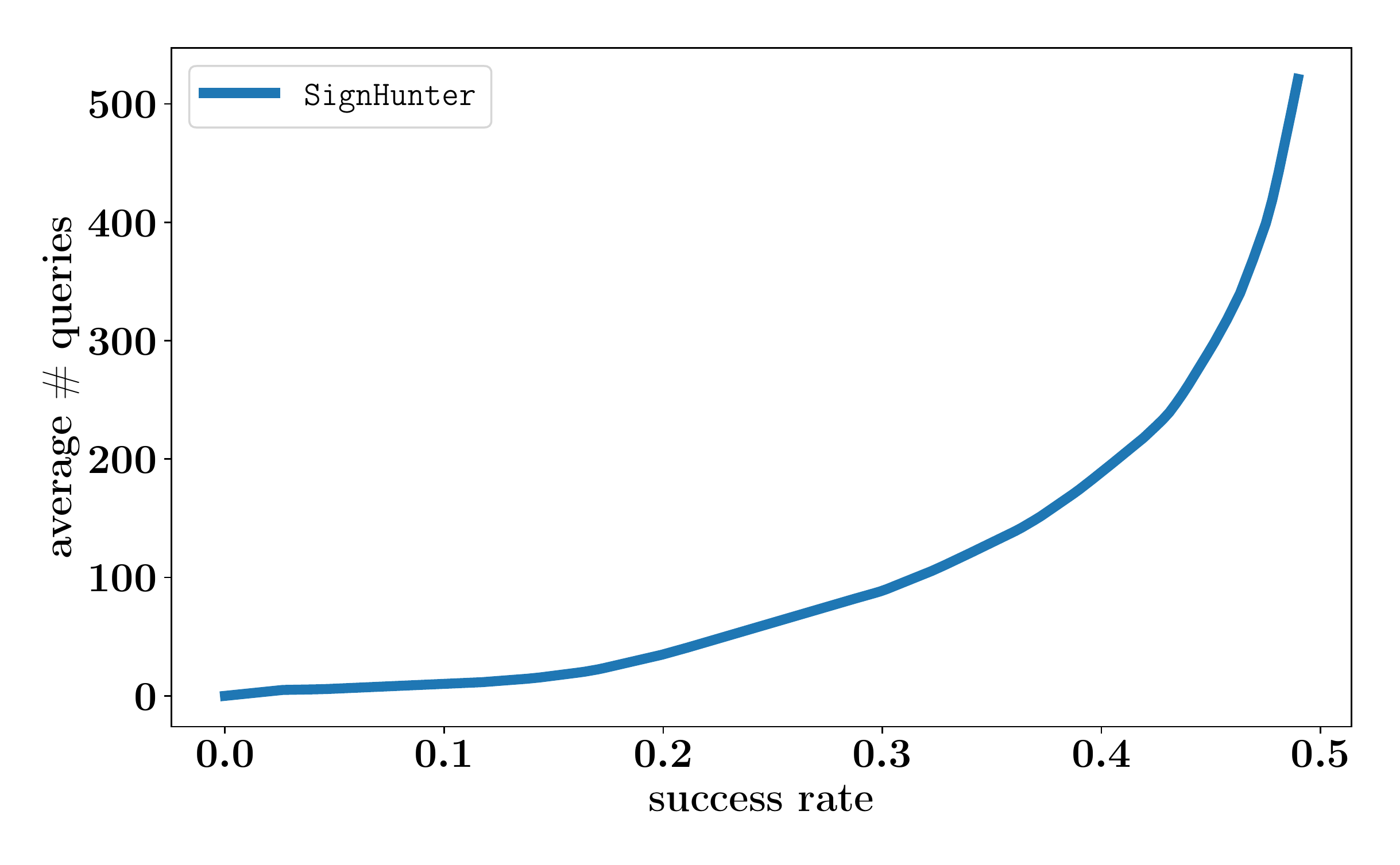} &
			\includegraphics[width=0.4\textwidth ]{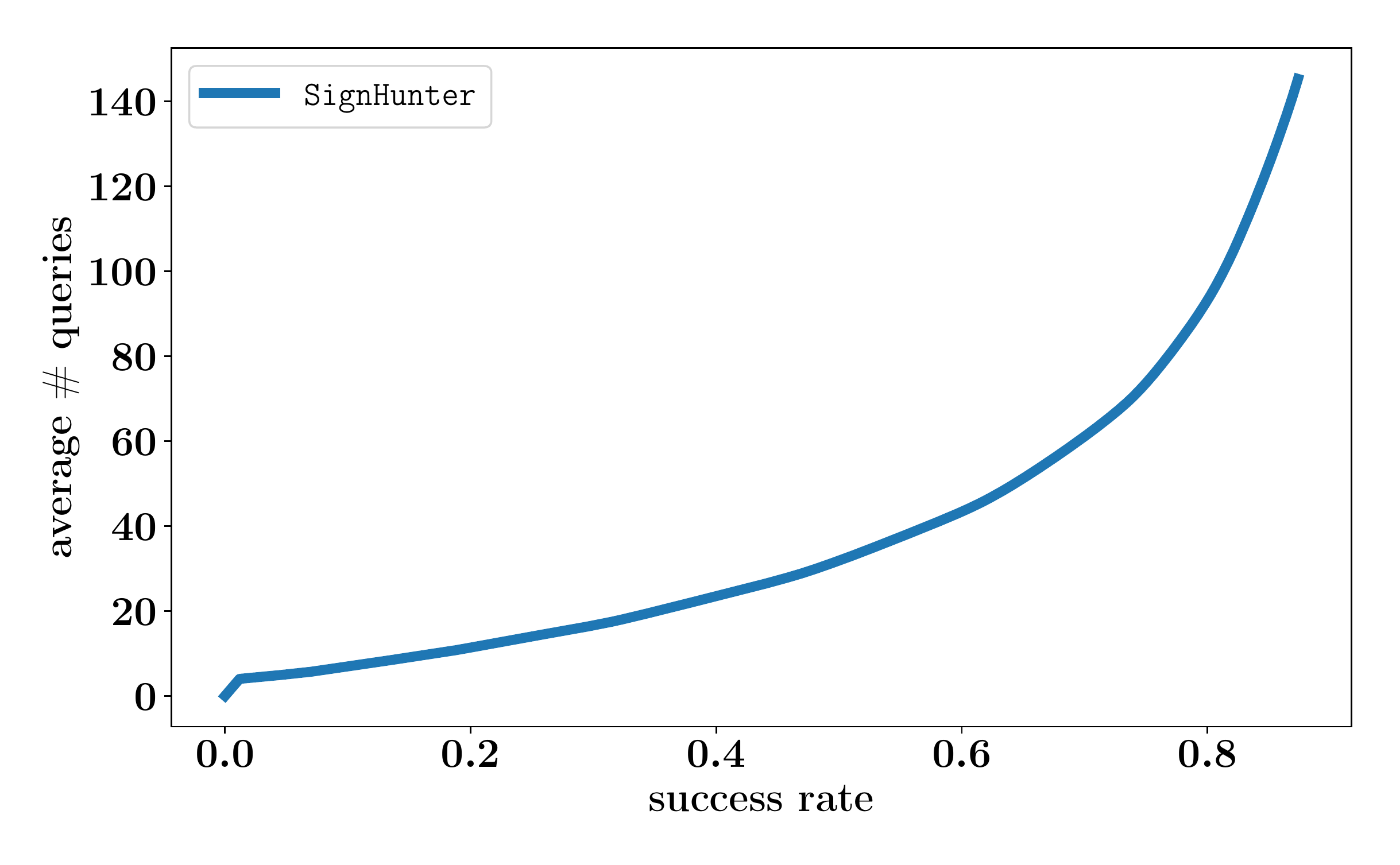} 
			\\
		\end{tabular}
	}
	\caption{Performance curves of attacks on the public black-box challenges for \mnist (first column), \cifar (second column) and \imgnt (third column). Plots of  \emph{Avg. Loss} row reports the loss as a function of the number of queries averaged over all images. The \emph{Avg. Hamming Similarity} row shows the Hamming similarity of the sign of the attack's estimated gradient $\hat{\vg}$ with true gradient's sign $\vq^*$, computed as $1 - ||\sgn(\hat{\vg}) - \vq^*||_H/ n$ and averaged over all images. Likewise, plots of the \emph{Avg. Cosine Similarity} row show the normalized dot product of $\hat{\vg}$ and $\vg^*$ averaged over all images. The \emph{Success Rate} row reports the attacks' cumulative distribution functions for the number of queries required to carry out a successful attack up to the query limit of $5,000$ queries. The \emph{Avg. \# Queries} row reports the average number of queries used per successful image for each attack when reaching a specified success rate: the more effective the attack, the closer its curve is to the bottom right of the plot.
	}
	\label{fig:challenges}
\end{figure*}
\clearpage
\onecolumn
\section*{Appendix G.  Estimating Hamming Oracle }

This section illustrates our experiment on the distribution of the magnitudes of gradient coordinates as summarized in Figure~\ref{fig:hist-mag-grad}. \emph{How to read the plots:} Consider the first histogram in Plot (a) from below; it corresponds to the $1000^{th}$ image from the sampled \mnist evaluation set, plotting the histogram of the values $\{|\partial L(\vx, y)/\partial x_i |\}_{1\leq i \leq n}$, where the \mnist dataset has dimensionality $n=784$. These values are in the range $[0, 0.002]$. Overall, the values are fairly concentrated---with exceptions, in Plot (e) for instance, the magnitudes of the $\sim400^{th}$ image's gradient coordinates are spread from $0$ to $\sim 0.055$. Thus, a Monte Carlo estimate of the mean of $\{|\partial L(\vx, y)/\partial x_i |\}_{1\leq i \leq n}$ would be an appropriate approximation. We release these figures in the form of \textsf{TensorBoard} logs.	

\begin{figure*}[h]
	\centering
	\resizebox{\textwidth}{!}{
		\begin{tabular}{c||c|c|c|}
			& \mnist & \cifar & \imgnt \\ 
			\toprule
			\toprule
			\put(-3,10){
				\rotatebox{90}{ 
					Original Images $\vx$
			}}
			&\begin{overpic}[width=0.31\textwidth,unit=1mm, 
				,trim=5 26 0 20,clip]{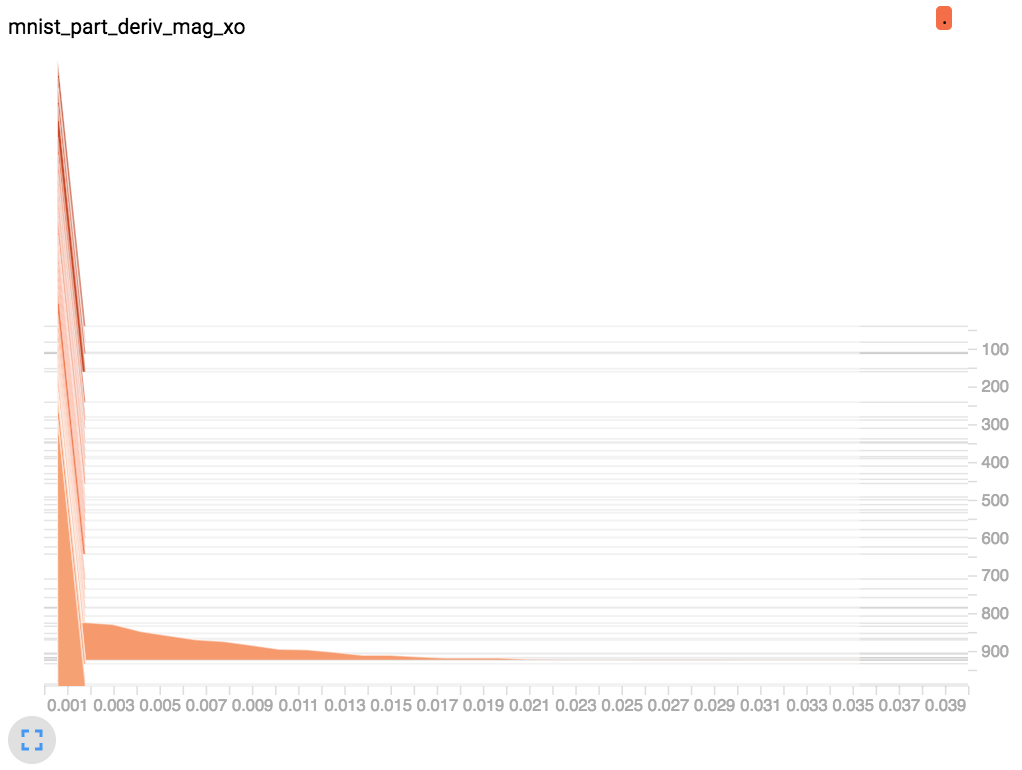}
				\put(5,-2){ \bf \tiny gradient coordinate magnitude}
				\put(-2,7){
					\rotatebox{90}{\bf \tiny  \#gradient coordinates}}
				\put(48,4){
					\rotatebox{90}{\bf \fontsize{4}{3}\selectfont test image index}}
			\end{overpic}
			\hspace{1em}
			&
			\begin{overpic}[width=0.31\textwidth,unit=1mm, 
				,trim=5 26 0 20,clip]{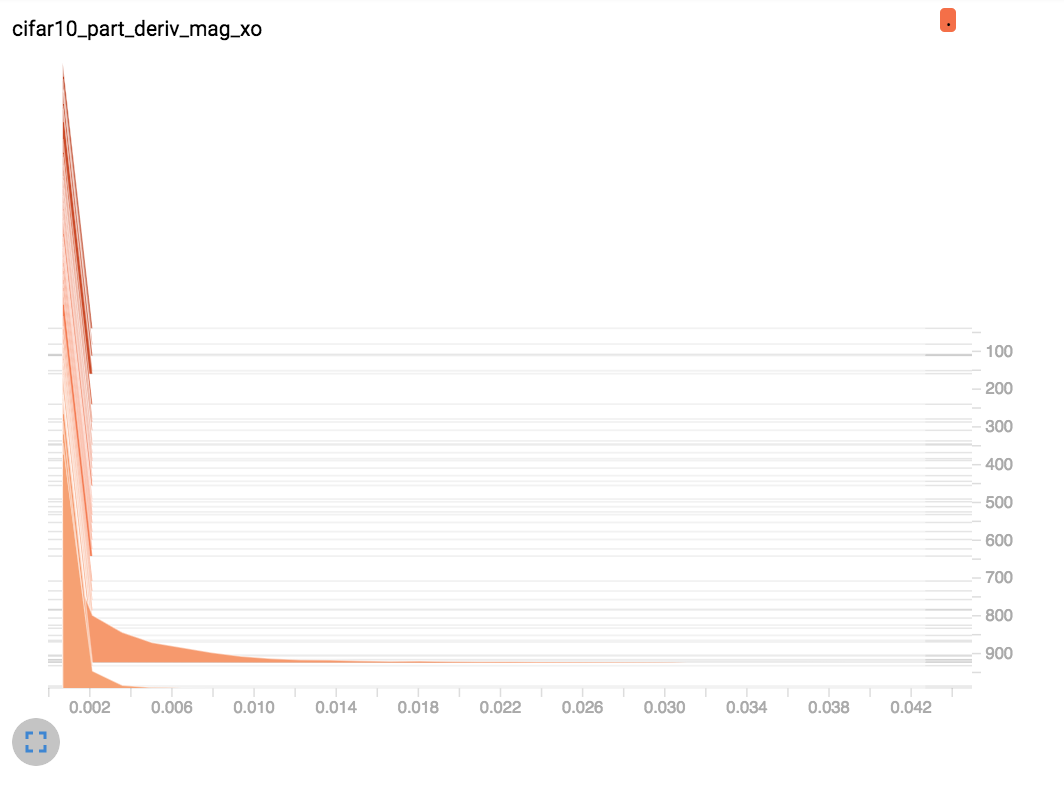}
				\put(5,-2){ \bf \tiny gradient coordinate magnitude}
				\put(-2,7){
					\rotatebox{90}{\bf \tiny  \#gradient coordinates}}
				\put(48,4){
					\rotatebox{90}{\bf \fontsize{4}{3}\selectfont test image index}}
			\end{overpic}
			\hspace{1em}
			&
			\begin{overpic}[width=0.31\textwidth,unit=1mm, 
				,trim=5 26 0 20,clip]{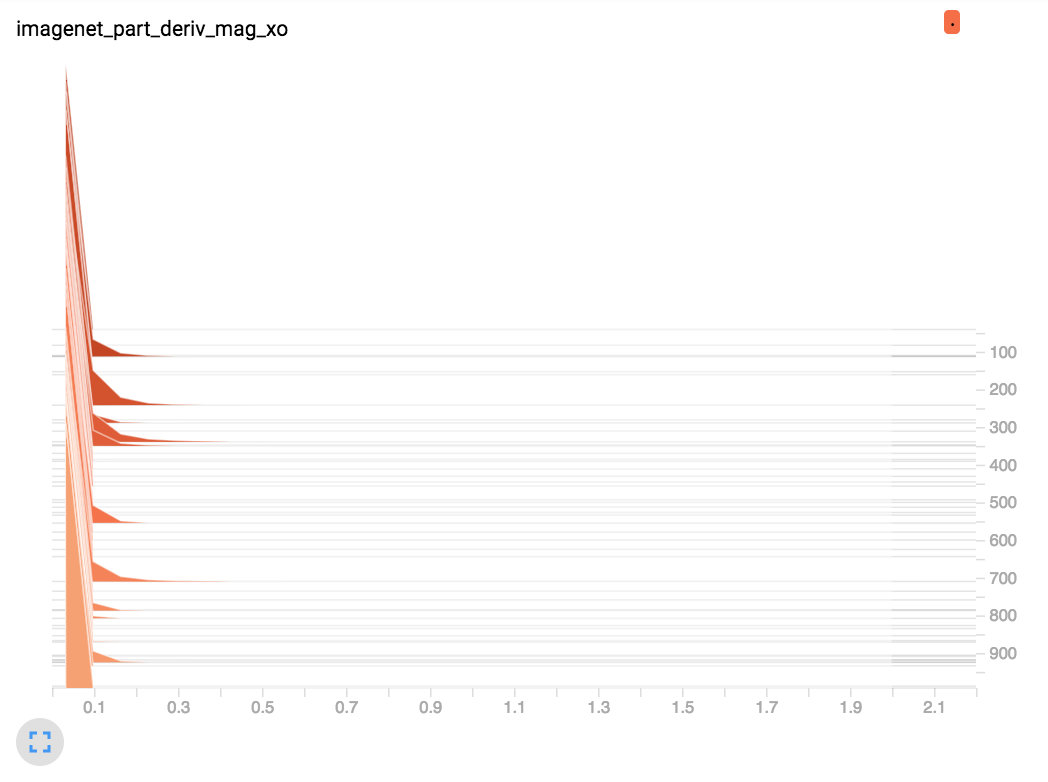}
				\put(5,-2){ \bf \tiny gradient coordinate magnitude}
				\put(-2,7){
					\rotatebox{90}{\bf \tiny  \#gradient coordinates}}
				\put(48,4){
					\rotatebox{90}{\bf \fontsize{4}{3}\selectfont test image index}}
			\end{overpic}
		   \hspace{1em}
			\\
			&\tiny{\textbf{(a)}} & \tiny{\textbf{(b)}} & \tiny{\textbf{(c)}}
			\\ \midrule
			\put(-3,-10){
				\rotatebox{90}{\small Perturbed Images $\in B_{\infty}(\vx, \epsilon)$
			}}
			&\begin{overpic}[width=0.31\textwidth,unit=1mm, 
				,trim=5 26 0 20,clip]{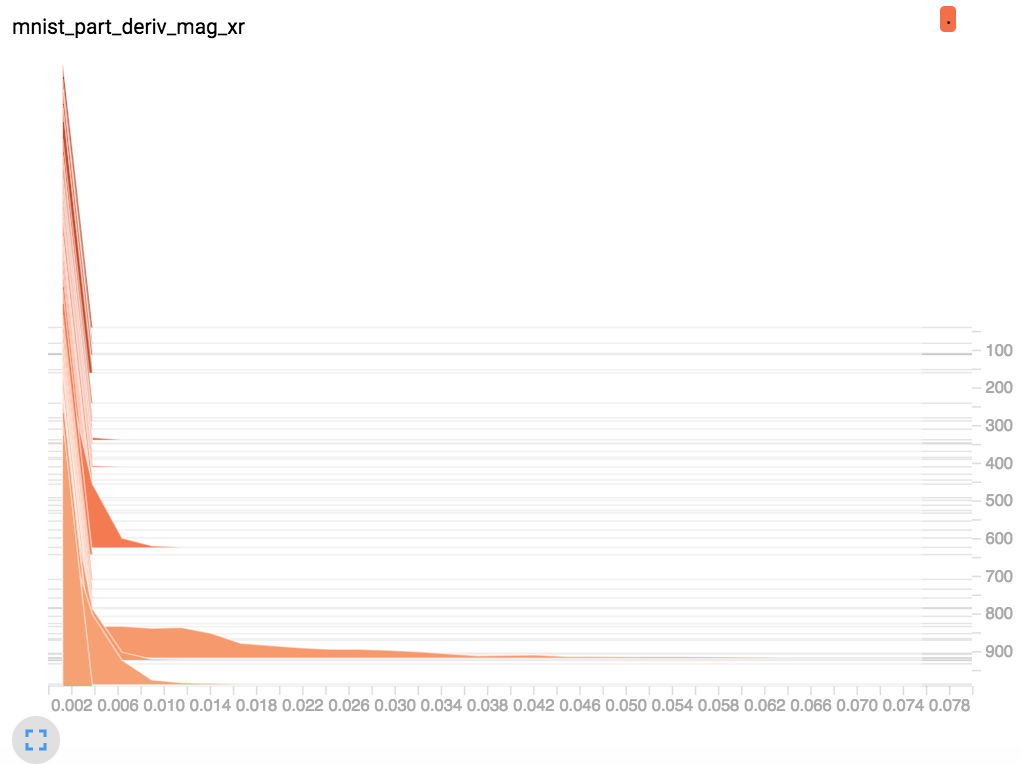}
				\put(5,-2){ \bf \tiny gradient coordinate magnitude}
				\put(-2,7){
					\rotatebox{90}{\bf \tiny  \#gradient coordinates}}
				\put(48,4){
					\rotatebox{90}{\bf \fontsize{4}{3}\selectfont test image index}}
			\end{overpic}
			\hspace{1em}
			&
			\begin{overpic}[width=0.31\textwidth,unit=1mm, 
				,trim=5 26 0 20,clip]{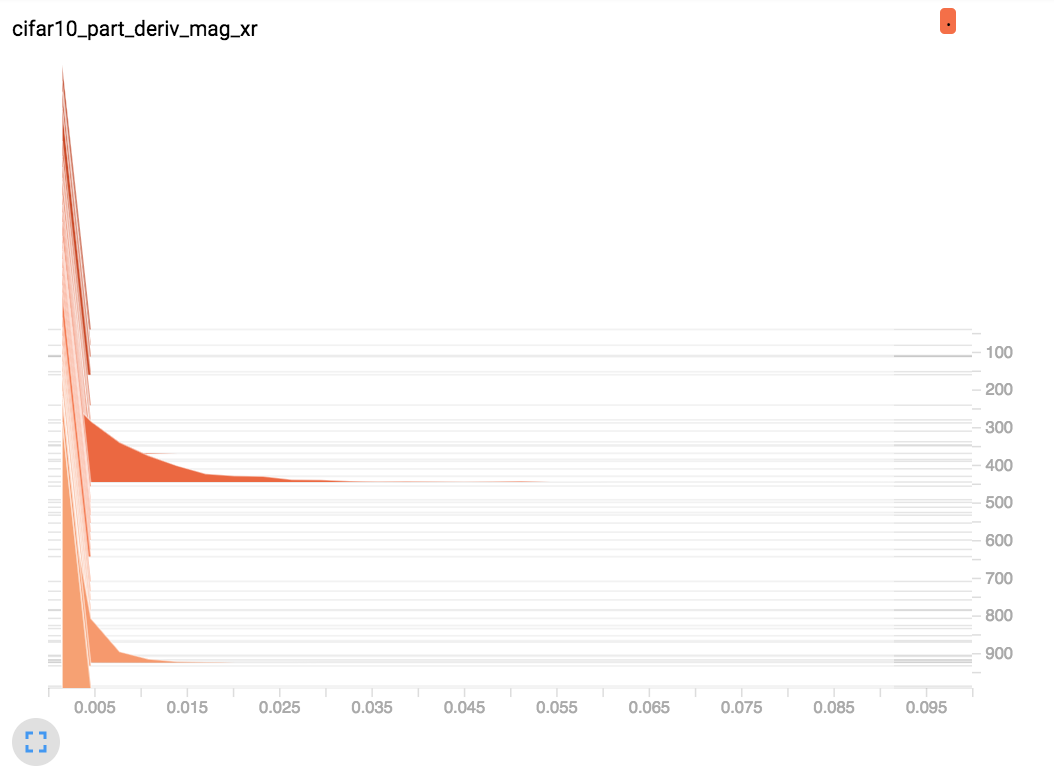}
				\put(5,-2){ \bf \tiny gradient coordinate magnitude}
				\put(-2,7){
					\rotatebox{90}{\bf \tiny  \#gradient coordinates}}
				\put(48,4){
					\rotatebox{90}{\bf \fontsize{4}{3}\selectfont test image index}}
			\end{overpic}
			\hspace{1em}
			&
			\begin{overpic}[width=0.31\textwidth,unit=1mm, 
				,trim=5 26 0 20,clip]{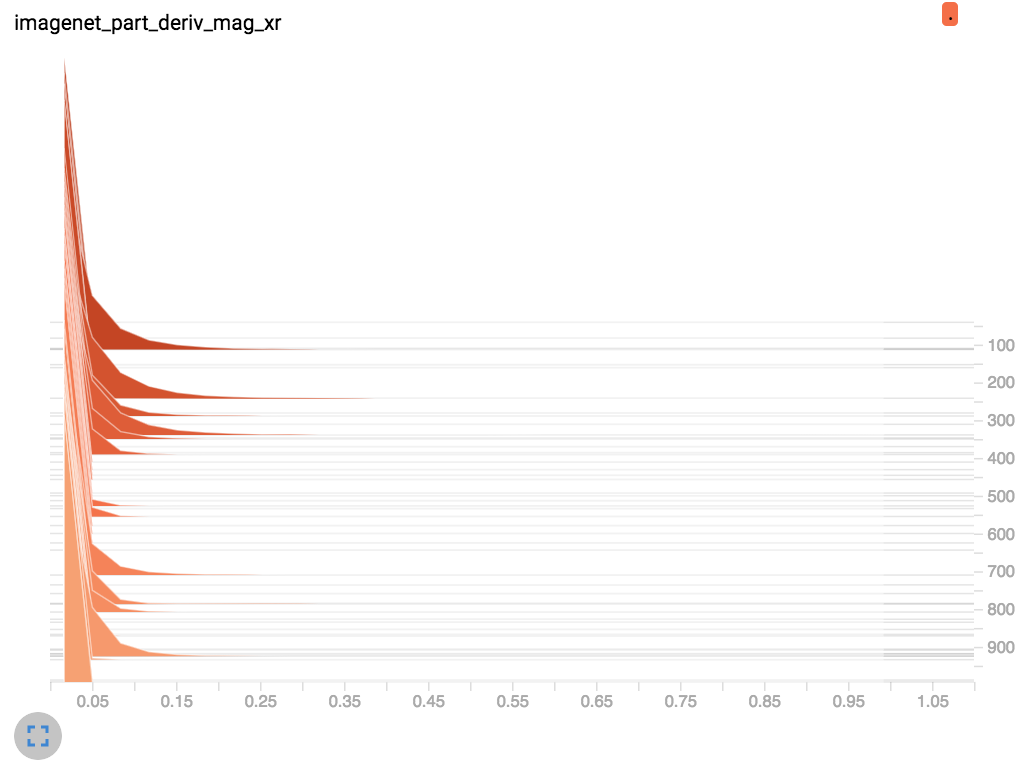}
				\put(5,-2){ \bf \tiny gradient coordinate magnitude}
				\put(-2,7){
					\rotatebox{90}{\bf \tiny  \#gradient coordinates}}
				\put(48,4){
					\rotatebox{90}{\bf \fontsize{4}{3}\selectfont test image index}}
			\end{overpic}
		   \hspace{1em}
			\\
			&\tiny{\textbf{(d)}} & \tiny{\textbf{(e)}} & \tiny{\textbf{(f)}} \\
			\bottomrule
		\end{tabular}
	}
	\caption{\emph{Magnitudes of gradient coordinates are concentrated:} Plots (a), (b), and (c) show histograms of the magnitudes of gradient coordinates of the loss function $L(\vx, y)$ with respect to the input point (image) $\vx$ for \mnist, \cifar, and \imgnt neural net models over $1000$ images from the corresponding evaluation set, respectively. Plots (d), (e), (f) show the same but at input points (images) sampled randomly within $B_\infty(\vx, \epsilon)$: the $\linf$-ball of radius $\epsilon=0.3$, $12$, and $0.05$ around the images in Plots (a), (b), and (c), respectively. 
	}
	\label{fig:hist-mag-grad}
\end{figure*}
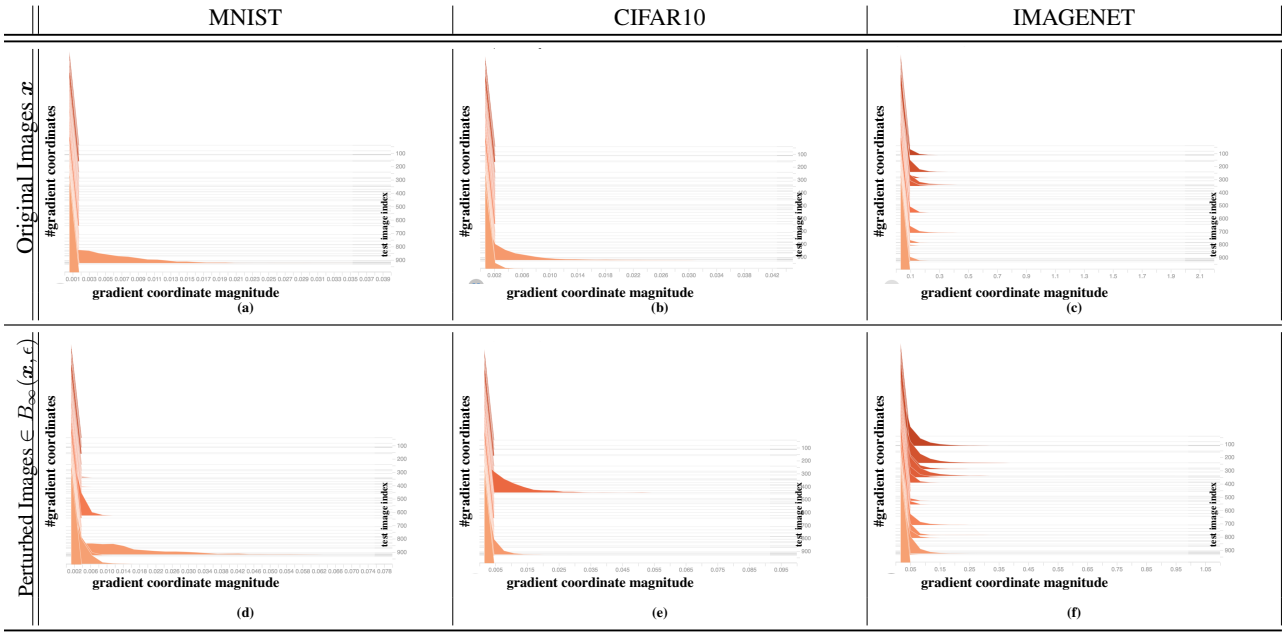

\bibliography{bib}
\bibliographystyle{icml2019}


\end{document}